\documentclass[a4paper,11pt]{article}
\usepackage{defs}

\title{Reasoning about Causality in Games\footnote{Published in \emph{Artificial Intelligence}. DOI: \href{https://doi.org/10.1016/j.artint.2023.103919}{\texttt{10.1016/j.artint.2023.103919}}. Correspondence to \href{mailto:lewis.hammond@cs.ox.ac.uk}{\texttt{lewis.hammond@cs.ox.ac.uk}} and \href{mailto:james.fox@cs.ox.ac.uk}{\texttt{james.fox@cs.ox.ac.uk}}.}}
\author[1]{Lewis Hammond}
\author[1]{James Fox}
\author[2]{Tom Everitt}
\author[1]{Ryan Carey}
\author[1]{\\Alessandro Abate}
\author[1]{Michael Wooldridge}
\affil[1]{University of Oxford}
\affil[2]{DeepMind}
\date{}

\begin{document}

\maketitle

\begin{center}
    {\large \bfseries Abstract}
\end{center}
\begin{changemargin}{0.5in}{0.5in} 
    Causal reasoning and game-theoretic reasoning are fundamental topics in artificial intelligence, among many other disciplines: this paper is concerned with their intersection. Despite their importance, a formal framework that supports both these forms of reasoning has, until now, been lacking.
    We offer a solution in the form of \emph{(structural) causal games}, which can be seen as extending 
    Pearl's causal hierarchy to the game-theoretic domain, or as extending Koller and Milch's multi-agent influence diagrams 
    to the causal domain.
    {We then consider} three key questions:
    \begin{itemize}
        \item[i)] How can the (causal) dependencies in games -- either between variables, or between strategies -- be modelled in a uniform, principled manner?
        \item[ii)] {How may causal queries be computed in causal games, and what assumptions does this require?}
        \item[iii)] How do causal games {compare} to existing formalisms?
    \end{itemize}
    {To address question i), we introduce \emph{mechanised games}, which encode dependencies between agents' decision rules and the distributions governing the game.
    In response to question ii), we present definitions of predictions, interventions, and counterfactuals, and discuss the assumptions required for each.
    Regarding question iii), we describe correspondences between causal games and other formalisms, and explain how causal games can be used to answer queries that other causal or game-theoretic models do not support.}
    Finally, we highlight possible applications of causal games, {aided} by an extensive open-source Python library.
\end{changemargin}

\restoregeometry

\newpage

\tableofcontents

\newpage

\section*{Notation}
\begin{longtable}{p{0.2\textwidth - 2\tabcolsep} p{0.7\textwidth - 2\tabcolsep} >{\raggedleft\arraybackslash}p{0.1\textwidth - 2\tabcolsep}}
    \label{tab:notation}\\
    \toprule
    \textbf{Symbol} & \textbf{Object} & \textbf{Page} \\
    \midrule
    $\Anc_V$ & Ancestors of $V$ & \pageref{def:graph_not}\\
    $\conr \graph$ & Condensed $\R$-Relevance Graph & \pageref{def:con-graph}\\
    $\Ch_V$ & Children of $V$ & \pageref{def:graph_not}\\
    $\Desc_V$ & Descendants of $V$ & \pageref{def:graph_not}\\
    $\Do$ & Do Operator & \pageref{def:do}\\
    $\dom(V)$ & Domain of $V$ & \pageref{def:randomvar}\\
    $\exovar_V$ & Exogenous Variable for $V$ & \pageref{def:exo}\\
    $\mathscr{E}$ & Edges & \pageref{def:BN}\\
    $\efg$ & Extensive-Form Game & \pageref{def:EFG}\\
    $\Fa_V$ & Family of $V$& \pageref{def:graph_not}\\
    $\graph$ & Graph & \pageref{def:BN}\\
    $\I$ & Intervention & \pageref{def:softint}\\
    $J$ & Intervention Set & \pageref{def:interset} \\
    $\model$ & Model & \pageref{def:BN}\\
    $\model(\zeta_k)$ & Perturbed Model & \pageref{def:pertMAID}\\
    $\mecvar_V$ & Mechanism Variable for $V$ & \pageref{def:mecvar}\\
    $\mec{\graph}$ & Mechanised Graph & \pageref{def:mechanised_maid}\\
    $\mec{\model}$ & Mechanised Model & \pageref{def:mechanised_maid}\\
    $N$ & Agents & \pageref{def:EFG}\\
    $\Pa_V$ & Parents of $V$ & \pageref{def:graph_not}\\
    $\Pr$ or $P$ & Probability Distribution & \pageref{def:BN}, \pageref{def:EFG}\\
    $\Pr^{\bm{\pi}}$ or $P^\sigma$ & Probability Distribution Combining $\Pr$ with ${\bm{\pi}}$ or $P$ with $\sigma$ & \pageref{def:jointEFG}, \pageref{def:jointMAID} \\
    $\mathcal{Q}$ & Queries & \pageref{def:queries}\\
    $\R$ & Rationality Relations & \pageref{def:ratrelation} \\
    $\R(\mec{\model})$ & $\R$-Rational Outcomes of $\model$ & \pageref{def:ratoutcomes}\\
    $r_D$ & Rationality Relation for $D$ & \pageref{def:ratrelation}\\
    $\relr \graph$ & $\R$-Relevance Graph & \pageref{def:rel-graph}\\
    $V$ & Variable & \pageref{def:graph_not}\\
    $v$ & Value of $V$ & \pageref{def:graph_not}\\
    $\bm{V}$ & Variables & \pageref{def:graph_not}\\
    $\bm{v}$ & Values of $\bm{V}$ & \pageref{def:graph_not}\\
    $\delta$ & Kronecker Delta Function & \pageref{def:Kronecker}\\
    $\Delta$ & Probability Simplex & \pageref{def:graph_not}\\
    $\Theta_V$ & Parameter Variable for $V$ & \pageref{def:paramvar}\\
    $\bm{\theta}$ & Parameters & \pageref{def:BN}\\
    $\mu^i$ & Mixed Policy for Agent $i$ & \pageref{def:policies}\\
    ${\bm{\pi}}^i$ & (Behavioural) Policy for Agent $i$ & \pageref{def:policy}\\
    $\dot{{\bm{\pi}}}^i$ & Pure (Behavioural) Policy for Agent $i$ & \pageref{def:policies}\\
    ${\bm{\pi}}$ & (Behavioural) Policy Profile & \pageref{def:policy}\\
    $\Pi_D$ & Decision Rule Variable for $D$ & \pageref{def:decrulevar}\\
    $\sigma$ & (Behavioural) Strategy Profile & \pageref{def:strategy}\\
    $\perp_\graph$ & d-Separated in $\graph$ & \pageref{def:dsep}\\
    $\indep$ & Independent & \pageref{def:dsep}\\
    $\prec$ & Topological Ordering & \pageref{topo}\\
    \bottomrule
\end{longtable}

\newpage

\section{Introduction}
\label{sec:introduction}

Causal reasoning and game-theoretic reasoning are core capabilities for intelligent systems, and as such, they are fundamental research topics in artificial intelligence (AI).
Causal reasoning is concerned with identifying causal relationships and estimating the effects of interventions. Game-theoretic reasoning is concerned with strategic behaviour: how rational decision-makers interact, taking into account others' incentives. Whilst formal treatments of causality \cite{Spirtes1993, Rubin2005, pearl2009causality, peters2017elements} and the game-theoretic foundations of multi-agent systems \cite{wooldridge2009introduction,Nisan2007,Shoham2008} have individually led to many recent applications in AI, our present concern is with  techniques that \emph{combine} causal and strategic reasoning.

Models that support both these kinds of reasoning offer a wide range of possible applications including analysing incentives \cite{Everitt2021,Langlois2021}, fairness \cite{Kilbertus2017,Kusner2017,Zhang2018a,Nabi2018}, and blame and intention \cite{Halpern2018,Friedenberg2019}. As systems of multi-agent systems become increasingly ubiquitous and sophisticated, the problem of how to formally reason about these notions becomes increasingly acute. %
A framework that supports causal analysis of systems containing multiple self-interested agents would therefore appear to be of great importance \cite{pfeffer2007reasoning,Friedenberg2019,Everitt2021}.

Causal questions are often more challenging to answer in multi-agent settings; one must consider not only the causal dependencies present in the environment, but also the dependencies between agents' strategies. 
Similarly, reasoning strategically about what the effects of an action \emph{would be}, or what other agents \emph{would have done} under different circumstances, naturally leads one to consider both interventions and counterfactual possibilities when analysing games. These causal concepts, however, are typically left implicit in game-theoretic models.

The central purpose of this paper is to introduce a unifying framework for modelling games that supports both causal and game-theoretic reasoning. 
This framework -- \emph{(structural) causal games} -- can be interpreted in two different ways. In one sense, it lifts Koller and Milch's (henceforth, K\&M) multi-agent influence diagrams (MAIDs) \cite{koller2003multi} from the probabilistic models at level one of Pearl's `causal hierarchy' \cite{pearl2009causality} to causal models that support both interventions and counterfactuals, corresponding to levels two and three of the hierarchy, respectively. Building on K\&M's graphical representation also means we may employ existing game-theoretic concepts such as \emph{strategic relevance} \cite{koller2003multi} and \emph{sufficient recall} \cite{milch2008ignorable}, which
can be elicited purely from the structure of the game.
Alternatively, causal games can be interpreted as generalising the models of the causal hierarchy to the game-theoretic domain by introducing decision variables that lack a distribution until chosen by the corresponding agent, and a set of real-valued utility variables representing the payoff of each agent. {Causal games thus support both game-theoretic and causal queries, as well as combinations of the two. It is our hope that this framework serves as a foundation for further work at the intersection of causality and game theory.}

\subsection{{Contributions}}
{In answering the three key questions introduced above, we make the following contributions.}

\noindent {i) How can the (causal) dependencies in games -- either between variables, or between strategies -- be modelled in a uniform, principled manner?
\begin{itemize}
    \item We introduce \emph{mechanised MAIDs} in Section \ref{sec:mechanised_maids_relevance}, which allow us to cleanly model the dependencies of decision rules on each other and on the parameters of the game.
    \item We also generalise K\&M's notion of \emph{strategic relevance} to $\R$-\emph{relevance} to enable modelling many different decision-making principles. 
    \item Furthermore, we derive sound and complete graphical criteria for detecting relevant variables when agents are playing best responses to one another.
\end{itemize}
\noindent ii) How may causal queries be computed in such models, and what assumptions does this require?
\begin{itemize}
    \item We generalise Pearl's causal hierarchy of models to the game-theoretic domain by introducing (structural) causal games in Section \ref{sec:causality_in_games}.
    \item We then describe how these models can be used to answer conditional, interventional, and counterfactual queries. By quantifying over the equilibria in the game (leading to \emph{first order} queries such as ``is it the case that in every Nash equilibrium, setting variable $X$ to value $x$ would increase agent $i$'s expected payoff?'') and taking into account causal effects due to rational agents who adapt their strategies to changes in the environment, such queries strictly generalise those in other causal models.
    \item In essence, the assumptions required to model a given problem as a causal game are the same as for their non-game-theoretic counterparts. To mechanise the causal game we further require assumptions about the decision-making principles that agents use play the game.
\end{itemize}
\noindent iii) How do the models we propose compare to existing formalisms?
\begin{itemize}
    \item We introduce subgames and two classic equilibrium refinements -- subgame perfectness and trembling hand perfectness -- to MAIDs in Section \ref{sec:subgames_eqs}, and provide a detailed comparison between MAIDs and EFGs (including several equivalence results) in Section \ref{sec:EFG_connections}.
    \item We also show that often more subgames can be found in a MAID than in the corresponding EFG, ruling out more non-credible threats, and making the computation of equilibria more efficient.
    \item Finally, we discuss a range of applications in which causal games subsume prior work in Section \ref{sec:applications}, and (dis)advantages with respect to other work in Section \ref{sec:discussion}.
\end{itemize}}

{A previous conference paper contained preliminary results on subgames and equilibrium refinements in MAIDs, as well as their connections to EFGs \cite{Hammond2021}, but did not contain any discussion of causality, which is the main focus of this paper. Similarly, a previous tool paper introduced an earlier version of our codebase \cite{pycid}, which implements many of the concepts in this paper though does not contain any theoretical work, which is the emphasis of the present paper.}

{We conclude this section with a review of other related work. Before the primary exposition of our results, we also provide the relevant background material on causal models, EFGs, and MAIDs, in Section \ref{sec:background}.} Proofs, further examples, and details of our codebase are relegated to Appendices \ref{app:proofs}, \ref{app:examples}, and \ref{app:code}, respectively.

{\subsection{Related Work}}
\label{sec:related}

Causal games build on Pearl's hierarchy of causal models \cite{pearl2009causality},
{and work on influence diagrams (IDs) -- a kind of graphical model used to capture single-agent decision-making scenarios \cite{howard2005influence}. Later works have explicitly considered \emph{causal} IDs, or CIDs \cite{dawid2002influence,Everitt2021}.}
Indeed, Heckerman and Shachter's proposal to unify causal modelling and decision analysis via CIDs \cite{Heckerman1994} can be viewed as a precursor to our proposal to unify causal modelling and game theory via causal games.
None of these theories, however, place emphasis on games or strategic interactions between multiple agents, which is our setting of interest.

{In contrast, multi-agent IDs (MAIDs) generalise IDs \cite{koller2003multi,milch2008ignorable}, and were}
originally developed by K\&M as a way to efficiently represent and solve games.
{One useful feature of MAIDs (which causal games inherit) is that their graphical structure encodes what information is, or is not, relevant for making an optimal decision. Thus in order to find an equilibrium, it is often possible to remove some of the edges (i.e. the \emph{ignorable information}) prior to solving the game \cite{milch2008ignorable}.}
Many other graphical models of games, often inspired by MAIDs, have been introduced \cite{Kearns2007}{, including networks of influence diagrams (NIDs) \cite{Gal2008}, interactive dynamic influence diagrams (I-DIDs) \cite{Doshi2008,Polich2007}, and temporal action graph games \cite{Jiang2009}.} Like MAIDs however, most of these works are motivated by computational concerns, and none incorporate rigorous causal reasoning.

{Modelling causal relationships in game-theoretic scenarios often leads to cyclic dependencies, such as when one agent's best response depends on the other's, and vice versa. This can result in multiple solutions, a feature that is also captured by generalised or cyclic causal models \cite{Halpern1998,Bongers2016}, chain graphs \cite{Lauritzen2002}, probabilistic relational models \cite{Getoor2007a,Maier2013}, and credal networks \cite{Cozman2000}, among others. In mechanised (structural) causal games, these} solutions correspond to different equilibria induced by {(serial) \emph{relations}} representing the potentially non-deterministic decision-making processes of agents in the game (such as when an agent selects a decision rule using an $\argmax$ operation, for example).
{This is essential for modelling equilibria in games, where agents may be indifferent between decision rules.}
The result is a relational causal model with cycles -- a formalism first explored concurrently with this work \cite{Ahsan2022}, but without any reference to the game-theoretic scenarios that are our focus.

Perhaps the most similar work to this paper is on \emph{settable systems},\footnote{We note that Gonzalez-Soto et al. also define a kind of causal game \cite{GonzalezSoto2019}, though theirs is essentially a Bayesian game where the outcomes are computed using an unknown causal model. This is a simpler and more restricted setup than both settable systems and our causal games.} {which are partly inspired by generalised structural causal models and} can be used to {capture} optimisation, equilibria, and learning \cite{White2009,White2014}. In order to deal with {cycles}, settable systems duplicate intervenable variables into a `response' and `setting' variable, so as to indicate which side of the structural function each occurs on. 
In these models, the emphasis is on the causal analysis of optimisation procedures at a relatively low level of abstraction, meaning that the algorithms used by agents to select actions, or the {procedures} used to select one equilibrium from many, are explicitly instantiated. 
In contrast, causal games represent the causal dependencies arising from the fact that agents select their decision rules rationally and non-deterministically, leading us to ask \emph{first-order} causal queries. Settable systems concentrate on problems such as how to capture the data and attributes of a machine learning process, whereas we focus on problems such as identifying subgames and equilibrium refinements.
{Despite these differences, settable systems are nonetheless a useful comparator for causal games.}

{The concurrency and semantics community is another that has produced work at the intersection of games and causality. Much of the most influential recent work on the foundations of denotational semantics uses distributed games that are based on \emph{event structures} -- a partially ordered model of discrete events \cite{Nielsen1979} -- for deterministic \cite{Rideau2011} and probabilistic \cite{Winskel2013} concurrent games.}
Other related approaches to concurrent games use simpler mathematical formalisms \cite{Bradfield2015,Zahoransky2021}.
Though containing the same primitive concepts, these works are motivated primarily by the problem of deriving formal, low-level semantics for programs or probabilistic systems, whereas we are interested in more high-level models of strategic interactions that can be applied to a wide range of scenarios and disciplines.

{Indeed, closely related causal models have been used (often in the context of analysing AI systems) to} define notions of blame and intent {(}both in the single-agent and multi-agent settings{)} \cite{Halpern2018,Friedenberg2019}{, harm \cite{Richens2022}, incentives to control or respond to certain variables \cite{Everitt2021,Langlois2021}, \cite{Kilbertus2017,Kusner2017,Zhang2018a,Nabi2018}, social influence \cite{Jaques2019}, and}
reasoning patterns such as manipulation and signalling \cite{pfeffer2007reasoning}.
{We show in Sections \ref{sec:causality_in_games} and \ref{sec:applications} that causal games subsume the models in these works and allow for even richer concepts.}
{Aside from analysing AI systems, one other relevant domain of application is in economic analysis and mechanism design.
For example, Toulis and Parkes use a behavioural causal model to determine long-term effects of policy interventions on multi-agent economies \cite{Toulis2016} -- though their emphasis is on dynamical systems and behavioural analysis -- and some regulators such as the UK's Financial Conduct Authority include an informal `causal chain' in their cost-benefit analyses when proposing policy analyses \cite{Authority2020}. We provide a more formal case study of this second example in  Section \ref{sec:applications}.}

\section{Background}
\label{sec:background}

We assume a basic familiarity with both probabilistic graphical models and game theory, though for completeness, we briefly review Pearl's causal hierarchy \cite{pearl2009causality} and two game-theoretic models: extensive form games (EFGs) \cite{vonNeumann1928,Kuhn1953} and multi-agent influence diagrams (MAIDs) \cite{koller2003multi}. Readers familiar with these models may safely skip these sections.

Throughout this paper, we use capital letters $V$ for random variables%
\label{def:randomvar}, lowercase letters $v$ for their instantiations, and bold letters $\bm{V}$ and $\bm{v}$ respectively for sets of variables and their instantiations.
We let $\dom(V)$ denote the domain of $V$ (where by default we assume that $\dom(V)$ is finite) and abuse notation somewhat by writing $\dom(\bm{V}) \coloneqq \bigtimes_{V \in \bm{V}}\dom(V)$. $\Pa_V$ denotes the parents of a variable $V$ in a graphical representation and $\pa_V$ the instantiation of $\Pa_V$.\label{def:graph_not} We also define $\Ch_V$, $\Anc_V$, $\Desc_V$, and $\Fa_V \coloneqq \Pa_V \cup \{V\}$ as the children, ancestors, descendants, and family of $V$, respectively (where note that neither $\Anc_V$ nor $\Desc_V$ contain $V$, by convention). As with $\pa_V$, their instantiations are written in lowercase.
We use $\Delta(\dom(\bm{V}))$ to denote the set of all probability distributions over the values of $\bm{V}$, and therefore write $\Delta(\dom(\bm{A}) \mid \dom(\bm{B})) \coloneqq \bigtimes_{\bm{b} \in \dom(\bm{B})} \Delta(\dom(\bm{A}))$ to express the set of all conditional probability distributions (CPDs) over $\bm{A}$ given the values of $\bm{B}$. 
Unless otherwise indicated, we use superscripts to indicate an agent \label{def:agent} $i \in N = \{1, \dots, n\}$ and subscripts to index the elements of a set; for example, the decision variables belonging to agent $i$ are denoted $\bm{D}^i = \{D^i_1,\ldots,D^i_m\}$.

\subsection{Causal Models}
\label{sec:causal_models}

Pearl's \emph{causal hierarchy} consists of three kinds of model: associational, interventional, and counterfactual \cite{pearl2009causality}. Each step up the hierarchy demands stricter assumptions. The lowest level, \textit{association}, pertains to correlations between variables that allow for \textit{predictions} about a system. For this, it is sufficient to use observational data to construct a joint probability distribution over all of the variables in that system, which can then be represented graphically as a \textit{Bayesian network} (BN). On the second level, we wish to reason about the effects of \emph{interventions} -- deliberate alterations made to the variables from outside the system. This requires the edges of the BN to reflect causal, not just associational, relationships, giving rise to a \textit{causal Bayesian network} (CBN). The final level of the hierarchy is concerned with counterfactual questions -- asking what would have happened had something been different, given that we made a certain observation -- which corresponds to conditioning (as in associational queries), then intervening. Answering such questions requires knowledge of the underlying deterministic relationships between the variables, typically characterised using a \textit{structural causal model} (SCM).

\subsubsection*{Association}

\begin{definition}[\citenum{pearl2009causality}]
\label{def:BN}
    A \textbf{Bayesian network (BN)} over a set of random variables $\bm{V}$ with joint distribution $\Pr(\bm{V})$ is a structure $\model = (\graph, \bm{\theta})$\label{def:model} where $\graph = (\bm{V}, \mathscr{E})$ is a directed acyclic graph (DAG) with vertices $\bm{V}$ and edges $\mathscr{E}$ that is \textbf{Markov compatible} with $\Pr$, meaning that $\Pr(\bm{v}; \bm{\theta}) = \prod_{V \in \bm{V}} \Pr(v \mid \pa_{V}; \theta_{V})$. We drop the parameters $\bm{\theta} = \{\theta_{V}\}_{V \in \bm{V}}$\label{def:params} of the CPDs from our notation where unambiguous.
\end{definition}

We can use the \textit{d-separation} graphical criterion to identify the set of conditional independencies that any Markov compatible joint distribution over a DAG $\graph$ must satisfy \cite{pearl1988probabilistic}.%

\begin{definition}[\citenum{pearl2009causality}]
    \label{def:dsep}
    A \textbf{path} $p$ in a DAG $\graph = (\bm{V}, \mathscr{E})$ is a sequence of unrepeated adjacent variables in $\bm{V}$. A path $p$ is said to be \textbf{blocked} by a set of variables $\bm{Y} \subset \bm{V}$ if and only if {$p$ contains either}:
    \begin{itemize}
        \item A \emph{chain} $X \rightarrow W \rightarrow Z$ or $X \gets W \gets Z$, or a \emph{fork} $X \leftarrow W \rightarrow Z$, and $W \in \bm{Y}${; or}
        \item A \emph{collider} $X \rightarrow W \leftarrow Z$ and $W \notin \Anc_{\bm{Y}} \cup \{Y\}$.
    \end{itemize}
    For disjoint sets $\bm{X},\bm{Y},\bm{Z}$, the set $\bm{Y}$ \textbf{d-separates} $\bm{X}$ from $\bm{Z}$, denoted $\bm{X} \perp_\graph \bm{Z} \mid \bm{Y}$, if 
    {every path in $\graph$ from a variable in $\bm{X}$ to a variable in $\bm{Z}$ is blocked by a variable in $\bm{Y}$.
    Otherwise, $\bm{X}$ is said to be \textbf{d-connected} to $\bm{Z}$ given $\bm{Y}$, denoted $\bm{X} \not\perp_\graph \bm{Z} \mid \bm{Y}$.}
\end{definition}

For example, in the graph $\graph$ shown in in Figure \ref{fig:BN_SCM:BN}, we have that $A \not\perp_\graph B \mid C$ due to the active path $A \gets D \to B$, but that $A \perp_\graph B \mid D$, as conditioning on $D$ blocks the connection along the aforementioned path as well as along the path $A \gets C \to D \to B$.

If $\bm{X} \perp_\graph \bm{Z} \mid \bm{Y}$ in $\graph$, then $\bm{X}$ and $\bm{Z}$ are probabilistically independent conditional on $\bm{Y}$ in the sense that $\Pr(\bm{x} \mid \bm{y}, \bm{z}) = \Pr(\bm{x} \mid \bm{y})$, written $\bm{X} \indep \bm{Z} \mid \bm{Y}$\label{def:probind}, in every distribution $\Pr$ that is Markov compatible with $\graph$ and for which $\Pr(\bm{y}, \bm{z})>0$. Conversely, if $\bm{X} \not\perp_\graph \bm{Z} \mid \bm{Y}$, then $\bm{X}$ and $\bm{Z}$ are dependent conditional on $\bm{Y}$ in at least one distribution Markov compatible with $\graph$ \cite{Verma1990}.

\begin{figure}[h]
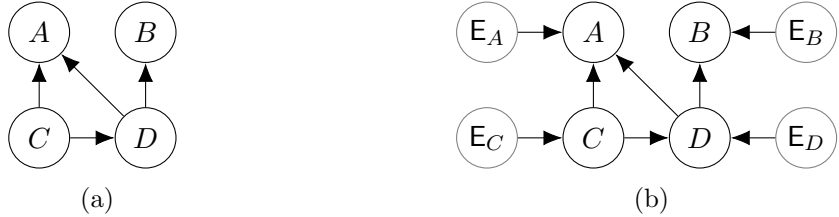

    \begin{subfigure}[b]{0.45\linewidth}
    \centering
    \begin{influence-diagram}
        \node (A) [] {$A$};
        \node (B) [right = of A] {$B$};
        \node (C) [below = of A] {$C$};
        \node (D) [below = of B] {$D$};
        \edge {C} {A, D};
        \edge {D} {A, B};
    \end{influence-diagram}
    \caption{}
    \label{fig:BN_SCM:BN}
    \end{subfigure}
    \begin{subfigure}[b]{0.45\linewidth}
    \centering
    \begin{influence-diagram}
        \node (A) [] {$A$};
        \node (B) [right = of A] {$B$};
        \node (C) [below = of A] {$C$};
        \node (D) [below = of B] {$D$};
        \edge {C} {A, D};
        \edge {D} {A, B};
        \node (A_exo) [exogenous, left = of A] {$\exovarv{A}$};
        \node (B_exo) [exogenous, right = of B] {$\exovarv{B}$};
        \node (C_exo) [exogenous, left = of C] {$\exovarv{C}$};
        \node (D_exo) [exogenous, right = of D] {$\exovarv{D}$};
        \edge {A_exo} {A};
        \edge {B_exo} {B};
        \edge {C_exo} {C};
        \edge {D_exo} {D};
    \end{influence-diagram}
    \caption{}
    \label{fig:BN_SCM:SCM}
    \end{subfigure}
 \caption{(a) A (C)BN with a Markov compatible joint distribution $\Pr(a,b,c,d) = \Pr(a \mid c,d)\Pr(b \mid d)\Pr(d \mid c)\Pr(c)$. (b) An SCM representing the same system. Each variable $V$ is now associated with an exogenous parent $\exovar_V \in \exovars$.}
 \label{fig:BN_SCM}
\end{figure}
A second well-established graphical criterion will also be a useful auxiliary result in Section \ref{sec:mechanised_maids_relevance}.
\begin{definition}[\citenum{shachter2013bayes}]
    \label{def:requisite}
    Given a DAG $\graph$, a variable $V$ is a \textbf{requisite probability node} for $\Pr( \bm{x} \mid \bm{y} )$ if there exist two parameterisations $\bm{\theta} \neq \bm{\theta}'$ of $\graph$ for BNs $\model$ and $\model'$ differing only on $\theta_V$ such that $\Pr( \bm{x} \mid \bm{y} ; \bm{\theta} ) \neq \Pr( \bm{x} \mid \bm{y} ; \bm{\theta}')$.
\end{definition}
\begin{lemma}[\citenum{Geiger1990}]
    \label{lem:reachability}
    Given a BN $\model$, a variable $V$ is a requisite probability node for the query $\Pr( \bm{X} \mid \bm{Y})$ if and only if $\mecvar_{V} \not\perp_{\meczero{\graph}} \bm{X} \mid \bm{Y}$.
\end{lemma}

\subsubsection*{Intervention}
Associational models such as BNs are, in general, insufficient to answer questions about interventions as they do not describe how the joint distribution changes in response; a causal model, or level two model, such as a causal Bayesian network (CBN), is required.
The graph underlying a CBN differs from that of a BN only in its causal interpretation: the directed edges now represent the fact that intervening on a variable cannot affect those causally `upstream' of it. The simplest form of intervention{, a \emph{hard intervention} $\Do(\bm{Y} = \bm{y})$,\label{def:do}
sets the values of variables $\bm{Y}$ to some $\bm{y}$. We denote the resulting joint distribution by $\Pr_{\bm{y}}(\bm{V})$ or, equivalently, $\Pr(\bm{V}_{\bm{y}})$ \cite{pearl2009causality}.}\footnote{This is also known as an atomic \cite{Correa2020}, structural \cite{Eberhardt2007}, surgical \cite{pearl2009causality}, or independent \cite{Korb2004} intervention.}

\begin{definition}[\citenum{pearl2009causality}]
    \label{def:CBN}
        A \textbf{causal Bayesian network (CBN)} is a BN $\model = (\graph, \bm{\theta})$ such that $\graph$ is Markov compatible with $\Pr_{\bm{y}}$ for every $\bm{Y} \subseteq \bm{V}$ and $\bm{y} \in \dom(\bm{Y})$, and that:
        $$\Pr_{\bm{y}}(v \mid \pa_V) =
        \begin{cases}
            1 & \text{ when $V \in \bm{Y}$ and $v$ is consistent with $\bm{y}$,}\\
            \Pr(v \mid \pa_V) & \text{ when $V \notin \bm{Y}$ and $\pa_V$ is consistent with $\bm{y}$.}
        \end{cases}$$
\end{definition}

{
By the Law of Total Probability, $\Pr_{\bm{y}}(v \mid \pa_V) = 0$ when $v$ is inconsistent with $\bm{y}$. When $\pa_V$ is inconsistent with $\bm{y}$, then conditioning on a zero-probability event means that $\Pr_{\bm{y}}(v \mid \pa_V)$ is undefined.
More generally, a \emph{soft intervention}, specified using a partial distribution $\I$\label{def:softint} over variables $\bm{Y}$ replaces each CPD $\Pr(Y \mid \Pa_Y)$ with a new CPD $\I(Y \mid \Pa^*_Y; \theta^*_Y)$ for each $Y \in \bm{Y}$, where $\Pa^*_Y$ may differ from $\Pa_Y$.}\footnote{A soft intervention is also known as a parametric \cite{Eberhardt2007} or dependent \cite{Korb2004} intervention, and is referred to as conditional or stochastic when deterministic or stochastic, respectively \cite{Correa2020}.
Hard interventions are a special case in which each $\I(Y \mid \Pa^*_Y) = \delta(Y {,} y)$ for some $y \in \dom(Y)$, where $\delta$ is the Kronecker delta function\label{def:Kronecker}.}
Any intervention $\I$ on the set of variables $\bm{Y}$ leads to a new joint distribution: 
$$\Pr_\I(\bm{v}) \coloneqq \prod_{Y \in \bm{Y}} \I(y \mid \pa^*_{Y}) \cdot \prod_{V \in \bm{V} \setminus \bm{Y}} \Pr(v \mid \pa_{V}).$$
We represent an intervention $\I$ on $\bm{Y}$ graphically by outlining each variable $Y \in \bm{Y}$, replacing {each} variable name $Y$ with $Y_{\I}$, and {removing or adding edges from parent variables as necessary, if $\Pa^*_Y \neq \Pa_Y$}. {More generally, we use $\Pr(\bm{V}_{\I})$ and $\Pr_{\I}(\bm{V})$ interchangeably, and denote the new graph and model as $\graph_{\I}$ and $\model_{\I}$ respectively.} When $\I$ is a hard intervention, we {simply sever all incoming edges to $\bm{Y}$, setting their values to $\bm{y}$, and write $\bm{V}_{\bm{y}}$, $\graph_{\bm{y}}$, and $\model_{\bm{y}}$ respectively}.
Examples of hard and soft interventions are shown in Figures \ref{fig:CBNintervention:1} and \ref{fig:CBNintervention:2} respectively.

\begin{figure}[h]
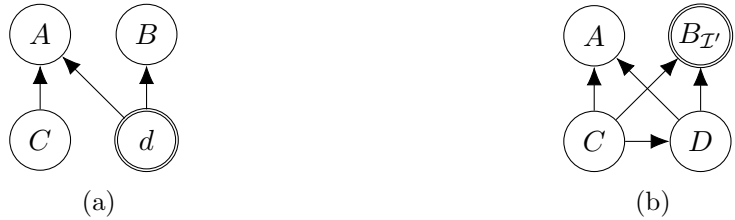

    \centering
    \begin{subfigure}[b]{0.45\linewidth}
        \centering
        \begin{influence-diagram}
        \node (A) [] {$A$};
        \node (B) [right = of A] {$B$};
        \node (C) [below = of A] {$C$};
        \node (D) [intervened, below = of B] {$d$};
        \edge {C} {A};
        \edge {D} {A, B};
    \end{influence-diagram}
        \caption{}
        \label{fig:CBNintervention:1}
    \end{subfigure}
    \begin{subfigure}[b]{0.45\linewidth}
        \centering
        \begin{influence-diagram}
        \node (A) [] {$A$};
        \node (B) [intervened, right = of A] {$B_{\I'}$};
        \node (C) [below = of A] {$C$};
        \node (D) [below = of B] {$D$};
        \edge {C} {A, D, B};
        \edge {D} {A, B};
    \end{influence-diagram}
        \caption{}
        \label{fig:CBNintervention:2}
    \end{subfigure}
    \caption{(a) The graph $\graph_{d}$ representing the CBN in Figure \ref{fig:BN_SCM:BN} after a hard intervention $\Do(D = d)$. (b) A similar graph $\graph_{\I}$ where $\I(B \mid \Pa^*_{B})$ is a soft intervention with $\Pa^*_{B} = \Pa_{B} \cup \{C\}$.}
    \label{fig:CBNintervention}
\end{figure}

\subsubsection*{Counterfactuals}

In counterfactual queries, evidence about the actual state of the world informs us about a hypothetical scenario in which some variables have been modified. For instance, we might be interested in the probability of $\bm{x}$ in the scenario in which $\bm{y}$, given that (in fact) we observed $\bm{z}$, written $\Pr(\bm{x}_{\bm{y}} \mid \bm{z})$. 
To answer such questions in general, one must appeal to level three of the causal hierarchy, such as by using a structural causal model (SCM). In SCMs, variables are partitioned into exogenous and endogenous sets $\exovars$\label{def:exo} and $\bm{V}$ respectively, where each endogenous variable $V \in \bm{V}$ is deterministically related to its parents via a structural function $f_V : \dom(\bm{V} \setminus \{V\}) \times \dom(\exovars) \to \dom(V)$ that specifies the mechanism governing the values of the variable, and where all stochasticity is relegated to the distribution $\Pr(\exovars ; \bm{\theta})$ over the exogenous variables.

    In this paper, we make the simplifying assumption that all SCMs are \emph{Markovian}, meaning that each variable $V$ has exactly one exogenous parent $\exovar_V$ and the exogenous variables are independent. We also depart from convention by describing SCMs as a particular form of CBN (which, in turn, are a particular form of BN), using deterministic distributions $\Pr({V} \mid {\Pa_V}) = \delta\big({V,} f_V({\Pa_V})\big)$ for each endogenous variable $V$. This equivalent formulation will prove useful for avoiding unnecessary repetition and notation when introducing causal models of games in Section \ref{sec:causality_in_games}.

\begin{definition}[\citenum{pearl2009causality}]
    \label{def:SCM}
    A (Markovian) \textbf{structural causal model (SCM)} is a CBN $\model = (\graph, \bm{\theta})$ where $\graph = (\exovars \cup \bm{V}, \mathscr{E})$ is a DAG over exogenous variables $\exovars = \{  \exovar_V\}_{V \in \bm{V}}$ and endogenous variables $\bm{V}$, where $\Pa_V \cap \exovars = \{ \exovar_V \}$. The parameters $\bm{\theta}$ assign deterministic distributions $\Pr(v \mid \pa_V; \theta_V)$ to each endogenous variable and a stochastic distribution $\Pr(\exovars ; \bm{\theta}) = \prod_{\exovar \in \exovars} \Pr(\exovar; \theta_{\exovar})$ to the exogenous variables.
\end{definition}

Using such a model, we can evaluate a general counterfactual query $\Pr(\bm{x}_{\I} \mid \bm{z})$ by following three steps \cite{pearl2009causality}: 
\begin{enumerate}
    \item Update $\Pr(\exovals)$ to $\Pr(\exovals \mid \bm{z})$ by conditioning on observation $\bm{z}$ (`abduction');
    \item Apply the intervention $\I(\bm{Y} \mid \Pa^*_{\bm{Y}})$ to the variables $\bm{Y}$ (`action');
    \item Return the marginal distribution $\Pr(\bm{x})$ in this modified model (`prediction').
\end{enumerate}
By convention, exogenous variables are viewed as beyond the realm of observation and intervention. Further, we assume that each $\I(Y \mid \Pa^*_{Y})$ is a deterministic function of its parents. Soft, stochastic interventions may be modelled by adding a new exogenous variable $\exovar^*_Y \in \Pa^*_Y$.

Note that by marginalising out the exogenous variables in an SCM, we can form a standard (C)BN with a joint distribution $\Pr(\bm{v}) = \sum_{\exovals} \Pr(\bm{v}, \exovals)$. 
This CBN is Markov compatible with the SCM's graph $\graph$ when restricted to variables in $\bm{V}$. Moreover, %
any CBN is also a BN. Thus, each model in the causal hierarchy can be seen to generate those on lower levels; every query that can be answered in a lower level model can also be answered in a higher level model.

\subsection{Game-Theoretic Models}
\label{sec:game_models}

We now review two important formalisms for representing sequential strategic decision-making scenarios, extensive form games (EFGs) and multi-agent influence diagram (MAIDs).
Both of these models are illustrated using the following example of a signalling game \cite{spence1978job}.

\begin{example}[Job Market]
    \label{ex:job_market}
    A worker, who is either hard-working or lazy, is hoping to be hired by a firm. They have the option of pursuing a university education, but know that they will suffer from three years of studying, especially if they are lazy. 
    The firm prefers hard-workers, but is using an automated hiring system that can only observe the worker's education, not their temperament.
\end{example}

\begin{figure}[h]
    \centering
    \begin{subfigure}[b]{0.6\linewidth}
        \centering
        \begin{istgame}[scale=1.0]
        \xtdistance{20mm}{20mm}
        \istroot(0)[chance node]{$V^0$}
        \istb<grow=left>{h}[a]
        \istb<grow=right>{\neg h}[a]
        \istb<grow=left>{\text{\footnotesize $p$}}[b]
        \istb<grow=right>{\text{\footnotesize $1-p$}}[b]
        \endist
        \xtdistance{10mm}{20mm}
        \istroot(1)(0-1)<180, red!70>{$V^1_1$}
        \istb<grow=north>{g}[l]
        \istb<grow=south>{\neg g}[l]
        \endist
        \istroot(2)(0-2)<0, red!70>{$V^1_2$}
        \istb<grow=north>{g}[r]
        \istb<grow=south>{\neg g}[r]
        \endist
        \istroot'[north](a1)(1-1)<315, blue!70>{$V^2_1$}
        \istb{j}[bl]{(4,3)}
        \istb{\neg j}[br]{(-1,-1)}
        \endist
        \istroot(b1)(1-2)<45, blue!70>{$V^2_2$}
        \istb{j}[al]{(5,3)}
        \istb{\neg j}[ar]{(0,-1)}
        \endist
        \istroot(a2)(2-2)<100, blue!70>{$V^2_3$}
        \istb{j}[al]{(5,-2)}
        \istb{\neg j}[ar]{(0,0)}
        \endist
        \istroot'[north](b2)(2-1)<225, blue!70>{$V^2_4$}
        \istb{j}[bl]{(3,-2)}
        \istb{\neg j}[br]{(-2,0)}
        \endist
        \xtInfoset(a1)(b2){\textcolor{blue!70}{$I^2_1$}}
        \xtInfoset(b1)(a2){\textcolor{blue!70}{$I^2_2$}}[below]
        \xtSubgameBox(1){(1-1)(1-2)(2-1)(2-2)(a1-1)(a1-2)(b1-1)(b1-2)(a2-1)(a2-2)(b2-1)(b2-2)}[black,inner sep = 25pt]
        \end{istgame}
        \caption{}
        \label{fig:job_market_game:EFG}
    \end{subfigure}
    \begin{subfigure}[b]{0.3\linewidth}
        \centering
        \begin{influence-diagram}
            \node (T) [] {$T$};
            \node (U1) [utility, right = of T, player1] {$U^1$};
            \node (U2) [utility, right = of U1, player2] {$U^2$};
            \node (D1) [decision, below right = 1.4cm and 0.7cm of T, player1] {$D^1$};
            \node (D2) [decision, right = of D1, player2] {$D^2$};
            \edge [information] {T} {D1};
            \edge [information] {D1} {D2};
            \path (T) edge[->, bend left=45] (U2);
            \edge {T} {U1};
            \edge {D1} {U1};
            \edge {D2} {U1,U2};
        \end{influence-diagram}
        \caption{}
        \label{fig:job_market_game:MAID}
    \end{subfigure}
    \caption{(a) An EFG representing Example \ref{ex:job_market}. (b) A MAID representing the same game.
    }
    \label{fig:job_market_game}
\end{figure}

\subsubsection*{Extensive Form Games}
\label{sec:EFGs}

The material on EFGs is required only for Section \ref{sec:EFG_connections}, and so may be safely skipped or referred back to later, depending on the reader's preferences.

\begin{definition}[\citenum{Kuhn1953}]
    \label{def:EFG}
    An \textbf{extensive form game (EFG)} is a structure $\efg = (N, T, P, A, \lambda, I, U)$, where:
    \begin{itemize}
        \item $N = \{1,\dots,n\}$ is a set of agents.
        \item $T = (\bm{V}, \mathscr{E})$ is a game tree with nodes $\bm{V}$ 
        that are partitioned into sets $\bm{V}^0, \bm{V}^1, \dots, \bm{V}^n, \bm{L}$ where $R \in \bm{V}$ is the root of $T$, $\bm{L}$ are the leaves of $T$, $\bm{V}^0$ are chance nodes, and $\bm{V}^i$ are the decision nodes controlled by agent $i \in N$. The nodes are connected by edges $\mathscr{E}$.
        \item $P = \{P_1,\dots,P_{\vert\bm{V}^0\vert}\}$ is a set of probability distributions $P_j(\Ch_{V^0_j})$ over the children of each chance node $V^0_j$.
        \item $A$ is a set of actions, where $A^i_j \subseteq A$ denotes the set of actions available at $V^i_j \in \bm{V}^i$.
        \item $\lambda : \mathscr{E} \rightarrow A$ is a labelling function mapping each edge $(V^i_j, V^k_l)$ to an action $a \in A^i_j$.
        \item $I = \{I^1,\dots,I^n\}$ contains a collection of information sets $I^i \subset 2^{\bm{V}^i}$, which partition the decision nodes controlled by agent $i$. Each information set $I_j^i \in I^i$ is defined such that for all $V^i_k, V^i_l \in I_j^i$, the available actions $A^i_j \coloneqq A_{V^i_k} = A_{V^i_l}$ are the same at both nodes.
        \item $U : \bm{L} \rightarrow \mathbb{R}^n$ is a utility function mapping each leaf node to a vector that determines the final payoff for each agent.
    \end{itemize}
\end{definition}

Figure \ref{fig:job_market_game:EFG} shows Example \ref{ex:job_market}'s signalling game in extensive form. Nature, as a chance node $V^0$, flips a biased coin at the root of the tree to decide whether the worker is hard-working $h$ (with probability $p$) or lazy $\neg h$ (with probability $1-p$). 
The worker's decision whether to go $g$ or avoid $\neg g$ university is represented at nodes $V^1_1, V^1_2 \in \bm{V}^1$ and the automated hiring system's decisions are given by $V^2_1, V^2_2, V^2_3, V^2_4 \in \bm{V}^2$, each with the option to offer a job or reject the worker ($j$ or $\neg j$). The two non-singleton information sets (marked by dotted black lines) represent the fact that the hiring system does not receive the worker's temperament as an input.

The payoffs for the worker and the firm are given by the first and second elements at the leaves of the tree respectively. The worker receives a payoff of $5$ if they are given a job offer, but they incur a cost of $1$ for going to university if they are hard-working and a cost of $2$ for going to university if they are lazy. The firm receives a payoff of $3$ if they hire a hard worker. If they offer a job to a lazy worker, they incur a cost of $2$, and if they reject a hard worker, they incur an opportunity cost of $1$.

\begin{definition}
    \label{def:strategy}
    Given an EFG $\efg = (N, T, P, A, \lambda, I, U)$, a (behavioural) \textbf{strategy} $\sigma^i$ for agent $i$ is a set of probability distributions $\sigma_j^i(A^i_j)$ over the actions available to the agent at each of their information sets $I_j^i$. A strategy is \textbf{pure} when each $\sigma_j^i(a) \in \{0,1\}$ and \textbf{fully stochastic} when $\sigma_j^i(a) > 0$, for all $a \in A^i_j$. A \textbf{strategy profile} $\sigma = (\sigma^1, \dots, \sigma^n)$ is a tuple of strategies, and $\sigma^{-i} = (\sigma^1, \dots, \sigma^{i-1}, \sigma^{i+1},\dots, \sigma^n)$ denotes the partial strategy profile of all agents other than $i$, hence $\sigma = (\sigma^i, \sigma^{-i})$. 
\end{definition}

Combining a strategy profile $\sigma$ with the distributions in $P$ defines a probability distribution $P^\sigma$\label{def:jointEFG} over paths $\rho$ in $\efg$. For each path $\rho$ beginning from $R$ and terminating in a leaf node $\rho[\bm{L}] \in \bm{L}$, agent $i$ receives utility $U(\rho[\bm{L}])[i]$ -- the $i$\textsuperscript{th} entry in the corresponding payoff vector. Agent $i$'s expected utility under a strategy profile $\sigma$ is therefore given by $\expect_{\sigma} \big[ U(\rho[\bm{L}])[i] \big]$.

\begin{definition}
    \label{def:EFGsubgame}
    A \textbf{subgame} of an EFG, $\efg = (N, T, P, A, \lambda, I, U)$, is the game $\efg$ restricted to a subtree $T' = (\bm{V}',\mathscr{E}')$ of $T$ such that: for any information set $I^i_j$ in $\efg$, if there exists $V_k \in I^i_j \cap \bm{V}'$, then $I^i_j \subseteq \bm{V}'$; and for any $V_k \in \bm{V}'$, if $(V_k, V_l) \in \mathscr{\mathscr{E}}$, then $V_l \in \bm{V}'$ and $(V_k, V_l) \in \mathscr{E}'$.
\end{definition}

In other words, a subtree of the original game tree forms a subgame if it is closed under information sets and descendants.
Any EFG is trivially a subgame of itself, so a subgame on a strictly smaller subtree is called \textit{proper}. In this paper, we denote subgames in EFGs by enclosing them within dashed boxes. For instance, the EFG in Figure \ref{fig:job_market_game:EFG} has no proper subgames, but the EFG given later in Figure \ref{fig:warehouse:c} has two.

\subsubsection*{Multi-Agent Influence Diagrams}
\label{sec:MAIDs}

Influence diagrams (IDs) generalise BNs to the decision-theoretic setting by adding decision and utility variables \cite{howard2005influence, miller1976development}, and  
multi-agent influence diagrams (MAIDs) generalise IDs by introducing multiple agents \cite{koller2003multi, milch2008ignorable}. MAIDs can therefore be viewed as a BN over a graph without parameters for the decision variables, although technically lie between levels one and two of Pearl's causal hierarchy as the decisions are effectively modelled as causal interventions \citep{Heckerman1994}, but paths between non-decision variables need not encode causal relationships \cite{Everitt2021}.

\begin{definition}[\citenum{koller2003multi}]
    \label{def:MAID}
        A \textbf{multi-agent influence diagram (MAID)} is a structure $\model = (\graph, \bm{\theta})$ where $\graph = (N, \bm{V}, \mathscr{E})$ specifies a set of agents $N = \{1,\dots,n\}$ and a DAG $(\bm{V}, \mathscr{E})$ where $\bm{V}$ is partitioned into chance variables $\bm{X}$, decision variables $\bm{D} = \bigcup_{i \in N} \bm{D}^i$, and utility variables $\bm{U} = \bigcup_{i \in N} \bm{U}^i$. 
        The parameters $\bm{\theta} = \{\theta_{V}\}_{V \in \bm{V} \setminus \bm{D}}$ define the CPDs $\Pr(V \mid \Pa_V ; \theta_V)$ for each non-decision variable such that for \emph{any} parameterisation of the decision variable CPDs, the resulting joint distribution over $\bm{V}$ induces a BN.
\end{definition}

Figure \ref{fig:job_market_game:MAID} shows a MAID representing Example \ref{ex:job_market}. Chance variables, such as whether the worker's temperament is hard-working or lazy ($T$), are denoted by white circles. Decision and utility variables are represented using squares and diamonds respectively. The worker's decision ($D^1$) and utility ($U^1$) variables are displayed in red, and the firm's ($D^2$ and $U^2$) in blue. Instead of information sets, the fact that an agent is unaware of the value of a certain variable when making a decision is represented by a missing edge between the two (e.g., the absence of an edge $T \rightarrow D^2$). The parameters $\bm{\theta}$ define conditional distributions for variables $T$, $U^1$, and $U^2$, in accordance with the values shown in Figure \ref{fig:job_market_game:EFG}.

\begin{definition}
    \label{def:policy}
    Given a MAID $\model = (\graph, \bm{\theta})$, a \textbf{decision rule} $\pi_D$ for $D \in \bm{D}$ is a CPD $\pi_D(D \mid \Pa_D)$ and a \textbf{partial policy profile} $\pi_{\bm{D}'}$ is a set of decision rules $\pi_D$ for each $D \in \bm{D}' \subseteq \bm{D}$, where we write $\pi_{-\bm{D}'}$ for the set of decision rules for each $D \in \bm{D} \setminus \bm{D}'$. 
    A (behavioural) \textbf{policy} ${\bm{\pi}}^i$ refers to ${\bm{\pi}}_{\bm{D^i}}$, and a (full, behavioural) \textbf{policy profile} ${\bm{\pi}} = ({\bm{\pi}}^1,\ldots,{\bm{\pi}}^n)$ is a tuple of policies, where ${\bm{\pi}}^{-i} \coloneqq ({\bm{\pi}}^1, \dots, {\bm{\pi}}^{i-1}, {\bm{\pi}}^{i+1}, \dots, {\bm{\pi}}^n)$. 
    A decision rule is \textbf{pure} if $\pi_D(d \mid \pa_D) \in \{0,1\}$ and \textbf{fully stochastic} if $\pi_D(d \mid \pa_D) > 0$ for all $d \in \dom(D)$ and each  \textbf{decision context} $\pa_D \in \dom(\Pa_D)$; this holds for a policy (profile) if it holds for all decision rules in the policy (profile).
\end{definition}

By combining ${\bm{\pi}}$ with the partial distribution $\Pr$ over the chance and utility variables, we obtain a joint distribution:\label{def:jointMAID}
$$\Pr^{\bm{\pi}}(\bm{x},\bm{d},\bm{u}) \coloneqq \prod_{V \in \bm{V} \setminus \bm{D}} \Pr (v \mid \pa_{V}) \cdot \prod_{D \in \bm{D}} \pi_{D} (d \mid \pa_{D}),$$
over all the variables in $\model$; inducing a BN.
The expected utility for an agent $i$ given a policy profile ${\bm{\pi}}$ is defined as the expected sum of their utility variables in this BN, $\sum_{U \in \bm{U}^i} \expect_{{\bm{\pi}}} [ U ]$ \label{def:EU}. {This allows a Nash equilibrium (NE) \cite{nash1950equilibrium} to be defined, which identifies outcomes of a game where every agent is simultaneously playing a best-response.%
\footnote{In Section \ref{sec:NE}, we build upon what's already known about NE in MAIDs, by explaining the difference between mixed policies and behavioural policies in MAIDs and clarifying when {an} NE is guaranteed to exist.}}
\begin{definition}[\citenum{koller2003multi}] 
    \label{def:NE}
    A policy profile ${\bm{\pi}}$ is a \textbf{Nash equilibrium (NE)} in a MAID if, for every agent $i \in N$, ${\bm{\pi}}^i \in \argmax_{\hat{{\bm{\pi}}}^i \in \dom({\bm{\Pi}}^i)} \sum_{U \in \bm{U}^i} \expect_{(\hat{{\bm{\pi}}}^i, {\bm{\pi}}^{-i})} [ U ]$.
\end{definition}

{K\&M also define \emph{strategic relevance ($s$-relevance)} \citep{koller2003multi} as a concept to infer whether the choice of a decision rule can affect the optimality of another decision rule. They further show how $s$-relevance can be determined using a graphical criterion (\emph{$s$-reachability}). In what follows, we use capital letters for variables $\Pi_D$ and $\bm{\Theta}$ which may take different values $\pi_D$ and $\bm{\theta}$. A more detailed introduction to this notation is provided in Section \ref{sec:mechanised_maids}}.
\begin{definition}[\citenum{koller2003multi}]
    \label{def:strategic-relevance}
    Let $D_k, D_l \in \bm{D}$ be decision nodes in a MAID $\model$. $\Pi_{D_l}$ is \textbf{strategically relevant} (or \emph{$s$-relevant}) to $\Pi_{D_k}$ if there exist two joint distributions over $\bm{V}$ parameterised by $\bm{\Theta}$ and policy profiles ${\bm{\pi}}$ and ${\bm{\pi}}'$ respectively, and a decision rule $\pi_{D_k}$, such that:
    \begin{itemize}
        \item $\pi_{D_k} \in \argmax_{\hat{\pi}_{D} \in \dom(\Pi_{D})} \sum_{U \in \bm{U}^i} \expect_{(\hat{\pi}_{D}, {\bm{\pi}}_{-D})} [ U ]$;
        \item ${\bm{\pi}}$ differs from ${\bm{\pi}}'$ only at $\Pi_{D_l}$;
        \item $\pi_{D_k} \notin \argmax_{\hat{\pi}_{D} \in \dom(\Pi_{D})} \sum_{U \in \bm{U}^i} \expect_{(\hat{\pi}_{D}, {\bm{\pi}}'_{-D})} [ U ]$, and neither does any decision rule $\tilde{\pi}_{D}$ that agrees with $\pi_{D_k}$ on all $\pa_{D_k}$ such that $\Pr^{{\bm{\pi}}'}(\pa_{D_k}) > 0$.\footnote{This final condition is included in order to ensure that $\pi_{D_k}$ doesn't lead to poor decisions in decision contexts that occur with probability zero.}
    \end{itemize}
\end{definition}
{\begin{proposition}[\citenum{koller2003multi}]
    \label{prop:s-reachability}
    {$\Pi_{D'}$} is $s$-relevant to $\Pi_{D}$ if and only if {$\Pi_{D'} \not\perp_{\graph'} \bm{U}^i \cap \Desc_{D} \mid D, \Pa_{D}$ ($\Pi_{D}$ is $s$-reachable from $\Pi_{D'}$), where $\graph'$ is the same as $\graph$ with an additional variable $\Pi_{D'}$ and edge $\Pi_{D'} \to D'$}.
\end{proposition}}
{Both $s$-relevance and $s$-reachability only consider which other \emph{decision rules} matter under a particular assumption about the \emph{rationality} of agents (corresponding to a notion of subgame perfectness). In \cref{sec:mechanised_maids_relevance}, we generalise this idea to also consider the parameterisation of non-decision variables, which is key to reasoning about causality in games.
We also generalise the concepts to other assumptions about agents' rationality.
By considering the directed (but not necessarily acyclic) graph over all variables $\bm{\Pi_D} \coloneqq \{\Pi_D\}_{D \in \bm{D}}$ such that there is an edge $\Pi_{D'} \to \Pi_D$ if and only if $\Pi_{D'}$ is $s$-relevant to $\Pi_D$ -- called the ($s$-)relevance graph -- K\&M also introduce a weakening of the perfect recall assumption that is sufficient for the existence of an NE.
}

\begin{definition}[\citenum{koller2003multi}]
    \label{def:perfectrecall}
    Agent $i$ in a MAID $\model$ has \textbf{perfect recall} if there exists a topological ordering $D_1 \prec \cdots \prec D_m$ over $\bm{D}^i$ such that $\Fa_{D_j} \subset \Pa_{D_k}$ for any $1 \leq j < k \leq m$. $\model$ is said to have perfect recall if all agents in $\model$ have perfect recall.
\end{definition}

\begin{definition}[\citenum{milch2008ignorable}]
    \label{def:sufrecall}
    Agent $i$ in a MAID $\model$ has \textbf{sufficient recall} if the $s$-relevance graph restricted to just agent $i$'s decision rules is acyclic. The MAID $\model$ is said to have sufficient recall if all agents have sufficient recall.\footnote{Note that although sufficient recall is defined using the $s$-relevance graph, a similar criterion could be created for any set of rationality relations $\R$ and the resulting $\R$-relevance graph.
    }
\end{definition}

\begin{proposition}[\citenum{koller2003multi}]
    \label{prop:NEexistance}
    Any MAID with sufficient recall has at least one NE in behavioural policies.
\end{proposition}

\section{Mechanised MAIDs and Relevance}
\label{sec:mechanised_maids_relevance}

Although MAIDs allow us to elegantly and concisely represent the dependencies between variables in a game, the edges in the graph only tell part of the story. Indeed, game-theoretic models are traditionally presented as objects to be \emph{solved}, with the procedure used to produce solutions (i.e., policy profiles) left extrinsic to the game representation. However, the fact that agents act strategically {means that agents' policies may be dependent on some of the parameters of the game (as well as other agents' policies) in ways that} are crucial for causal reasoning, yet not explicitly represented by MAIDs.

In this section, we introduce \emph{mechanised MAIDs}, which allow us to model these dependencies alongside the existing dependencies of the MAID. 
This representation will be fundamental for causal reasoning in games (in Section \ref{sec:causality_in_games}), as well as for the introduction of subgames in MAIDs (in Section \ref{sec:MAID_subgames}). In the following subsections, we first define mechanised MAIDs before formally introducing the concept of \emph{relevance} and corresponding graphical criteria.

\subsection{Mechanised MAIDs}
\label{sec:mechanised_maids}

In order to explicitly capture the implicit dependencies between the decision rules and CPDs of a MAID $\model = (\graph, \bm{\theta})$, a \emph{mechanised MAID} $\mec{\model}$ adds two elements: a (graphical) representation of the decision rules and parameters, and a description of how these depend on one another. 

\subsubsection*{Mechanised Graphs}
For the first addition, we extend $\graph$ to form a \emph{mechanised graph} $\mec{\graph}$. For each decision variable $D$, a new parent $\Pi_D$ representing its decision rule\label{def:decrulenode} is added, and for each non-decision variable $V$, a new parent $\Theta_V$ representing the parameters of its CPD\label{def:paramnode}. We call these additional parents \emph{mechanism variables} $\mecvars$\label{def:mecvar}, as they determine the mechanisms by which values of the variables in the game are set,\footnote{This name is partially inspired by the field of mechanism design, in which agents act by reporting their types $t \in T$ and the mechanism $f : T \to O$ determines how this type profile (which we might view as $\bm{d}$) is mapped to an outcome $o \in O$ of the game (which we might view as $\bm{v}$).} and the variables $\bm{V}$ in the original MAID \emph{object-level variables}. To distinguish between different types of mechanism variable, we often refer to those for decisions $\bm{\Pi} = \mecvars_{\bm{D}}$ as \emph{decision rule variables}\label{def:decrulevar} and those for non-decisions $\bm{\Theta} = \mecvars_{\bm{V} \setminus \bm{D}}$ as \emph{parameter variables}\label{def:paramvar}.

In the mechanised graph $\mec{\graph}$, we also add new edges $\mathscr{E}' \subseteq \bigcup_{D \in \bm{D}}\big( (\mecvars \setminus \Pi_D) \times \Pi_D \big)$ from other mechanism variables into decision rule variables. This represents the fact that agents typically select a decision rule $\pi_D$ (i.e., the value of $\Pi_D$) based on both the parameterisation of the game (i.e., the values of $\bm{\Theta}$) and the selection of the other decision rules in the game ${\bm{\pi}}_{-D}$. For example, the worker and the firm's hiring system in Example \ref{ex:job_market} might want to select their decision rules $\pi_{D^1}$ and $\pi_{D^2}$ as a function of the probability $p$ that a worker has a hard-working temperament $T=h$. This would imply edges $\Theta_T \to \Pi_{D^1}$ and $\Theta_T \to \Pi_{D^2}$.

In general, an agent might select each of their decision rules based on the value of any other mechanism variable -- and so $\mathscr{E}'$ would be maximal (i.e., $\Pa_{\Pi_D} = \mecvars_{\bm{V} \setminus \{D\}}$) -- though typically some of these variables will not be relevant and thus $\mathscr{E}'$ will not be maximal, a notion we make precise in Section \ref{sec:relevance}. A mechanised graph $\mec{\graph}$ for Example \ref{ex:job_market} is shown in Figure \ref{fig:mechanised_game:maid}, where decision rule and parameter variables are represented using black and white rounded squares respectively.

\subsubsection*{Mechanised Games}
For the second addition to $\model$ -- a description of how the decision rules and parameters depend on one another -- we must specify how the values of the new mechanism variables are determined and how each of the CPDs of the original object-level variables are defined as a function of their new mechanism variable parent. Beginning with the latter, note that for any parameterisation $\mecvals_{\bm{V} \setminus \bm{D}} = \bm{\theta} \in \dom(\bm{\Theta})$ of the game and any policy profile $\mecvals_{\bm{D}} = {\bm{\pi}} \in \dom({\bm{\Pi}})$, the resulting joint distribution over $\bm{V}$ can be written as:
    $$\Pr^{\bm{\pi}}(\bm{v} ; \bm{\theta}) = \Pr(\bm{v} \mid \mecvals) \coloneqq \prod_{V \in \bm{V}} \Pr(v \mid \pa_{V}, \mecval_{V}),$$
where $\Pr(v \mid \pa_{V}, \mecval_{V})$ is simply the CPD $\Pr(v \mid \pa_{V}; \theta_{V})$ defined by parameters $\theta_{V} = \mecval_{V}$ if $V$ is a non-decision variable, or the decision rule $\pi_D(d \mid \pa_{D})$ defined by $\pi_D = \mecval_{V}$ if $V$ is a decision variable $D$. As such, given the parameters $\bm{\theta}$ of the MAID $\model$ and a policy profile ${\bm{\pi}}$, the distribution over $\bm{V}$ in $\mec{\model}$ is identical to that in the original MAID. 

Finally, we must specify how the values of the mechanism variables are determined. For the parameter variables, we simply set the distribution over each $\Theta_V$ to $\delta(\Theta_V {,} \theta_V)$, and let its domain be given by $\Delta(\dom(V) \mid \dom(\Pa_V))$. Intuitively, this may be viewed as the fact that any particular game induced over the graph corresponds to an instantiation of the parameter variables in the mechanised game, the values of which are known by all agents.\footnote{{This simply amounts to a common prior assumption. Though such an assumption can be relaxed, doing so introduces additional complexities that we do not address in this work.}}

In order to provide values for the decision rule variables, we introduce a set of \emph{rationality relations} $\R = \{r_D\}_{D \in \bm{D}}$ \label{def:ratrelation} that describe assumptions about how the agents choose decision rules. Concretely, each decision rule $\Pi_D$ is governed by a \emph{serial relation} $r_D \subseteq \dom(\Pa_{\Pi_D}) \times \dom(\Pi_D)$\label{def:serialrelation}.\footnote{A serial relation $r \subseteq A \times B$ is one such that for any $a \in A$ there exists at least one $b \in B$ such that $(a,b) \in r$.}
This accounts for the fact an agent may not deterministically choose a \emph{single} decision rule $\pi_D$ in response to some $\pa_{\Pi_D}$. 
In the remainder of the paper we abuse notation by {also} using $r_D(\pa_{\Pi_D})$ as a set-valued function to denote the \emph{set} of all decision rules $\pi_D$ such that $r_D(\pa_{\Pi_D}) = \pi_D$.\footnote{{Whether we refer to some $r_D(\pa_{\Pi_D}) \in \dom(\Pi_D)$ or $r_D(\pa_{\Pi_D}) \subseteq \dom(\Pi_D)$ will typically be unambiguous.}}
 
\begin{definition}
    \label{def:mechanised_maid}
    Given a MAID $\model = (\graph, \bm{\theta})$ and a set of rationality relations $\R$, a \textbf{mechanised MAID} is a structure $\mec{\model} = (\mec{\graph}, \bm{\theta}, \R)$, such that: the \textbf{mechanised graph} $\mec\graph=(N, \bm{V} \cup \mecvars, \mec{\mathscr{E}})$ is a directed (possibly cyclic) graph over $\bm{V}$ and $\mecvars$ with edges
    $\mec{\mathscr{E}} \coloneqq \mathscr{E} \cup \{(\mecvar_{V}, V)\}_{V \in \bm{V}} \cup \mathscr{E}'$, where $\mathscr{E}' \subseteq \bigcup_{D \in \bm{D}}\big( (\mecvars \setminus \Pi_D) \times \Pi_D \big)$; and $\R = \{ r_D \}_{D \in \bm{D}}$ is a set of \textbf{rationality relations}, where each $r_D \subseteq \dom(\Pa_{\Pi_D}) \times \dom(\Pi_D)$ is a serial relation.
\end{definition}

As an example, let us suppose that each agent in Example \ref{ex:job_market} plays a \emph{best response} with respect to the other. In this case, the values of $\Pi_{D^1}$ and $\Pi_{D^2}$ are determined by relations $\R^{\BR} = \{r^{\BR}_{D^1}, r^{\BR}_{D^2}\}$ such that:
\begin{equation}
    \label{eq:BR_relation}
    \pi_{D} \in r^{\BR}_{D}(\pa_{\Pi_{D}}) ~~\Leftrightarrow~~ \pi_{D} \in \argmax_{\hat{\pi}_{D} \in \dom(\Pi_{D})} \sum_{U \in \bm{U}^i} \expect_{(\hat{\pi}_{D}, \bm{\pi}_{-D})} [ U ],
\end{equation} 
for each $D \in \bm{D}^i$, %
where note that the expectation -- despite the fact that the notation above includes object-level variables -- is defined in terms of the \emph{mechanism variables} $\pa_{\Pi_{D}} = \mecvals_{\bm{V} \setminus \{D\}}$ and $\pi_{D}$. {Note that for games in which each agent has only one decision rule, this gives rise to an NE, as defined in \cref{def:NE}. In other words, the values of the mechanism variables determine the CPDs of the object-level variables, which in turn define the joint distribution over the object-level variables, and hence any expected quantities in the game.}
Many other kinds of game-theoretic equilibria can be naturally defined in terms of rationality relations, including subgame perfect equilibria and trembling hand perfect equilibria (introduced in Section \ref{sec:equilibrium_refinements}).\footnote{It would also be possible to define $\R$ such that, for example, every agent randomises their actions at every decision, or chooses the action that minimises their expected utility. Whilst it would be hard to view such behaviour as `rational' in the game-theoretic sense, one may think of $\R$ as defining the \emph{degree(s)} of rationality in a game.}

It will often be visually useful to restrict a mechanised MAID to just the mechanism variables, as can be seen in Figure \ref{fig:mechanised_game:mec_graph}. We can then view a mechanised graph as simply the composition of the original MAID and this graph of mechanism variables, via the addition of edges $\mecvar_{V} \to V$ for each $V \in \bm{V}$. 

\begin{figure}[h]
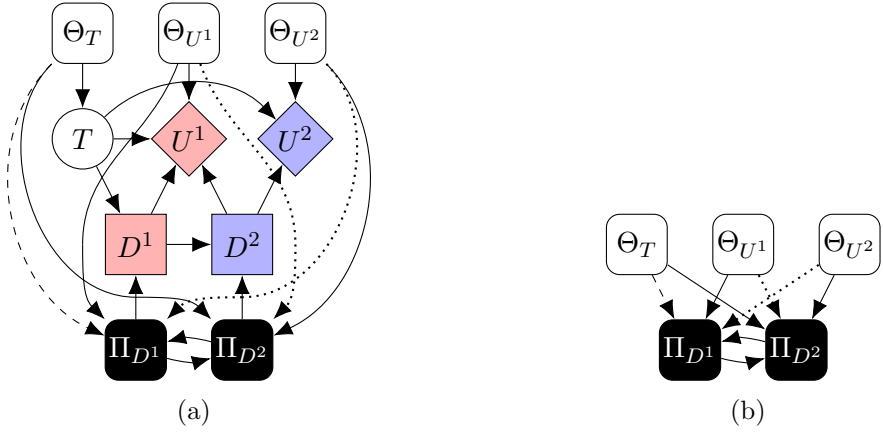

    \centering
    \begin{subfigure}[b]{0.45\linewidth}
        \centering
        \begin{influence-diagram}
        \node (X) [] {$T$};
        \node (U1) [utility, right = of X, player1] {$U^1$};
        \node (U2) [utility, right = of U1, player2] {$U^2$};
        \node (D1) [decision, below right = 1.4cm and 0.7cm of X, player1] {$D^1$};
        \node (D2) [decision, right = of D1, player2] {$D^2$};
        \edge [information] {X} {D1};
        \edge [information] {D1} {D2};
        \path (X) edge[->, bend left=45] (U2);
        \edge {X} {U1};
        \edge {D1} {U1};
        \edge {D2} {U1,U2};
        \node (X_mec) [relevancew, above = of X] {$\Theta_T$};
        \node (U2_mec) [relevancew, above = of U2] {$\Theta_{U^2}$};
        \node (D1_mec) [relevanceb, below = of D1] {$\Pi_{D^1}$};
        \node (U1_mec) [relevancew, above = of U1] {$\Theta_{U^1}$};
        \node (D2_mec) [relevanceb, below = of D2] {$\Pi_{D^2}$};
        \edge {X_mec} {X};
        \edge {D1_mec} {D1};
        \edge {D2_mec} {D2};
        \edge {U1_mec} {U1};
        \edge {U2_mec} {U2};
        \node (space1) [minimum size=0mm, node distance=2mm, below = 0.7cm of D1, draw=none] {};
        \draw (X_mec) edge[,in=180,out=-135] (space1.center)
        (space1.center) edge[,->,out=0,in=135] (D2_mec);
        \node (space2) [minimum size=0mm, node distance=2mm, below = 0.7cm of D2, draw=none] {};
        \draw (U2_mec) edge[thick, dotted, in=0,out=-45] (space2.center)
        (space2.center) edge[thick, dotted, ->,out=180,in=45] (D1_mec);
        \node (space3) [minimum size=0mm, node distance=2mm, left = 0.7cm of D1, draw=none] {};
        \draw (U1_mec) edge[in=90,out=-110] (space3.center)
        (space3.center) edge[->,out=-90,in=135] (D1_mec);
        \node (space4) [minimum size=0mm, node distance=2mm, right = 0.7cm of D2, draw=none] {};
        \draw (U1_mec) edge[thick, dotted, in=90,out=-70] (space4.center)
        (space4.center) edge[thick, dotted, ->,out=-90,in=45] (D2_mec);
        \path (D1_mec) edge[->, bend right=15] (D2_mec);
        \path (D2_mec) edge[->, bend right=15] (D1_mec);
        \draw (X_mec) edge[dashed, ->,out=-135,in=155] (D1_mec);
        \draw (U2_mec) edge[->, out=-45, in=25] (D2_mec);
    \end{influence-diagram}
    \caption{}
    \label{fig:mechanised_game:maid}
    \end{subfigure}
    \begin{subfigure}[b]{0.45\linewidth}
        \centering
        \begin{influence-diagram}
        \node (X_mec) [relevancew] {$\Theta_T$};
        \node (U1_mec) [relevancew, right = of X_mec] {$\Theta_{U^1}$};
        \node (U2_mec) [relevancew, right = of U1_mec] {$\Theta_{U^2}$};
        \node (D1_mec) [relevanceb, below right = 1.4cm and 0.7cm of X_mec] {$\Pi_{D^1}$};
        \node (D2_mec) [relevanceb, right = of D1_mec] {$\Pi_{D^2}$};
        \path (D1_mec) edge[->, bend right=15] (D2_mec);
        \path (D2_mec) edge[->, bend right=15] (D1_mec);
        \edge {X_mec} {D2_mec};
        \edge [dashed] {X_mec} {D1_mec};
        \edge {U1_mec} {D1_mec};
        \edge {U2_mec} {D2_mec};
        \edge [thick, dotted] {U1_mec} {D2_mec};
        \edge [thick, dotted] {U2_mec} {D1_mec};
        \end{influence-diagram}
        \caption{}
        \label{fig:mechanised_game:mec_graph}
    \end{subfigure}
    \caption{(a) A mechanised graph $\mec{\graph}$ representing Example \ref{ex:job_market}.
    Dotted edges represent those that are present in neither the $\R^{\BR}$-minimal nor the $s$-minimal mechanised graph. Dashed edges represent just those that are not present in the $s$-minimal mechanised graph. (b) The $\R^{\BR}$-relevance graph is formed by removing the dotted edges, and the $s$-relevance graph by further removing the dashed edge.}
    \label{fig:mechanised_game}
\end{figure}

\subsubsection*{Rational Outcomes}

If all of the rationality relations $\R$ in a MAID are satisfied by $\bm{\pi}$, then we say that $\bm{\pi}$ is an $\R$-\emph{rational outcome} of the game. For example, the $\R^{\BR}$-rational outcomes of the MAID $\model$ representing Example \ref{ex:job_market} are simply the {NEs} of $\model$. {Note that $\R$-rationality is not merely a convenience when reasoning about games, rather, it performs a necessary role by fully characterising the process by which decision rules are generated. As we will see later, the nature of the chosen rationality relations also allows us to deduce various facts about the game, purely from its graphical structure. In essence, such results hinge on the links between how agents choose their decision rules, and how the object-level variables depend on one another.}

\begin{definition}
    \label{def:rational_outcome}
    Given a mechanised MAID $\mec{\model}$, we say that any $\pi_D \in r_D(\pa_{\Pi_D})$ is an $\R$-\textbf{rational response} to $\pa_{\Pi_D}$, and that a (partial) policy profile ${\bm{\pi}}_{\bm{D}'}$ is $\R$-\textbf{rational} if $\pi_D \in r_D(\pa_{\Pi_D})$ for every $D \in \bm{D}'$. The set of (full) $\R$-rational policy profiles in $\mec{\model}$ are the $\R$-\textbf{rational outcomes} of the game {(where} we tend to drop $\R$ from the terms above when unambiguous) and is denoted by $\R(\mec{\model})$\label{def:ratoutcomes}.
\end{definition}

In general, the $\R$-rational outcomes in $\mec{\model}$ therefore define a \emph{set} of distributions $\{\Pr^{\bm{\pi}}\}_{{\bm{\pi}} \in \R(\mec{\model})}$ over the variables $\bm{V}$ where $\Pr^{\bm{\pi}}(\bm{v} ; \bm{\theta}) = \Pr(\bm{v} \mid \mecvals)$ for $\mecvals_{\bm{D}} = {\bm{\pi}}$ and $\mecvals_{\bm{V} \setminus \bm{D}} = \bm{\theta}$, which can be extended over the mechanised MAID as $\Pr(\bm{v}, \mecvals) = \Pr(\bm{v} \mid \mecvals)\delta(\mecvars {,} \mecvals)$.
We can thus view a mechanised MAID as a set of BNs induced by the $\R$-rational outcomes.\footnote{We continue, however, to highlight different variable types in our diagrams (rather than drawing each variable as a chance variable) for clarity of exposition.}
The key point here is that rationality relations form a principled, formal, and highly general way to model the inherent non-determinism in a game, in order to render the model suitable for causal reasoning.

Note that Definition \ref{def:rational_outcome} is related to the notion of a \emph{solution} in cyclic causal models \cite{Bongers2016,Ahsan2022}. In both formalisms, a solution corresponds to a joint probability distribution consistent with all cyclic relationships, and in general a model may have many or no solutions. This, in turn, will impact the answers to the causal queries that we might wish to ask (as we explain further in Section \ref{sec:causality_in_games}). Similarly, we can view the rational responses at a decision variable as a \emph{credal set} \cite{Levi1980}, and therefore the resultant model as a form of credal network \cite{Cozman2000}, where the imprecise probabilities arise from the fact that there may be more than one rational outcome in a game. %

\begin{remark}
    Often, agents might only be \emph{boundedly} rational \cite{Simon1957}. For example, due to computational costs, agents might seek only to satisfice (rather than optimise) their expected utility, leading to $\epsilon$-approximate variations of equilibria concepts in which each agent cannot improve their expected utility by more than $\epsilon > 0$ by deviating \cite{daskalakis2006note}. These bounded rationality conditions can also be captured using rationality relations.
\end{remark}

\subsection{Relevance}
\label{sec:relevance}

A natural question to ask is which edges in the mechanised graph $\mec{\graph}$ are necessary. In other words, which elements of $\Pa_{\Pi_D}$ are necessary for computing the rational response $r_D(\pa_{\Pi_D})$? In the mechanised game shown in Figure \ref{fig:mechanised_game:maid}, for example, we don't need to know $\theta_{U^2}$ in order to compute the set of best responses $\pi_{D^1}$ if we are given $\pi_{D^2}$. 
Whenever a mechanism variable $\mecvar_{V}$ is found to be irrelevant to $\Pi_D$ in this sense, we may remove the edge $\mecvar_{V} \to \Pi_D$ to make this independence explicit. Removing all irrelevant edges results in a subgraph $\mecr{\graph}$ of $\mec{\graph}$ that is \emph{minimal} with respect to $\R$. 

In the extreme case where all decision rules are chosen independently from all other mechanism variables, then all edges $\mathscr{E}'$ between mechanism variables can be pruned\label{def:minmec}. We call this subgraph of $\mec{\graph}$ the \emph{independent mechanised graph}\footnote{\citet{dawid2002influence} call this diagram the extended influence diagram, but they do not investigate the dependencies and relationships between mechanisms.} $\meczero{\graph}$\label{def:indextmaid}. Game-theoretic models thereby break the \emph{independent causal mechanism} assumption, which states that mechanisms should be causally and probabilistically independent of each other \cite{peters2017elements}.

Formally, we have the following definitions, which are similar in spirit to those proposed in earlier work \cite{koller2003multi,Nielsen1999,milch2008ignorable}, but differ in that they consider the case of non-deterministic rationality relations that encode the solutions of a game.

\begin{definition}
    \label{def:relevance}
    Given a mechanised MAID $\mec{\model}$ with rationality relations $\R$, $\mecvar_{V} \in \Pa_{\Pi_D}$ is $\R$-\textbf{relevant} to $\Pi_D$ if there exists $\pa_{\Pi_D} \neq \pa_{\Pi_D}'$ such that $r_D(\pa_{\Pi_D}) \neq r_D(\pa_{\Pi_D}')$, where $\pa_{\Pi_D}$ and $\pa_{\Pi_D}'$ differ only on $\mecvar_{V}$.
\end{definition}

\begin{definition}
    \label{def:minimality}
    Given a mechanised MAID $\mec{\model}$ with rationality relations $\R$, we say that its mechanised graph $\mec{\graph}$ is $\R$-\textbf{minimal}, denoted by $\mecr{\graph}$, when it contains an edge $\mecvar_{V} \to \Pi_D$ if and only if $\mecvar_{V}$ is $\R$-relevant to $\Pi_D$. When $\mecr{\graph}$ is restricted to the mechanism variables $\mecvars$, we refer to it as the $\R$-\textbf{relevance graph}, denoted $\relr{\graph}$.\label{def:rel-graph}
\end{definition}

Definition \ref{def:relevance} is a property of a \emph{game} after mechanising it via rationality relations. Therefore, an immediate follow-up question is whether given some choice of $\R$, we can identify $\R$-relevance (and hence prune edges in order to find the $\R$-minimal graph) for any MAID $\model = (\graph, \bm{\theta})$ simply by appealing to its underlying \emph{graph} $\graph$, and the form of $\R$. For many natural choices of $\R$ this is indeed the case, and we can derive sound and complete graphical criteria for identifying $\R$-relevance. For a given $\R$, we refer to these criteria as $\R$\emph{-reachability}. For example, the graphical criteria defining $\R^{\BR}$-reachability in Proposition \ref{prop:R^BR-reachability} enable us to remove the dotted edges $\Theta_{U^1} \to \Pi_{D^2}$ and $\Theta_{U^2} \to \Pi_{D^1}$ in Figure \ref{fig:mechanised_game:maid}.

\begin{proposition}
    \label{prop:R^BR-reachability}
    $\mecvar_{V}$ is $\R^{\BR}$-relevant to $\Pi_{D}$ if and only if $\mecvar_{V} \not\perp_{\meczero{\graph}} \bm{U}^i \cap \Desc_{D} \mid D, \Pa_{D}$ or $\mecvar_{V} \not\perp_{\meczero{\graph}} \Pa_D$, where if $D \in \bm{D}^i$, then $\bm{U}^i \cap \Desc_{D} \neq \varnothing$.
\end{proposition}

By applying $\R$-reachability to $\mec\graph$, we can find the $\R$-minimal mechanised graph $\mecr{\graph}$ and thus the $\R$-relevance graph $\relr{\graph}$. This latter object will be one that we make repeated use of, and can be viewed as a generalisation of K\&M's concept of a relevance graph simpliciter, to include all mechanism variables (as opposed to just those for decision variables) and for use with any $\R$.\footnote{Note that the direction of the edges of the relevance graphs in this paper are the same as in \cite{koller2003multi}, but are reversed compared with \cite{koller2001multi} and \cite{Hammond2021}. K\&M also refer to $V$ being relevant to/reachable from $D$, whereas we phrase this in terms of $\mecvar_V$ being relevant to/reachable from $\Pi_D$.} {To avoid confusion, we refer to K\&M's relevance graph as an \emph{$s$-relevance} graph. By generalising $s$-relevance (Definition \ref{def:strategic-relevance}) and $s$-reachability (Proposition \ref{prop:s-reachability}) to all mechanism variables, we can see in Figure \ref{fig:mechanised_game:maid}, for example, that $\Theta_T$ is not $s$-relevant to $\Pi_{D^1}$, though it is $\R^\BR$-relevant.}

Our generalisation of the idea of $s$-relevance to different rationality relations parallels our development of various equilibrium refinements within MAIDs (introduced in Section \ref{sec:equilibrium_refinements}), as opposed to the singular concept introduced by K\&M. For instance, we shall see in Section \ref{sec:equilibrium_refinements} that $s$-relevance corresponds to the concept of subgame perfectness. Moreover, use of the full mechanised graph as opposed to simply its restriction to the decision rule variables is not only important for reasoning about game-theoretic notions such as equilibria and subgames, but will also be critical for reasoning about causality in games where agents may adapt their policies in response to interventions on any mechanism variable (not just those representing decision rules).

\subsubsection*{Generalising the Soundness and Completeness Results}
We conclude this section by noting that it is possible to generalise the proof procedures for Propositions \ref{prop:R^BR-reachability} and \ref{prop:s-reachability} to other choices of $\R$. The key step in doing so is to identify, for any decision variable $D$, a `sufficient' (sound) but `minimal' (complete) set of queries $\mathcal{Q}_D \subseteq \mathcal{Q}$ whose values, given some $\mecvals = ({\bm{\pi}}, \bm{\theta})$, completely determine each $r_{D}$.
Formally, let $\mathcal{Q}(\mecvals)$\label{def:queries} be the set of probabilistic queries $\Pr^{\bm{\pi}}( \bm{x} \mid \bm{y} ; \bm{\theta})$ over $\bm{X}, \bm{Y} \subseteq \bm{V}$ that can be formed in a MAID given $\mecvals = ({\bm{\pi}}, \bm{\theta})$. Note that we can use such queries to express expected utilities, among many other things.
For any decision variable $D$, we are looking for a sufficient but minimal subset of queries $\mathcal{Q}_D \subseteq \mathcal{Q}$ and a truth-valued function $g_D$ such that:
\begin{equation}
    \label{eq:relation_restriction}
    \pi_{D} \in r_{D}(\mecvals_{\bm{V}\setminus \{D\}}) ~~\Leftrightarrow~~ g_D\big(\mathcal{Q}_D(\mecvals), \dom ( \bm{V} ) \big).
\end{equation}
In this setting, asking whether $\mecvar_{V}$ is $\R$-relevant to $\Pi_D$ reduces to asking whether the choice of $\mecval_V$ may affect some $\Pr^{\bm{\pi}}( \bm{x} \mid \bm{y}) \in \mathcal{Q}_D$, which is, in turn, equivalent to asking whether $V$ is a \emph{requisite probability node} for $\Pr^{\bm{\pi}}( \bm{x} \mid \bm{y})$ {\cite{shachter2013bayes}, introduced in Definition \ref{def:requisite}.} Whether a variable is a requisite probability node can be determined using a well-established graphical criterion {\cite{Geiger1990}, given as Lemma \ref{lem:reachability}.}

We can use this criterion to establish graphical criteria for $\R$-reachability, which, given a judicious choice of $\mathcal{Q}_D$, will be sound and complete. For example, the graphical criteria in Propositions \ref{prop:s-reachability} ($s$-reachability) and \ref{prop:R^BR-reachability} ($\R^{\BR}$-reachability) correspond to choices: 
\begin{align*}
    \mathcal{Q}^s_D &= \{ \Pr^{{\bm{\pi}}}(\bm{u}^i \cap \desc_{D} \mid d, \pa_{D}) \},\\
    \mathcal{Q}^\BR_D &= \{ \Pr^{{\bm{\pi}}}(\bm{u}^i \cap \desc_{D} \mid d, \pa_{D}), \Pr^{\bm{\pi}}(\pa_{D}) \},
\end{align*}
respectively. With such a soundness and completeness result in hand, we can identify an $\R$-relevance graph as follows:
\begin{center}
    \begin{tabular}{l l l}
        & $\relr{\graph}$ contains an edge $\mecvar_{V} \to \Pi_D$ & \\
        $\Leftrightarrow$ & $\mecvar_{V}$ is $\R$-relevant to $\Pi_D$ & by Definition \ref{def:minimality}\\
        $\Leftrightarrow$ & $r_D(\pa_{\Pi_D})$ may vary with the value of $\mecvar_{V}$ & by Definition \ref{def:relevance}\\
        $\Leftrightarrow$ & $\mathcal{Q}_D(\pa_{\Pi_D}, \pi_D)$ may vary with the value of $\mecvar_{V}$ & by (\ref{eq:relation_restriction})\\
        $\Leftrightarrow$ & $V$ is a requisite probability node for some $\Pr^{\bm{\pi}}( \bm{x} \mid \bm{y}) \in \mathcal{Q}_D$ & by Definition \ref{def:requisite}\\
        $\Leftrightarrow$ & $\mecvar_{V} \not\perp_{\meczero{\graph}} \bm{X} \mid \bm{Y}$ for some $\Pr^{\bm{\pi}}( \bm{x} \mid \bm{y}) \in \mathcal{Q}_D$ & by Lemma \ref{lem:reachability}
    \end{tabular}
\end{center}

\section{Causality in Games}
\label{sec:causality_in_games}

{Many types of queries may be of interest}  in game-theoretic scenarios. For instance, recalling Example \ref{ex:job_market}, we could ask questions corresponding to:

\begin{itemize}
    \item[1)] Predictions, such as a)\label{query:1a} `Given that the worker went to university, what is their wellbeing?' or b)\label{query:1b} `Given that the worker always decides to go to university, what is their wellbeing?'
    \item[2)] Interventions, such as a)\label{query:2a} `Given that the worker is forced to go to university, what is their wellbeing?' or b)\label{query:2b} `Given that the worker goes to university if and only if they are selected via a lottery system, what is their wellbeing?'
    \item[3)] Counterfactuals, such as a)\label{query:3a} `Given that the worker didn’t go to university, what would be their wellbeing if they had?' or b)\label{query:3b} `Given that the worker never decides to go to university, what would be their wellbeing if they always decided to go to university?'
\end{itemize}

Although these queries are ostensibly similar, they belong to different levels of the causal hierarchy. 
Predictions can be answered using (mechanised) MAIDs, which, as associational models, reside on level one. 
However, in order to reason about interventions and counterfactuals in games, we require models on levels two and three respectively. To this end, we introduce \emph{causal games} (CGs) and \emph{structural causal games} (SCGs). These can be viewed as generalising MAIDs to the causal setting, or CBNs and SCMs to the game-theoretic setting. 

\begin{figure}[h]
    \centering
    \begin{subfigure}[t]{0.4\linewidth}
        \centering
        \renewcommand{\arraystretch}{2}
        \begin{tabular}[t]{c c c}
            \tikzmark{v_end} SCM & SCIM & SCG\\
            CBN & CID & CG \\
            B\tikzmark{start}N & ID & MA\tikzmark{h_end}ID\\
        \end{tabular}
        \renewcommand{\arraystretch}{1}
        \begin{tikzpicture}[overlay,remember picture]
            \draw[->] let \p1=(start), \p2=(v_end) in ($(\x1,\y1)+(-1.25,0)$) -- node[label=left:level] {} ($(\x1,\y2)+(-1.25,0)$);
            \draw[->] let \p1=(start), \p2=(h_end) in ($(\x1,\y1)+(0,-0.75)$) -- node[label=below:agents] {} ($(\x2,\y1)+(0,-0.75)$);
        \end{tikzpicture}
    \end{subfigure}
    \begin{subfigure}[t]{0.3\linewidth}
        \centering
        \vspace{-0.5cm}
        \begin{tabular}[t]{l l}
            B & Bayesian\\
            C & Causal\\
            G & Game\\
            I & Influence\\
            M & Model\\
            MA & Multi-Agent\\
            N & Network\\
            S & Structural
        \end{tabular}
    \end{subfigure}
    \caption{{In this paper, we introduce CGs and SCGs}. The causal hierarchy (associational, interventional, and counterfactual) forms the {vertical} axis and the number of agents ($0$, $1$, and $n$) forms the {horizontal} axis. {Note that all models in this diagram can also be mechanised.}
    }
    \label{fig:names}
\end{figure}
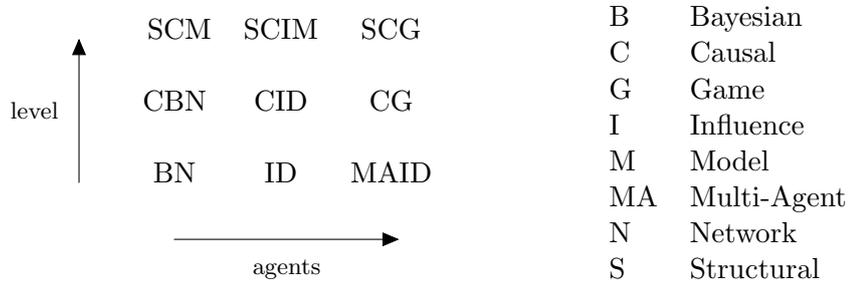

The connections between all of these models and their various acronyms are displayed in Figure \ref{fig:names}. Importantly, the models in each row are a special case of those in the row below (and the models in each column generalise those in the column to its left). As such, everything defined with respect to MAIDs (such as the mechanised games of Definition \ref{def:mechanised_maid} and the equilibrium refinements we introduce in Section \ref{sec:subgames_eqs}) are also well-defined in both CGs and SCGs.

{Note that as well as asking queries regarding the object-level variables -- such as queries \hyperref[query:1a]{1a}, \hyperref[query:2a]{2a}, and \hyperref[query:3a]{3a} -- \emph{after} some policy profile $\bm{\pi}$ has been chosen, mechanised games allow us to also ask queries regarding the mechanism variables -- such as queries \hyperref[query:1b]{1b}, \hyperref[query:2b]{2b}, and \hyperref[query:3b]{3b} -- \emph{before} fixing a policy. While this distinction between `post-policy' and `pre-policy' queries is less significant in the case of predictions, we shall see in Section \ref{sec:interventions} that this difference corresponds to whether agents can adjust their policies in response to an intervention or not.}

{This distinction, which is critical in game-theoretic settings, does not apply to standard causal models, as they do not contain strategic, decision-making agents. Similarly, standard game-theoretic models do not explicitly represent this distinction either, as the dependencies between agents' decisions and other aspects of the game are not explicitly captured. Computationally, pre-policy interventions correspond to altering the original game then calculating the outcomes, and post-policy interventions correspond to calculating the outcomes of the original game, and then altering them. In a sense, these latter queries are a more natural analogue of existing work; in this work we unify both types of query using the same formalism.}

In the remainder of this section, we show how the six types of causal queries represented by the examples above (each of which is written formally in \cref{table:queries}) can be answered using causal games. Doing so requires us to identify which level of the causal hierarchy a query belongs to, and to assess whether the query is pre-policy or post-policy. It {is} also simple to combine pre- and post-policy queries using mechanised games, though for clarity {and brevity} we do not do so here.

\setcounter{table}{0}
\begin{table}[h]
    \centering
    \begin{tabular}{p{3cm}p{3cm}p{3cm}p{3cm}}
        \toprule
            & 1) Prediction & 2) Intervention & 3) Counterfactual\\
        \midrule
        a) Post-policy & $\Pr^{\bm{\pi}}(u^1 \mid g)$ & $\Pr^{\bm{\pi}}(u^1_g)$ & $\Pr^{\bm{\pi}}(u^1_g \mid \neg g)$\\
        b) Pre-policy & $\Pr(u^1 \mid \bar{\pi}_{D^1})$ & $\Pr(u^1_{\hat{\pi}_{D^1}})$ &  $\Pr(u^1_{\bar{\pi}_{D^1}} \mid \tilde{\pi}_{D^1})$\\
        \bottomrule
    \end{tabular}
    \caption{Examples of the queries we may ask in causal models of games, corresponding to those listed at the top of this section.
    Recall that $\dom(D^1) = \{g, \neg g \}$, indicating whether or not the worker goes to university, and that $\bar{\pi}_{D^1}, \hat{\pi}_{D^1}, \tilde{\pi}_{D^1}$ denote possible values of the decision rule variable $\Pi_{D^1}$. {For example, $\bar{\pi}_{D^1}$ in query \hyperref[query:1b]{1b} is the decision rule in which the worker always decides to go to university, i.e., $\bar{\pi}_{D^1}(D_1 \mid T) = \delta(D_1, g)$. As introduced in Section \ref{sec:background}, $\Pr(\bm{x}_{\bm{y}})$ denotes the probability of $\bm{x}$ given a hard intervention $\do(\bm{Y} = \bm{y})$ and $\Pr(\bm{x}_{\bm{y}} \mid \bm{z})$ represents the counterfactual probability of $\bm{x}$ had $\bm{y}$ been true, given that (in fact) $\bm{z}$ was true.} 
    }
    \label{table:queries}
\end{table}

\subsection{Predictions}
\label{sec:predictions}

{Each policy profile ${\bm{\pi}}$ in a MAID $\model$ induces a BN with joint distribution $\Pr^{\bm{\pi}}(\bm{V}; \bm{\theta})$. We can therefore easily compute the probability of $\bm{x}$ given some observation $\bm{z}$, written $\Pr^{\bm{\pi}}(\bm{x} \mid \bm{z})$, under a given policy profile $\bm{\pi}$. Note that the distribution $\Pr^{\bm{\pi}}(\bm{V}; \bm{\theta})$ in the MAID can equivalently be written in the mechanised MAID as $\Pr(\bm{V} \mid \mecvals)$ with $\mecvals=(\bm{\pi}, \bm{\theta})$.}

However, in game-theoretic settings, we typically assume only that a \emph{rational outcome} of the game will be chosen, not some unique ${\bm{\pi}}$. Moreover, we may not have any reason to favour one rational outcome over another, implying that we ought to evaluate queries with respect to a \emph{set} of policy profiles. In the definition below, we therefore consider the distribution $\Pr^{\bm{\pi}}(\bm{x} \mid \bm{z})$ induced by each rational outcome that is consistent with the observation $\bm{z}$.
Note that, as remarked in Section \ref{sec:mechanised_maids}, there may be many or no such rational outcomes.
\begin{definition}
    \label{def:prediction_query}
    Given a mechanised MAID $\mec{\model}$ with rationality relations $\R$, the \textbf{answer to a conditional query} of the probability of $\bm{x}$ given observation $\bm{z}$ is given by the set $\Pr^{\R}(\bm{x} \mid \bm{z}) \coloneqq \big\{ \Pr^{\bm{\pi}}(\bm{x} \mid \bm{z}) \big\}_{{\bm{\pi}} \in \R(\mec{\model} \mid \bm{z})}$ where $\R(\mec{\model} \mid \bm{z}) \coloneqq \{{\bm{\pi}} \in \R(\mec{\model}) : \Pr^{\bm{\pi}}(\bm{z}) > 0\}$ is the set of \textbf{conditional rational outcomes}. In general, $\bm{Z} \subseteq \bm{V} \cup \mecvars$ can include mechanism variables, and so we compute $\Pr^{\bm{\pi}}(\bm{x} \mid \bm{z})$ in $\model$ as $\Pr(\bm{x} \mid \bm{z}, \mecvals')$ in $\mec{\model}$, where $\mecvars' = \mecvars \setminus \bm{Z}$ and $\mecvals_{\bm{D}} = {\bm{\pi}}$.
\end{definition}

    More generally, we can view queries in games as \emph{first-order} queries defined over formulae in which ${\bm{\pi}}$ is a free variable, such as $\varphi({\bm{\pi}}) \equiv \Pr^{\bm{\pi}}(\bm{x} \mid \bm{z})$. The answers to these queries are therefore only well-defined when this free variable becomes bound, or when {considering a set of answers, as in Definition \ref{def:prediction_query}. For example, we can} bind ${\bm{\pi}}$ by quantifying over it as in $\exists {\bm{\pi}} \in \R(\mec{\model} \mid \bm{z}).\, \big( \varphi({\bm{\pi}}) \sim q \big)$ where $\sim \, \in \{<, \leq, =, \geq, > \}$ and $q \in [0,1]$ (as is often done in first-order logics for reasoning about and verifying multi-agent systems \cite{Chatterjee2010,Wooldridge2016,Kwiatkowska2020a}), or by returning bounds such as $\max_{_{{\bm{\pi}} \in \R(\mec{\model} \mid \bm{z})}} \varphi({\bm{\pi}})$ (as is often done in credal networks \cite{Cozman2000}). {Above, we quantify over $\R(\mec{\model} \mid \bm{z})$, but it is also possible to quantify over any desired set of policies, or even to posit a prior distribution $\Pr({\bm{\Pi}})$ over policies. These queries strictly generalise those in BNs and models such as settable systems \cite{White2009,White2014}, which effectively consider only a single instantiation of $\bm{\Pi}$.}

By way of illustration, let us return to Example \ref{ex:job_market} and query \hyperref[query:1a]{1a} from earlier in this section. We can interpret this question as asking about the expected utility of the worker given the observation that they went to university, written $\expect_{\bm{\pi}} [U^1 \mid g]$ for some policy ${\bm{\pi}}$. For the worked examples in this section, we assume that $p = \Pr(T=h) = \frac{1}{2}$ and that the automated hiring system and the worker are playing best responses to one another, i.e., $\R = \R^{\BR}$. The various rational outcomes induced by this choice (i.e., the NEs of the game) are:
{\begin{enumerate}
    \item The worker always chooses $\neg g$. The hiring system chooses $j$ if the worker chose $\neg g$, otherwise they choose $j$ with any probability $q \in [0,1]$.
    \item The worker chooses $\neg g$ if $T = \neg h$, and otherwise chooses $g$ with probability $\frac{1}{2}$. The hiring system chooses $j$ if the worker chose $g$, otherwise they choose $j$ with probability $\frac{4}{5}$.
    \item The worker always chooses $g$. The hiring system chooses $j$ if the worker chose $g$, otherwise they choose $j$ with any probability $q \in [0,\frac{3}{5}]$.
\end{enumerate}}
{The expected utilities for the worker under these rational outcomes are 5, 4, and $\frac{7}{2}$, respectively}. In query \hyperref[query:1a]{1a}, the conditional rational outcomes $\R(\mec{\model} \mid g)$ are the NEs consistent with observing $D^1 = g$. To answer the query, we must therefore compute $\Pr^{\R}(u^1 \mid g)$, which yields $\{ \expect_{\bm{\pi}} [U^1 \mid g] \}_{{\bm{\pi}} \in \R(\mec{\model} \mid g)} = \{\frac{7}{2}, 4\}$.

As noted above, Definition \ref{def:prediction_query} can also be employed when the observations made are of the \emph{mechanism variables}. For example, in query \hyperref[query:1b]{1b} we condition on the observation that the worker's strategy is given by $\delta(D^1 {,} g)$, which we refer to as $\bar{\pi}_{D^1}$, i..e, the worker \emph{always} decides to go to university. Based on this, we can use the mechanised MAID to compute $\Pr^{\R}(u^1 \mid \bar{\pi}_{D^1})$ and thus that $\{ \expect [ U^1 \mid \bar{\pi}_{D^1} ] \}_{{\bm{\pi}} \in \R(\mec{\model} \mid \bar{\pi}_{D^1})} = \{ \frac{7}{2} \}$. Note that this set is distinct from the answer to query \hyperref[query:1a]{1a} as observations of mechanism and object-level variables provide us with different information, i.e.,\ $\R(\mec{\model} \mid \bar{\pi}_{D^1}) \neq \R(\mec{\model} \mid g)$. In particular, observations of mechanism variables serve primarily to rule out certain rational outcomes by conditioning on decision rules (conditioning on \emph{parameter} variables tells us nothing as they have deterministic distributions and no parents).

One advantage of computing predictions in MAIDs (as opposed to in EFGs, for instance) is that we may exploit the conditional independencies in the graph. For example, if we were interested in how likely a worker is to be hard-working given that they went to university and were hired, then $T \perp_{\graph} D^2 \mid D^1$ implies that $\Pr^{\bm{\pi}}(h \mid g, j) = \Pr^{\bm{\pi}}(h \mid g)$ for any policy profile ${\bm{\pi}}$. When answering queries over object-level variables using mechanised MAIDs, we implicitly condition on the values of the mechanism variables to represent the fact the game and policy under consideration are fixed. For example, the query $\Pr^{\bm{\pi}}(h \mid g, j)$ in $\model$ is given by $\Pr(h \mid g, j, \mecvals)$ in $\mec{\model}$, where $\mecvals_{\bm{D}} = {\bm{\pi}}$. Hence, although $T \not\perp_{\mec{\graph}} D^2 \mid D^1$, we do have $T \perp_{\mec{\graph}} D^2 \mid D^1, \mecvars$, as expected.

\subsection{Interventions}
\label{sec:interventions}

Interventional queries concern the effect of causal influences from outside a system. This becomes especially interesting in the case of games, when interventions affect not only the environment but also how the self-interested agents adapt their policies in response. In order to answer such queries, the edges in a MAID must reflect the causal structure of the world.
This gives rise to the following definition, which can be viewed simply as a CBN without parameters for the decision variables. Note that as causal games are a form of MAID, they also support the associational queries introduced in the preceding section (just as CBNs may also be used to compute both interventional and associational queries).

\begin{definition}
    \label{def:cg}
        A \textbf{causal game (CG)} $\model = (\graph, \bm{\theta})$ is a MAID such that for \emph{any} parameterisation of the decision variable CPDs $\bm{\pi}$, the induced model with joint distribution $\Pr^{\bm{\pi}}(\bm{V})$ is a CBN.
\end{definition}

Unlike CBNs, CGs let us ask about the effect of an intervention \emph{before} or \emph{after} a policy profile has been selected, which we refer to as \emph{pre-} and \emph{post-policy} queries respectively. Asking about the effect of an intervention after a particular policy profile ${\bm{\pi}}$ has been selected (as in query \hyperref[query:2a]{2a}) is simply the same as performing an interventional query on the CBN with joint distribution $\Pr^{\bm{\pi}}$. 
Asking about the effect of an intervention \emph{before} a policy profile has been selected (as in query \hyperref[query:2b]{2b}) means that agents are made aware of the intervention before selecting their decision rules, and thus they may react to its effects. In other words, the intervention can be viewed as producing a slightly different game that the agents then (knowingly) play.

Our key observation is that pre-policy interventions can be modelled as interventions on the \emph{mechanism variables} in the mechanised CG, which ensures that the effects are propagated through the processes via which agents select their decision rules. This is because the additional mechanism variables and their outgoing edges in a mechanised CG represent causal (though potentially non-deterministic) processes via which parameterisations for the object-level variables are selected.
Post-policy interventions, in turn, can be modelled as standard interventions on object-level variables.
We write $\mec{\model}_{\I}$ for an intervention $\I$ which may contain both pre-policy and post-policy interventions.

This unification of pre- and post-policy interventions is one of the key benefits of mechanised models.
Indeed, post-policy interventions, and pre-policy interventions on parameter variables, are defined exactly as in CBNs, while a pre-policy intervention on a decision rule variable $\Pi_D$ corresponds to replacing $r_D : \dom(\Pa_{\Pi_D}) \to \dom(\Pi_D)$ by some new $r^*_D : \dom(\Pa^*_{\Pi_D}) \to \dom(\Pi_D)$, where we may have $\Pa^*_{\Pi_D} \neq \Pa_{\Pi_D}$.
As for conditional queries, in our definition (which mirrors that introduced for cyclic causal models \cite{Bongers2016}, but where the cyclic dependencies are governed by relations instead of functions \cite{Ahsan2022}) we quantify over the set of rational outcomes that are consistent with the given intervention. 
\begin{definition}
    \label{def:interquery}
    Given a mechanised CG $\mec\model$ with rationality relations $\R$, the \textbf{answer to an interventional query} of the probability of $\bm{x}$ given intervention $\I$ on variables $\bm{Y}$ is given by the set $\Pr^{\R}(\bm{x}_{\I}) \coloneqq \big\{ \Pr^{\bm{\pi}}(\bm{x}_{\I}) \big\}_{{\bm{\pi}} \in \R(\mec{\model}_{\I})}$ where $\R(\mec{\model}_{\I})$ is the set of \textbf{interventional rational outcomes} in the mechanised MAID $\mec{\model}_{\I}$ with rationality relations $\R^* \coloneqq \{ r^*_D \}_{\Pi_{D} \in \bm{Y}} \cup \{ r_D \}_{\Pi_{D} \notin \bm{Y}}$ and parameters determined by $\Pr(\bm{\Theta}_{\I})$. Note that if $\I$ {is fully post-policy}, then the rational outcomes remain the same, i.e., $\R(\mec{\model}_{\I}) = \R(\mec{\model})$ when $\bm{Y} \subseteq \bm{V}$.
\end{definition}

To illustrate these ideas, let us return to Example \ref{ex:job_market} and queries \hyperref[query:2a]{2a} and \hyperref[query:2b]{2b}. Query \hyperref[query:2a]{2a} concerns a \emph{post-policy} intervention since the worker is forced to go to university unbeknownst to the firm's hiring system; in other words, the hiring system does not observe this fact before computing a decision rule. To compute the worker's expected utility, we must calculate $\Pr^{\R}(u^1_{g})$ and thus perform a hard intervention $\Do(D^1 = g)$ in the mechanised game (shown in Figure \ref{fig:CGintervention:1}). As the set of rational outcomes does not change under a post-policy intervention, we have that $\Pr^{\R}(u^1_{g}) = \big\{ \Pr^{\bm{\pi}}(u^1_{g}) \big\}_{{\bm{\pi}} \in \R(\mec{\model})}$, which results in $\big\{ \expect_{\bm{\pi}} [ U^1_g ] \big\}_{{\bm{\pi}} \in \R(\mec{\model})} = \{-\frac{3}{2}, \frac{7}{2}\}$. Note that unlike query \hyperref[query:1a]{1a}, the fact that $D^1 = g$ tells us nothing about the value of $T$, as it is causally upstream of the intervention on $D^1$. Therefore, the wellbeing of the worker may decrease because they may be sent to university even when they are lazy.

To answer query \hyperref[query:2b]{2b}, concerning a \emph{pre-policy} intervention, we must compute $\Pr(u^1_{\hat{\pi}_{D^1}})$ where $\hat{\pi}_{D^1}(g \mid h) = \hat{\pi}_{D^1}(g \mid \neg h) = \frac{1}{2}$ represents the aforementioned lottery system, which selects students to attend university randomly with probability $\frac{1}{2}$. This time, the firm modifies their hiring system after observing the intervention -- denoted by $\I$ and shown in Figure \ref{fig:CGintervention:2} -- but before the hiring system computes a decision rule, and so the new set of rational outcomes is given by $\R(\mec{\model}_{\I}) = \{ (\hat{\pi}_{D^1}, \pi_{D^2}) : \pi_{D^2} \in r_{D^2}(\pa_{\Pi_{D^2}}) \}$. In other words, we set $\Pi_{D^1} = \hat{\pi}_{D^1}$ using a hard intervention and then allow the hiring system to best respond to this decision rule using $r_{D^2}$. Note that $\R(\mec{\model}_{\I}) \neq \R(\mec{\model} \mid \hat{\pi}_{D^1}) = \varnothing$, as there is no NE in the game that contains decision rule $\hat{\pi}_{D^1}$. The lottery system removes any signalling effect of going to university, resulting in an optimal policy for the hiring system of always offering a job to the worker and expected utility $\big\{ \expect_{\bm{\pi}} [ U^1_{\hat{\pi}_{D^1}} ] \big\}_{{\bm{\pi}} \in \R(\mec{\model}_{\I})} = \{ \frac{17}{4} \}$.

\begin{figure}[h]
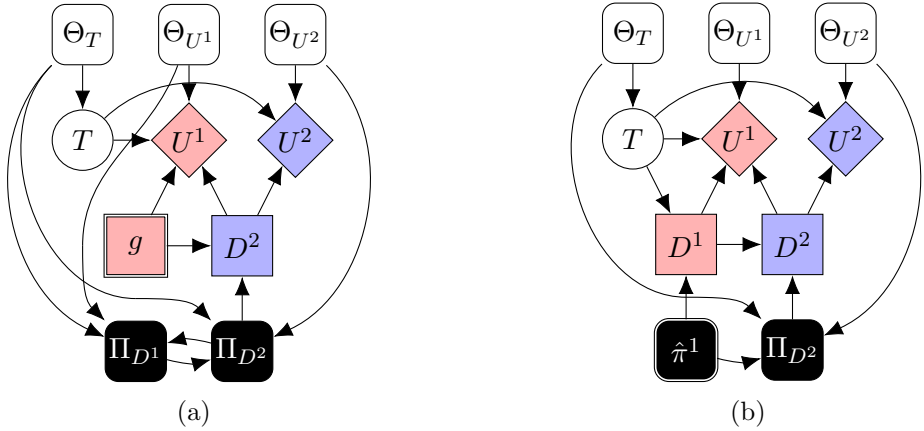

    \centering
    \begin{subfigure}[b]{0.45\linewidth}
        \centering
        \begin{influence-diagram}
        \node (X) [] {$T$};
        \node (U1) [utility, right = of X, player1] {$U^1$};
        \node (U2) [utility, right = of U1, player2] {$U^2$};
        \node (D1) [intervened, decision, below right = 1.4cm and 0.7cm of X, player1] {$g$};
        \node (D2) [decision, right = of D1, player2] {$D^2$};
        \edge [information] {D1} {D2};
        \path (X) edge[->, bend left=45] (U2);
        \edge {X} {U1};
        \edge {D1} {U1};
        \edge {D2} {U1,U2};
        \node (X_mec) [relevancew, above = of X] {$\Theta_T$};
        \node (U2_mec) [relevancew, above = of U2] {$\Theta_{U^2}$};
        \node (D1_mec) [relevanceb, below = of D1] {$\Pi_{D^1}$};
        \node (U1_mec) [relevancew, above = of U1] {$\Theta_{U^1}$};
        \node (D2_mec) [relevanceb, below = of D2] {$\Pi_{D^2}$};
        \edge {X_mec} {X};
        \edge {D2_mec} {D2};
        \edge {U1_mec} {U1};
        \edge {U2_mec} {U2};
        \node (space1) [minimum size=0mm, node distance=2mm, below = 0.7cm of D1, draw=none] {};
        \draw (X_mec) edge[,in=180,out=-135] (space1.center)
        (space1.center) edge[,->,out=0,in=135] (D2_mec);
        \node (space3) [minimum size=0mm, node distance=2mm, left = 0.7cm of D1, draw=none] {};
        \draw (U1_mec) edge[in=90,out=-110] (space3.center)
        (space3.center) edge[->,out=-90,in=135] (D1_mec);
        \path (D1_mec) edge[->, bend right=15] (D2_mec);
        \path (D2_mec) edge[->, bend right=15] (D1_mec);
        \draw (X_mec) edge[->,out=-135,in=155] (D1_mec);
        \draw (U2_mec) edge[->, out=-45, in=25] (D2_mec);
        \end{influence-diagram}
        \caption{}
        \label{fig:CGintervention:1}
    \end{subfigure}
    \begin{subfigure}[b]{0.45\linewidth}
        \centering
        \begin{influence-diagram}
            \node (X) [] {$T$};
            \node (U1) [utility, right = of X, player1] {$U^1$};
            \node (U2) [utility, right = of U1, player2] {$U^2$};
            \node (D1) [decision, below right = 1.4cm and 0.7cm of X, player1] {$D^1$};
            \node (D2) [decision, right = of D1, player2] {$D^2$};
            \edge [information] {X} {D1};
            \edge [information] {D1} {D2};
            \path (X) edge[->, bend left=45] (U2);
            \edge {X} {U1};
            \edge {D1} {U1};
            \edge {D2} {U1,U2};
            \node (X_mec) [relevancew, above = of X] {$\Theta_T$};
            \node (U2_mec) [relevancew, above = of U2] {$\Theta_{U^2}$};
            \node (D1_mec) [intervened, relevanceb, below = of D1] {$\hat{\pi}^1$};
            \node (U1_mec) [relevancew, above = of U1] {$\Theta_{U^1}$};
            \node (D2_mec) [relevanceb, below = of D2] {$\Pi_{D^2}$};
            \edge {X_mec} {X};
            \edge {D1_mec} {D1};
            \edge {D2_mec} {D2};
            \edge {U1_mec} {U1};
            \edge {U2_mec} {U2};
            \node (space1) [minimum size=0mm, node distance=2mm, below = 0.7cm of D1, draw=none] {};
            \draw (X_mec) edge[,in=180,out=-135] (space1.center)
            (space1.center) edge[,->,out=0,in=135] (D2_mec);
            \path (D1_mec) edge[->, bend right=15] (D2_mec);
            \draw (U2_mec) edge[->, out=-45, in=25] (D2_mec);
        \end{influence-diagram}
        \caption{}
        \label{fig:CGintervention:2}
    \end{subfigure}
    \caption{(a) A mechanised game showing the hard post-policy intervention $\Do(D^1 = g)$, where incoming edges to $D^1$ are severed. (b) A mechanised game showing the hard pre-policy intervention $\Do(\Pi_{D^1} = \hat{\pi}^1)$, where incoming edges to $\Pi_{D^1}$ are severed.}
    \label{fig:causality_in_games}
\end{figure}

\begin{remark}
    In previous work, soft interventions on object-level variables $V$ have been modelled as hard interventions on its mechanism variable $\mecvar_{V}$ \cite{dawid2002influence}. While these can be viewed as essentially equivalent in the single-agent case,\footnote{To be precise, this is true when the agent has \emph{sufficient recall}; see Definition \ref{def:sufrecall}.} possible dependencies between mechanism variables in mechanised games means that these two types of intervention may have markedly different effects in the multi-agent setting. The difference between pre- and post-policy interventions results from a difference in the \emph{information} that is available to agents when they make their decisions, rather than to the chronology of play in a game (as is also the case for the structure of EFGs).
\end{remark}

\subsection{Counterfactuals}
\label{sec:counterfactuals}

The final type of question we investigate arises when we combine predictions and interventions, as in query \hyperref[query:3a]{3a}: `Given that the worker didn't go to university, what would be their wellbeing if they had?'. Such questions are \emph{counterfactual}, as they combine observations made in the actual world (in which the worker didn't go to university), with questions pertaining to a counterfactual world (where they did go to university). Answering these queries in games is significantly more nuanced and complex than those of the preceding subsections. {To do so, we must first consign all stochasticity to a set of exogenous variables, one for each variable in the causal game.
Just as in an SCM, each variable $V$ is thus associated with an exogenous variable $\exovar_V$, and is governed by a deterministic CPD $\Pr^{\bm{\pi}}(V \mid \Pa_V)$, where $\exovar_V \in \Pa_V$.
}

\begin{definition}
    \label{def:scg}
    A (Markovian) \textbf{structural causal game (SCG)} $\model = (\graph, \bm{\theta})$ is a causal game over exogenous and endogenous variables $\exovars \cup \bm{V}$ such that for \emph{any} (deterministic) parameterisation of the decision variable CPDs $\dot{\bm{\pi}}$, the induced model with joint distribution $\Pr^{\dot{\bm{\pi}}}(\bm{V}, \exovars)$ is an SCM.
\end{definition}

{An SCG can be seen as an SCM without parameters for the decision variables. Given a policy $\bm{\pi}$, we recover an SCM, as we explain in more detail below. Meanwhile, mechanised SCGs can be viewed as (a special case of) a relational causal model with cycles \cite{Ahsan2022}.}

When we mechanise SCGs, although we introduce mechanism variables for the exogenous variables (as can be seen in Figure \ref{fig:maids:mechanised_maid}), we view them and their object-level exogenous children as beyond the realm of observation and intervention, just as in SCMs. As such, mechanism variables of exogenous variables can largely be ignored.
As in SCMs, interventional distributions $\I(V \mid \Pa^*_V)$ must be deterministic, and soft interventions {may be defined by introducing} a new exogenous variable $\exovar^*_V$ to $\Pa^*_V$ \cite{Correa2020}. With these caveats in place, both pre-policy and post-policy predictions and interventions may be defined in this model as described in \ref{sec:predictions} and \ref{sec:interventions} respectively.

\begin{figure}[h]
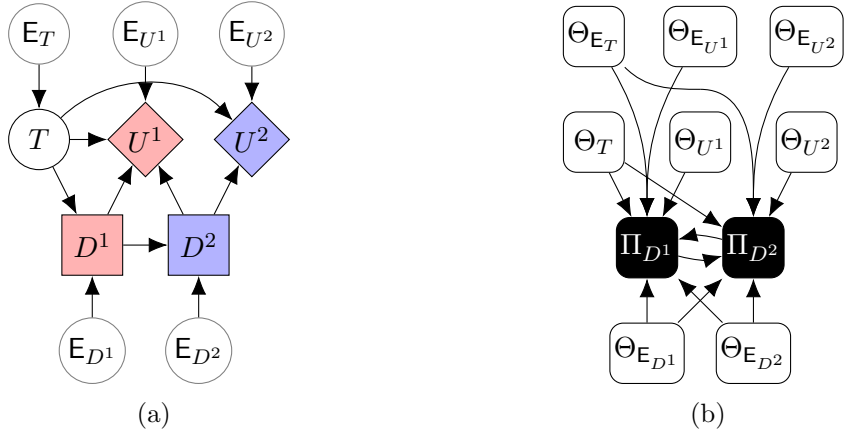

    \centering
    \begin{subfigure}[b]{0.45\linewidth}
        \centering
        \begin{influence-diagram}
        \node (X) [] {$T$};
        \node (U1) [utility, right = of X, player1] {$U^1$};
        \node (U2) [utility, right = of U1, player2] {$U^2$};
        \node (D1) [decision, below right = 1.4cm and 0.7cm of X, player1] {$D^1$};
        \node (D2) [decision, right = of D1, player2] {$D^2$};
        \edge [information] {X} {D1};
        \edge [information] {D1} {D2};
        \path (X) edge[->, bend left=45] (U2);
        \edge {X} {U1};
        \edge {D1} {U1};
        \edge {D2} {U1,U2};
        \node (X_exo) [exogenous, above = of X] {$\exovar_T$};
        \node (U2_exo) [exogenous, above = of U2] {$\exovar_{U^2}$};
        \node (D1_exo) [exogenous, below = of D1] {$\exovar_{D^1}$};
        \node (U1_exo) [exogenous, above = of U1] {$\exovar_{U^1}$};
        \node (D2_exo) [exogenous, below = of D2] {$\exovar_{D^2}$};
        \edge {X_exo} {X};
        \edge {D1_exo} {D1};
        \edge {D2_exo} {D2};
        \edge {U1_exo} {U1};
        \edge {U2_exo} {U2};
    \end{influence-diagram}
    \caption{}
    \label{fig:maids:maid}
    \end{subfigure}
    \begin{subfigure}[b]{0.45\linewidth}
        \centering
        \begin{influence-diagram}
            \node (X_mec) [relevancew] {$\Theta_T$};
            \node (U1_mec) [relevancew, right = of X_mec] {$\Theta_{U^1}$};
            \node (U2_mec) [relevancew, right = of U1_mec] {$\Theta_{U^2}$};
            \node (X_exo_mec) [relevancew, above = of X_mec] {$\Theta_{\exovar_T}$};
            \node (U1_exo_mec) [relevancew, above = of U1_mec] {$\Theta_{\exovar_{U^1}}$};
            \node (U2_exo_mec) [relevancew, above = of U2_mec] {$\Theta_{\exovar_{U^2}}$};
            \node (D1_mec) [relevanceb, below right = 1.4cm and 0.7cm of X_mec] {$\Pi_{D^1}$};
            \node (D2_mec) [relevanceb, right = of D1_mec] {$\Pi_{D^2}$};
            \node (D1_exo_mec) [relevancew, below = of D1_mec] {$\Theta_{\exovar_{D^1}}$};
            \node (D2_exo_mec) [relevancew, below = of D2_mec] {$\Theta_{\exovar_{D^2}}$};
            \path (D1_mec) edge[->, bend right=15] (D2_mec);
            \path (D2_mec) edge[->, bend right=15] (D1_mec);
            \edge {X_mec} {D1_mec,D2_mec};
            \edge {U1_mec} {D1_mec};
            \edge {U2_mec} {D2_mec};
            \edge {D1_exo_mec} {D1_mec, D2_mec};
            \edge {D2_exo_mec} {D1_mec, D2_mec};
            \node (space1) [minimum size=0mm, node distance=2mm, above = 0.7cm of U1_mec, draw=none] {};
            \draw (X_exo_mec) edge[,in=180,out=-45] (space1.center)
            (space1.center) edge[,->,out=0,in=90] (D2_mec);
            \path (X_exo_mec) edge[->, in=90,out=-60] (D1_mec); 
            \path (U1_exo_mec) edge[->, in=90,out=-120] (D1_mec);
            \path (U2_exo_mec) edge[->, in=90,out=-120] (D2_mec);
        \end{influence-diagram}
        \caption{}
        \label{fig:maids:mechanised_maid}
    \end{subfigure}
    \caption{(a) A (Markovian) SCG representing Example \ref{ex:job_market}. Note that we have included exogenous variables for $U^1$ and $U^2$, although as neither is stochastic, this is not strictly necessary. (b) The $\R^{\BR}$-relevance graph of this game.
    }
    \label{fig:maids}
\end{figure}

{While computing predictions and interventions in SCGs is therefore relatively straightforward, there are two main difficulties that arise when computing counterfactuals. The first is the choice of how to represent stochastic decision rules using structural functions and exogenous variables, and the second is the problem of updating our beliefs about the policy profile played in the counterfactual world given our evidence about the policy profile played in the actual world. We resolve each of these difficulties in turn.}

\subsubsection*{Decision Rules as Structural Functions}

We assume that agents play the same kind of game regardless of the level in the causal hierarchy at which we model them. In these games, (behavioural) play equates to selecting decision rules which stochastically sample a decision conditional on the value of some (non-exogenous) parents. In \emph{structural} causal games, we represent these decision rules using structural functions and exogenous variables. {One proposal would therefore be to view each agent as choosing both a `structural decision rule' $\dot{\pi}_D(D \mid \Pa_D)$ and a distribution $\Pr(\exovar_D)$, with a shared mechanism parent $\Pi_D$ for both $D$ and $\exovar_D$. This, however, leads to a different type signature for decision rules, and moreover leads to a formalism in which (pre-policy) interventions can be made upstream of stochastic variables, which are ruled out in SCMs.} We {therefore} propose an equivalent formulation in which each {agent} controls only their decision variables and not their exogenous parents.

{Unfortunately, while} our assumptions about the rationality of the agents tell us what CPDs are assigned to their decision variables, {they} are insufficient for telling us what precise deterministic mechanisms the agents use to implement these CPDs (as a function of some stochastic exogenous variable). In fact, unless we chose to explicitly restrict the form of the mechanism, such as by stipulating that it belongs to some parametric class, there will typically be \emph{infinitely} many deterministic functions that induce a particular distribution over a decision variable \cite{Balke1994}. Without specifying such functions, it will not (in general) be possible to answer counterfactual queries in games, and yet the precise form of these functions may impact the answers to these queries \cite{dawid2002influence}.

In essence, the choice of how to represent a decision rule $\pi_D \in \Delta(\dom(D) \mid \dom(\Pa_D \setminus \{\exovar_D\}))$ using a stochastic exogenous variable $\exovar_D$ and a deterministic mechanism $\dot{\pi}_D \in \Delta(\dom(D) \mid \dom(\Pa_D))$ is the choice of what part of the decision rule we assume remains fixed across counterfactual worlds ($\exovar_D$) and what part may vary ($\dot{\pi}_D$). Assuming that we have no pre-existing knowledge about this representation, {we propose to stay true to the spirit of behavioural policies by viewing} each agent's randomisation as independent between both:
\begin{itemize}
    \item \emph{Decision rules}, in the sense that learning about an agent's random choice under one decision rule $\pi_D$ is uninformative in settings where the agent is using a different decision rule $\pi'_D$;
    \item \emph{Decision contexts}, in the sense that an agent's decision rule $\pi_D$ can naturally be interpreted as independently sampling an action $d$ \emph{after} seeing {an assignment} $\overline{\pa}_D$ of the non-exogenous parents $\overline{\Pa}_D \coloneqq \Pa_D \setminus \{\exovar_D\}$.
\end{itemize}
We formalise this assumption by representing each exogenous variable $\exovar_D$ as a set containing a variable $\exovar^{\pi_D, \overline{\pa}_D}_D$ for each $\pi_D \in \Delta(\dom(D) \mid \dom(\overline{\Pa}_D))$ and $\overline{\pa}_D \in \dom(\overline{\Pa}_D)$, where $\dom(\exovar^{\pi_D, \overline{\pa}_D}_D) = \dom(D)$ (i.e., $\exovar_D$ is a random field with independently distributed elements). Given a stochastic decision rule $\pi_D(D \mid \overline{\Pa}_D)$, we may then define a canonical structural representation by setting:
\begin{align}
    \label{eq:dist_def}
    \begin{split}
        \Pr(\exovar^{\pi_D, \overline{\pa}_D}_D=d) &\coloneqq \pi_D(d \mid \overline{\pa}_D),\\
        \dot{\pi}_D(D = d \mid \overline{\pa}_D, \exoval_D) &\coloneqq \delta(D {,} \exoval^{\pi_D, \overline{\pa}_D}_D),
    \end{split}
\end{align}
where note that $\dot{\pi}_D$ is effectively parameterised by $\pi_D$, i.e., we have $\dot{\pi}_D(D \mid \overline{\Pa}_D, \exovar_D; \pi_D)$. The joint distribution over $\exovar_D$ is simply the \emph{product of probability distributions} over $\exovar^{\pi_D, \overline{\pa}_D}_D$ \cite{Bauer1996}. 
Proposition \ref{prop:structural_equiv} below then follows immediately,\footnote{Though note that there are many constructions that would result in these properties; we merely present one particularly simple example.} and means that we may continue to interpret decision rules, policy profiles, and rationality relations as we do in MAIDs and CGs, where each agent plays some $\pi_D \in r_D(\pa_{\Pi_D}) \subseteq \Delta(\dom(D) \mid \dom(\overline{\Pa}_D))$ and each $\pi_D$ parameterises $\dot{\pi}_D$. Moreover, this additional structure allows us to generalise the definition of counterfactuals in SCMs to counterfactuals in games.
\begin{proposition}
    \label{prop:structural_equiv}
    For distributions over $\exovar_D$ and $D$ as governed by equations (\ref{eq:dist_def}), there is a one-to-one correspondence between the set of stochastic decision rules $\Delta(\dom(D) \mid \dom(\overline{\Pa}_D))$ and the set of deterministic decision rules $\dom(\dot{\Pi}_D) \subset \Delta(\dom(D) \mid \dom(\Pa_D))$. Moreover, given two such corresponding decision rules $\pi_D$ and $\dot{\pi}_D$, then $\int_{\dom(\exovar_D)} \dot{\pi}_D(d \mid \overline{\pa}_D, \exoval_D)\Pr(\exoval_D) \,d \exoval_D = \pi_D(d \mid \overline{\pa}_D)$.
\end{proposition}

\subsubsection*{Counterfactual Rational Outcomes}
The second difficulty of answering counterfactual queries in games arises due to the possible existence of multiple rational outcomes. If we have evidence that the equilibrium ${\bm{\pi}}$ was played in the actual world, how and to what extent should that inform us of the equilibrium ${\bm{\pi}}'$ played in the counterfactual world where the values of some mechanism variables may have changed? 

{To answer this question, we begin by introducing our approach to answering counterfactual queries in SCGs, which mirrors Pearl's approach for SCMs} (described in Section \ref{sec:causal_models}). {That is, we condition, intervene, and then compute the resulting distribution, though we generalise this to observations and interventions on both object-level and mechanism variables.} In particular, we compute the set $\Pr^{\R}(\bm{x}_{\I} \mid \bm{z})$ as follows:

\begin{enumerate}
    \item For every \emph{actual} rational outcome ${\bm{\pi}} \in \R(\mec{\model} \mid \bm{z})$, update $\Pr(\exovals)$ to $\Pr^{{\bm{\pi}}}(\exovals \mid \bm{z})$ (`abduction');
    \item Apply the intervention $\I$, on variables $\bm{Y}$, recomputing any rational responses to form ${\bm{\pi}}'$ and adding new exogenous variables $\exovars^*$ to form $\exovars' = \exovars \cup \exovars^*$ where required  (`action');
    \item Return each marginal distribution $\int_{\dom(\exovars')} \Pr^{\bm{\pi}'}(\bm{x} \mid \exovals') \Pr(\exovals') \, d\exovals'$ in this modified model for each new \emph{counterfactual} rational outcome ${\bm{\pi}'}$ (`prediction').
\end{enumerate}

In the first step, we update our beliefs about the exogenous and decision rule variables in the actual world under ${\bm{\pi}}$; in the second, we apply an intervention $\I$ to form the counterfactual world (and recompute the rational outcomes based on this change); and in the third, we return the new set of distributions that are consistent with our beliefs from the first step and the results of the intervention made in the second step.

{By reading these steps carefully, one notices a difficulty in SCGs that does not arise in an SCM:
when we condition on $\bm{z}$ in the first step we obtain a set of policies $\R(\mec{\model} \mid \bm{z})$ that are rational in the \emph{actual} world.
Then in step two, we compute new rational responses $\bm{\pi}'$, and 
it is $\bm{\pi}'$ rather than $\R(\mec{\model} \mid \bm{z})$ that features in the \emph{counterfactual} world of the final step.
This raises the question of the extent to which knowledge of the rational policies $\R(\mec{\model} \mid \bm{z})$
should be used to compute ${\bm{\pi}}'$.
}

{Let us first note that if $\I$ is a \emph{post}-policy intervention, then this has no impact on the rational outcomes of the counterfactual world -- they are simply the same as the actual world, given by $\R(\mec{\model} \mid \bm{z})$ -- and hence there is no difficulty. The issue only arises in \emph{pre}-policy counterfactuals, such as query \hyperref[query:3b]{3b} (`Given that the worker never decides to go to university, what would be their wellbeing if they always decided to go to university?'), where the intervention (in query \hyperref[query:3b]{3b}, on the worker's decision rule) means that the set of rational outcomes in the counterfactual world will be different from those in the actual world.}

{Because each policy profile $\bm{\pi}$ is made up of decision rules $\pi_D$, we can formalise this question by asking which decision rule variables $\bm{\Pi}(\I) \subseteq \bm{\Pi}$ are \emph{invariant} to the intervention $\I$. Written in terms of the three-step process above, we must have ${\bm{\pi}}(\I) = {\bm{\pi}}'(\I)$, i.e., for any invariant decision rule variable $\Pi \in \bm{\Pi}(\I)$ then the counterfactual decision rule $\pi'_D$ is equal to the actual decision rule $\pi_D$. Those that are not invariant must have their values recomputed in the new counterfactual model. For instance, as argued above, when $\I$ is post-policy, $\bm{\Pi}(\I) = \bm{\Pi}$. That is, none of the values of the decision rule variables need to be recomputed. How should we choose $\bm{\Pi}(\I)$ when $\I$ is not (fully) post-policy?}

{There are multiple principles that could be invoked in order to make this choice. The simplest -- let us call it the \emph{simplicity principle} -- is to recompute the values of \emph{all} decision rule variables, i.e., $\bm{\Pi}(\I) = \varnothing$. In other words, the intervention means that `all bets are off' after the intervention is made, and the actual rational outcomes $\R(\mec{\model} \mid \bm{z})$ have no bearing on the counterfactual rational outcomes. Under this principle, computing pre-policy counterfactuals is} reminiscent of the approach taken in existing work on cyclic causal models \cite{Bongers2016}, where two sets of solutions are induced by two halves of a twin graph. {The problem with this principle is that it may require us to ignore information gathered from our observation $\bm{z}$. For example, if $\bm{Z} = \bm{z}$ implies that $\Pi_D = \pi_D$ in the actual world and $\I$ is \emph{causally downstream} of $\Pi_D$ (and hence can have no effect on the value of $\Pi_D$), then this means we also know that $\Pi_D = \pi_D$ in the counterfactual world, i.e., $\Pi_D \in \bm{\Pi}(\I)$.}

{To solve this problem, we can instead invoke the \emph{closest possible world principle}, where we retain as much information as possible from our knowledge of the rational policies $\R(\mec{\model} \mid \bm{z})$. While the values of some decision rules may still need to be recomputed, by keeping $\bm{\Pi}(\I)$ as large as possible, we avoid the need for re-solving the entire game, and can provide a more accurate answer to counterfactual queries. The process of computing $\bm{\Pi}(\I)$ under this principle is slightly more complicated, however, as it involves propagating the effects of an intervention through models that contain both cycles and non-determinism. In the remainder of the paper we therefore employ the simplicity principle above, which is also more in keeping with prior work.\footnote{In practice, for queries \hyperref[query:3a]{3a} and \hyperref[query:3b]{3b}, it happens to be the case that both principles lead to the same answer.} An algorithm for computing ${\bm{\Pi}}(\I)$ according to the closest possible world principle is, however, provided in Appendix \ref{app:closest_possible_world} for reference.}

\subsubsection*{Defining Counterfactuals in Games}

{Given a set of invariant decision rule variables ${\bm{\Pi}}(\I)$, the answer produced by the three-step process given above can be written as follows.}

\begin{definition}
    \label{def:counterfactual_query}
    Given a mechanised SCG $\mec\model$ with rationality relations $\R$, the \textbf{answer to a counterfactual query} of the probability of $\bm{x}$ given observation $\bm{z}$ and intervention $\I$ on variables $\bm{Y}$ is given by the set:
    $$\Pr^{\R}(\bm{x}_{\I} \mid \bm{z}) \coloneqq \big\{ {\textstyle \int_{\dom(\exovars')}} \Pr^{\bm{\pi}'}(\bm{x}_{\I} \mid \exovals, \exovals^*) \Pr(\exovals^*) \Pr^{{\bm{\pi}}}(\exovals \mid \bm{z}) \, d\exovals' \big\}_{({\bm{\pi}},{\bm{\pi}}') \in \R(\mec{\model}_{\I} \mid \bm{z})},$$
    where $\exovars^* = \exovars' \setminus \exovars$ are any newly added exogenous variables as a result of a soft intervention,
    $$\R(\mec{\model}_{\I} \mid \bm{z}) \coloneqq \big\{ ({\bm{\pi}},{\bm{\pi}}') \in \R(\mec{\model} \mid \bm{z}) \times \R(\mec{\model}_{\I}) : \bm{\pi}(\I) = {\bm{\pi}}'(\I) \big\},$$
    is the set of \textbf{actual-counterfactual rational outcomes}, and ${\bm{\Pi}}(\I)$ is the set of \textbf{invariant decision rule variables}. 
    Note that if $\I$ {is fully post-policy}, then the rational outcomes remain the same, i.e., ${\bm{\Pi}}(\I) = {\bm{\Pi}}$ and $\R(\mec{\model}_{\I} \mid \bm{z}) = \big\{ ({\bm{\pi}},{\bm{\pi}}) : {\bm{\pi}} \in \R(\mec{\model} \mid \bm{z}) \big\}$ when $\bm{Y} \subseteq \bm{V}$.
\end{definition}

{In the definition of $\Pr^{\R}(\bm{x}_{\I} \mid \bm{z})$, for every actual policy $\bm{\pi}$ and counterfactual policy $\bm{\pi'}$, we compute the product of three quantities. $\Pr^{{\bm{\pi}}}(\exovals \mid \bm{z})$ is the updated distribution over the exogenous variables under $\bm{\pi}$, and corresponds to the first step. If $\I$ is a soft intervention, then we add further variables $\exovars^* = \exovars' \setminus \exovars$ to capture the stochasticity of $\I$, which leads to the term $\Pr(\exovals^*)$.\footnote{Note that the distribution over these `fresh' exogenous variables does not depend on the policy, actual or counterfactual.} We then condition on both $\exovals$ and $\exovals^*$ and compute the value of $\bm{x}_{\I}$ under $\bm{\pi}'$, hence the term $\Pr^{\bm{\pi}'}(\bm{x}_{\I} \mid \exovals, \exovals^*)$. Finally, in the third step, we marginalise over all exogenous variables $\exovars'$. The set $\R(\mec{\model}_{\I} \mid \bm{z})$ defines the pairs of policies that we must consider. Namely, actual rational outcomes $\bm{\pi} \in \R(\mec{\model} \mid \bm{z})$ and counterfactual rational outcomes $\bm{\pi}' \in \R(\mec{\model}_{\I})$ such that the decision rules invariant to $\I$, denoted $\bm{\Pi}(\I)$, remain the same: $\bm{\pi}(\I) = {\bm{\pi}}'(\I)$.}

We briefly demonstrate the three step process above by returning to queries \hyperref[query:3a]{3a} and \hyperref[query:3b]{3b}.
To answer \hyperref[query:3a]{3a} we must compute $\Pr^{\bm{\pi}}(u^1_g \mid \neg g)$. First, we note that as this involves a post-policy intervention then we only need to consider the \emph{actual} rational outcomes, as ${\bm{\Pi}}(\I) = {\bm{\Pi}}$. Thus, for each ${\bm{\pi}} \in \R(\mec{\model} \mid \neg g)$ we begin by updating $\Pr(\exovals)$ to $\Pr^{{\bm{\pi}}}(\exovals \mid \neg g)$. In this case, such an update amounts to changing $\Pr^{{\bm{\pi}}}(\exoval_T, \exoval_{D^1}, \exoval_{D^2})$ to $\Pr^{{\bm{\pi}}}(\exoval_T, \exoval_{D^1}, \exoval_{D^2} \mid \neg g)$, where note that we independently update each $\exovar^{\pi_D, \overline{\pa}_D}_D$ for every such ${\bm{\pi}}$ and endogenous decision context $\overline{\pa}_D$. Following this, we apply the intervention $\Do(D^1 = g)$. The final answer to the query is therefore rather simple in this case, and is given by $\Pr^{\R}(u^1_g \mid \neg g) = \big\{ \Pr^{\bm{\pi}} (u^1_g \mid \neg g) \big\}_{{\bm{\pi}} \in \R(\mec{\model} \mid \neg g)}$. We thus have $\big\{ \expect_{\bm{\pi}} [U^1_g \mid \neg g] \big\}_{{\bm{\pi}} \in \R(\mec\model \mid \neg g)} = \{\frac{7}{2}\}$.

Query \hyperref[query:3b]{3b} involves a hard pre-policy intervention $\I$ that sets $\Pi_{D^1}$ to $\overline{\pi}_{D^1}$. As ${\bm{\Pi}}(\I) = \varnothing$, the answers to the query are given under the interventional outcomes, i.e., $\R(\mec{\model}_{\I} \mid \tilde{\pi}_{D^1}) = \big\{ ({\bm{\pi}}',{\bm{\pi}}') : {\bm{\pi}} \in \R(\mec{\model}_{\I})\big\}$. In this particular game, $\R(\mec{\model}_{\I}) = \R(\mec{\model} \mid \overline{\pi}_{D^1})$ and so the answer is the same as the answer to query \hyperref[query:1b]{1b}.

\begin{remark}
    The reasons for choosing between a CG or an SCG to model a given problem are analogous to the respective reasons for choosing a CBN or an SCM. Using the latter, one can reason about counterfactuals, path-specific effects, and questions of identifiability \cite{avin2005identifiability}. The former, however, is a simpler formalism, and requires less knowledge about the precise functions holding between the variables, making CGs a more attractive choice when one does not require any of the features listed above.
\end{remark}

\section{Solution Concepts and Subgames}
\label{sec:subgames_eqs}

In Section \ref{sec:mechanised_maids_relevance}, we explained how a set of rationality relations can be used to capture the process by which agents choose their decision rules, and thus which mechanisms agents need to consider when doing so. In this section, we build on these ideas in three subsections. Firstly, we detail the distinction between mixed and behavioural policies and their relation to {NEs} in MAIDs. Secondly, we introduce the concept of \emph{subgames} within MAIDs, which, analogously to their EFG counterparts, allow us to analyse and solve parts of the game independently. Finally, we introduce several \emph{equilibrium refinements} for MAIDs, which are discussed in relation to their EFG counterparts in Section \ref{sec:equivalences}. With these contributions, we aim to place causal games on a more equal footing with EFGs as a tool for game-theoretic analysis. Note that as (S)CGs are refinements of MAIDs, the results in this section also apply to these models. Concepts in this section will be explained with the help of the following example, shown in Figure \ref{fig:warehouse:a}.

\begin{example}[Warehouse Robots]
    \label{ex:warehouse}
    Two robots are working together in a warehouse. Robot one is responsible for fulfilling orders; it can decide to move quickly or slowly. It will not break anything if it moves slowly, but might break something if it moves quickly. Robot two is responsible for keeping the warehouse tidy and so must decide whether to patrol or not, but it can only observe what robot one does. If it patrols, robot two can repair broken items, but by doing so it might obstruct robot one and prevent it from completing its order. Robot one is rewarded for fulfilling orders, the quicker the better, and penalised for breaking things. Robot two is rewarded for everything in the warehouse ending up in a state of repair, but incurs a small energy cost for patrolling.
\end{example}

To parameterise this game, let us suppose that robot one breaks something if it moves quickly ($D^1 = q$) with probability $\frac{1}{3}$ 
but will not break anything otherwise, and that robot two obstructs robot one with probability $\frac{1}{2}$ if it patrols ($D^2 = p$) and probability zero otherwise. Finally, we define the utility functions such that robot one receives a reward of $5$ or $2$ for completing an order quickly or slowly respectively, but it also incurs a penalty of $-3$ for breakages. Robot two receives a reward of $6$ for everything being in a state of repair, but incurs a penalty of $-1$ for patrolling. Given this parameterisation, we can easily calculate the expected payoff of each agent for the four possible action combinations. For reference, we show these in Figure \ref{fig:warehouse:c} using an EFG that only bifurcates on the two robots' decisions; the chance variable $B$ has been marginalised out to create a reduced EFG (as detailed in Appendix \ref{sec:transformations}).

\begin{figure}[h]
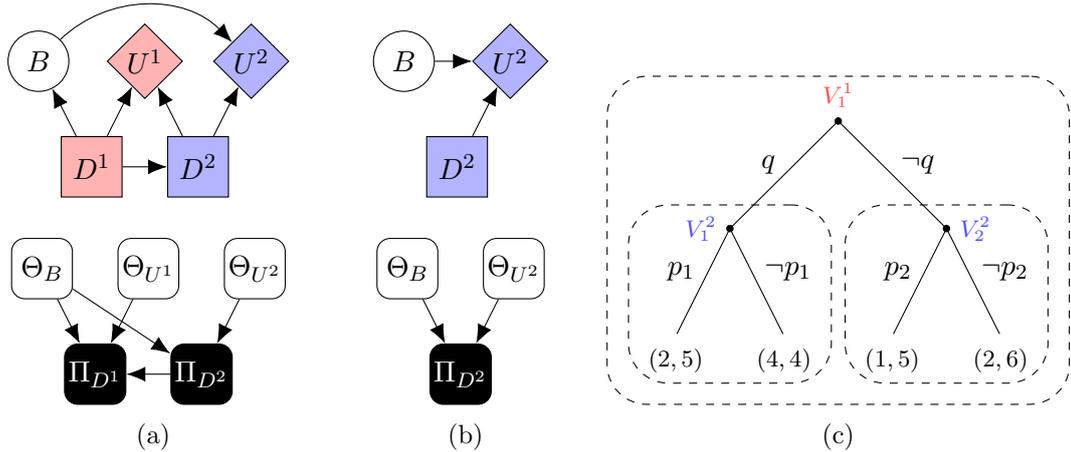

    \centering
    \begin{subfigure}[b]{0.3\linewidth}
        \centering
        \begin{influence-diagram}
            \node (T) [] {$B$};
            \node (U1) [utility, right = of T, player1] {$U^1$};
            \node (U2) [utility, right = of U1, player2] {$U^2$};
            \node (D1) [decision, below right = 1.4cm and 0.7cm of T, player1] {$D^1$};
            \node (D2) [decision, right = of D1, player2] {$D^2$};
            \edge [information] {D1} {D2};
            \path (T) edge[->, bend left=45] (U2);
            \edge {D1} {U1,T};
            \edge {D2} {U1,U2};
        \end{influence-diagram}
        
        \vspace{0.5cm}

        \begin{influence-diagram}
            \node (X_mec) [relevancew] {$\Theta_B$};
            \node (U1_mec) [relevancew, right = of X_mec] {$\Theta_{U^1}$};
            \node (U2_mec) [relevancew, right = of U1_mec] {$\Theta_{U^2}$};
            \node (D1_mec) [relevanceb, below right = 1.4cm and 0.7cm of X_mec] {$\Pi_{D^1}$};
            \node (D2_mec) [relevanceb, right = of D1_mec] {$\Pi_{D^2}$};
            \edge {X_mec} {D1_mec,D2_mec};
            \edge {D2_mec} {D1_mec};
            \edge {U1_mec} {D1_mec};
            \edge {U2_mec} {D2_mec};
        \end{influence-diagram}
        \caption{}
        \label{fig:warehouse:a}
    \end{subfigure}
    \begin{subfigure}[b]{0.2\linewidth}
        \centering
        \begin{influence-diagram}
            \node (X) [] {$B$};
            \node (U2) [utility, right = of X, player2] {$U^2$};
            \node (D2) [decision, below right = 1.4cm and 0.7cm of X, player2] {$D^2$};
            \edge {D2, X} {U2};
        \end{influence-diagram}

        \vspace{0.5cm}

        \begin{influence-diagram}
            \node (X_mec) [relevancew] {$\Theta_B$};
            \node (U2_mec) [relevancew, right = of X_mec] {$\Theta_{U^2}$};
            \node (D2_mec) [relevanceb, below right = 1.4cm and 0.7cm of X_mec] {$\Pi_{D^2}$};
            \edge {X_mec, U2_mec} {D2};
        \end{influence-diagram}
        \caption{}
        \label{fig:warehouse:b}
    \end{subfigure}
    \begin{subfigure}[b]{0.4\linewidth}
        \centering
        \begin{istgame}[scale=0.95]
            \xtdistance{15mm}{30mm}
            \istroot(0)<90,red!70>{$V^1_1$}
            \istb{q}[al]
            \istb{\neg q}[ar] 
            \endist
            \xtdistance{15mm}{15mm}
            \istroot(1)(0-1)<180, blue!70>{$V^2_1$}
            \istb{p_1}[al]{(2,5)}
            \istb{\neg p_1}[ar]{(4,4)} 
            \endist
            \istroot(2)(0-2)<0, blue!70>{$V^2_2$}
            \istb{p_2}[al]{(1,5)}
            \istb{\neg p_2}[ar]{(2,6)} \endist
            \xtSubgameBox(1){(1-1)(1-2)}[black,inner sep = 17pt]
            \xtSubgameBox(2){(2-1)(2-2)}[black,inner sep = 17pt]
            \xtSubgameBox(0){(1-1)(1-2)(2-1)(2-2)}[black,inner sep = 25pt]
        \end{istgame} 
        \caption{}
        \label{fig:warehouse:c}
    \end{subfigure}
    \caption{(a) A MAID $\model = (\graph, \theta)$ representing Example \ref{ex:warehouse} and its $s$-relevance graph. (b) The smallest proper $s$-subdiagram $\graph'$ of $\graph$ and its $s$-relevance graph. (c) A reduced EFG where each of the proper subgames corresponds to the two possible $s$-subgames for $\graph'$, where $D^1=q$ or $D^1=\neg q$ respectively.}
\label{fig:warehouse_maid}
\end{figure}

\subsection{Nash Equilibria}
\label{sec:NE}

A solution concept aims to identify a subset of the possible outcomes of a game that may occur if agents act rationally. In non-cooperative games, the most fundamental solution concept is a Nash equilibrium (NE) \cite{nash1950equilibrium}, a policy profile such that no agent may benefit by unilaterally deviating.
{In Example \ref{ex:warehouse}, for instance, the policy profile ${\bm{\pi}}^{\NE}$ in which robot one chooses $D_1 = q$ and robot two chooses $D_2 = p$ whatever the value of $D_1$ is an NE. 
Previous work introduced this concept to MAIDs, as recalled in Definition \ref{def:NE} \cite{koller2003multi}, but did not fully characterise when an NE is guaranteed to exist in a MAID. This existence depends on which class of policies agents are permitted to choose from.}

So far in this paper, we have viewed agents as employing \emph{behavioural policies}, where each agent selects decision rules for each of their decisions independently. In contrast, a \emph{mixed policy} allows an agent to coordinate their choice of decision rules at different decisions; it is a distribution over pure policies. In what follows, we use a dot $\dot{}$ to denote the determinism of pure policies and $\mu$ to denote mixed policies.

\begin{definition}
    \label{def:policies}
    Let $\dom(\dot{\Pi}_D)$ be the set of all possible pure decision rules for $D$, and recall that we write $\dom(\bm{V}) = \bigtimes_{V \in \bm{V}} \dom(V)$. A \textbf{mixed policy} for agent $i$ is some $\mu^i \in \Delta ( \dom(\dot{\bm{\Pi}}_{\bm{D}^i}) )$, a \textbf{behavioural policy} is some ${\bm{\pi}}^i \in \dom( \bm{\Pi}_{\bm{D}^i} )$, and a \textbf{pure policy} is some $\dot{{\bm{\pi}}}^i \in \dom( \dot{\bm{\Pi}}_{\bm{D}^i} )$.
\end{definition}

\begin{proposition}
    \label{prop:NEexistence}
    {A (behavioural policy) NE is not guaranteed to exist in a MAID.}
\end{proposition}

Nash's theorem establishes that any finite game is guaranteed to have an NE, as long as all agents are allowed to choose mixed policies \cite{nash1950equilibrium}. However, in general, there is no such guarantee when agents are only permitted to use behavioural policies (for {an example}, see Appendix \ref{app:non-existence_proofs}). Despite this negative result, behavioural policies are often regarded as more `natural' due to their representation of agents randomising at each decision point instead of once at the beginning of the game \cite{osborne1994course}. Moreover, behavioural policies respect the conditional independencies of the MAID's graph. As such, an interesting question is when an NE in behavioural policies is guaranteed to exist. As a first step, it is relatively straightforward to prove an analogue of Kuhn's theorem: if all agents {in} the game have perfect recall (i.e., agents never forget their past moves nor any of the information they knew previously, as introduced in Definition \ref{def:perfectrecall}), then an NE in behavioural policies is guaranteed to exist \cite{Kuhn1953}.

\begin{lemma}
    \label{lem:prequiv}
    Let ${\bm{\pi}}^{-i} \in \Delta(\dom(\bm{D}^{-i}) \mid \dom(\Pa_{\bm{D}^{-i}}))$ be a partial (behavioural or mixed) policy profile for agents $N \setminus \{i\}$ in a MAID $\model$. If agent $i$ has perfect recall in $\model$, then for every mixed policy $\mu^i$ there exists a behavioural policy ${\bm{\pi}}^i$ such that $\Pr^{(\mu^i, {\bm{\pi}}^{-i})}(\bm{v}) = \Pr^{({\bm{\pi}}^i, {\bm{\pi}}^{-i})}(\bm{v})$.
\end{lemma}

\begin{proposition}
    \label{prop:kuhn}
    In any MAID $\model$ with perfect recall, there exists a (behavioural) policy profile ${\bm{\pi}}$ that is an NE.
\end{proposition}

{Going further, because MAIDs reveal conditional independencies between variables, only a weaker criterion of \emph{sufficient recall} is sufficient for the existence of an NE in behavioural policies. K\&M implicitly prove this result -- included as Proposition \ref{prop:NEexistance} -- when proving that their Algorithm for finding an NE always converges under certain conditions (as these correspond with sufficient recall) \cite{koller2003multi}. Appendix \ref{app:non-existence_proofs} provides an example of a MAID in which all agents have sufficient but imperfect recall. In this example, there exists a mixed policy profile with no behavioural policy resulting in the same distribution over outcomes. There is still, however, an NE in behavioural policies. On the other hand, a MAID without sufficient recall may \emph{not} have an NE in behavioural policies ({an example is again} given in Appendix \ref{app:non-existence_proofs}).}

\begin{proposition}
    \label{prop:sr_pr}
    If an agent $i$ in a MAID $\model$ has perfect recall, then they also have sufficient recall. However, if an agent has sufficient recall, then they do not always have perfect recall.
\end{proposition}

\subsection{Subgames}
\label{sec:MAID_subgames}

Subgames in EFGs represent parts of the game that can be solved independently from the rest; we now introduce the analogous notion of $\R$-\emph{subgames} in MAIDs. These subgames have three uses: they allow us to introduce further equilibrium refinements (in Section \ref{sec:equilibrium_refinements}); they can reduce the cost of computing equilibria \cite{Hammond2021}; and they allow us to analyse agents' decision-making in isolation from other parts of the game, which may be useful for other forms of analysis (such as those discussed in Section \ref{sec:applications}). {In Appendix \ref{app:subgamecompute}, we demonstrate the second of these facts by empirically showing how subgames in MAIDs can be used to compute NEs much more efficiently than in EFGs.}

However, there are two key differences between subgames in EFGs and those in MAIDs. Firstly, because MAIDs explicitly represent conditional independencies between variables, we can often find more subgames in a MAID than in its corresponding EFG. Secondly, because the ability to solve part of a game independently varies with the solution concept, the $\R$-subgames vary with $\R$. Given a set of graphical criteria, $\R$-reachability, such as those in \cref{sec:relevance}, one can identify the structure of any $\R$-subgames -- which we refer to as an $\R$-subdiagram -- using only the underlying graph. 

\begin{definition}
    Given a mechanised MAID $\mec{\model} = (\mec{\graph}, \bm{\theta}, \R)$ and a set of sound and complete graphical criteria for $\R$-relevance -- ie., $\R$-reachability -- we refer to the subgraph $(\bm{V}',\mathscr{E}')$ of $\graph$, along with the set of agents $N' \subseteq N$ possessing decision variables in that subgraph, as an $\R$-\textbf{subdiagram} $\graph' = (N', \bm{V}',\mathscr{E}')$ if:
    \begin{itemize}
        \item $\bm{V}'$ contains every variable $Z$ such that $\mecvar_Z$ is $\mathcal{R}$-reachable from some $\Pi_{D}$ with $D \in \bm{V}'$;
        \item $\bm{V}'$ contains, for all $X,Y \in \bm{V}'$, every variable that lies on a directed path $X \pathto Y$ in $\graph$.
    \end{itemize}
    As before, we drop $\R$ from our notation where unimportant or unambiguous.
\end{definition}

The first condition on $\bm{V}'$ ensures that for any decision variable $D$ in the subdiagram, any variable whose mechanism may impact the rational response for $D$ is also included in the graph. This means that the rational responses for the decision rules in this part of the game are independent of what happens elsewhere. The second condition says that additional variables may also be included in the subdiagram as long as mediators are included too. This ensures that the CPDs for all the variables in the subgame remain consistent, and allows us to define a correspondence between subgames in MAIDs and subgames in EFGs (in Section \ref{sec:equivalences}). To find the subgames for each subdiagram, we must update the parameterisation of the remaining variables to be consistent with the original game and the structure of the graph.

\begin{definition}
    \label{def:MAIDsubgame}
    Given a mechanised MAID $\mec{\model} = (\mec{\graph}, \bm{\theta}, \R)$, an $\R$-\textbf{subgame} of $\model$ is a new MAID $\model'=(\graph', \bm{\theta}')$ where $\graph'$ is an $\R$-subdiagram of $\graph$ and
    $\bm{\theta}'$ is defined by $\Pr'(\bm{v}' ; \bm{\theta}') \coloneqq \Pr(\bm{v}' \mid \bm{z} ; \bm{\theta})$, where $\bm{z}$ is some instantiation of the variables $\bm{Z} = \bm{V}\setminus \bm{V'}$.\footnote{In fact, it can easily be appreciated that only the setting $\bm{z}$ of the variables that have a child in $\bm{V'}$ will matter.}
    An $\R$-subgame is \textbf{feasible} if there exists a policy profile ${\bm{\pi}}$ where $\Pr^{\bm{\pi}}(\bm{z}) > 0$. 
\end{definition}

\begin{figure}[h]
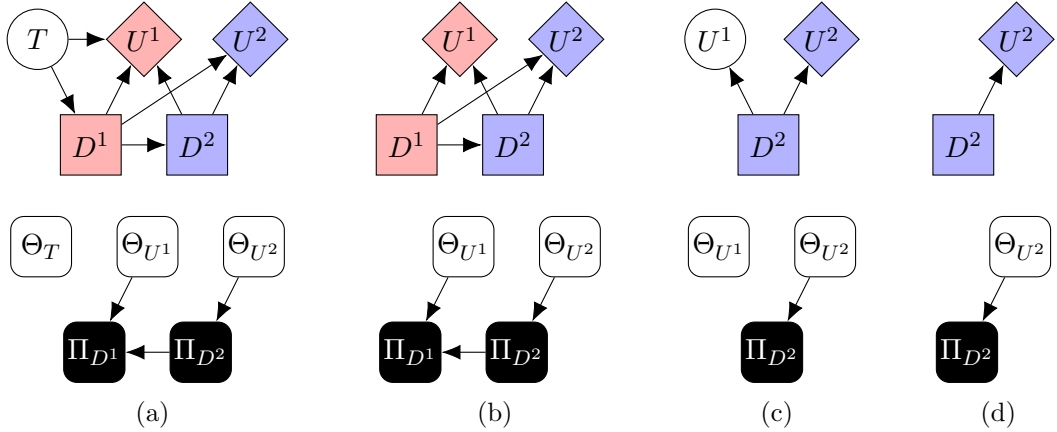

    \centering
    \begin{subfigure}[b]{0.3\linewidth}
        \centering
        \begin{influence-diagram}
            \node (T) [] {$T$};
            \node (U1) [utility, right = of T, player1] {$U^1$};
            \node (U2) [utility, right = of U1, player2] {$U^2$};
            \node (D1) [decision, below right = 1.4cm and 0.7cm of T, player1] {$D^1$};
            \node (D2) [decision, right = of D1, player2] {$D^2$};
            \edge [information] {T} {D1};
            \edge [information] {D1} {D2};
            \edge {T} {U1};
            \edge {D1} {U1,U2};
            \edge {D2} {U1,U2};
        \end{influence-diagram}
        
        \vspace{0.5cm}

        \begin{influence-diagram}
            \node (X_mec) [relevancew] {$\Theta_T$};
            \node (U1_mec) [relevancew, right = of X_mec] {$\Theta_{U^1}$};
            \node (U2_mec) [relevancew, right = of U1_mec] {$\Theta_{U^2}$};
            \node (D1_mec) [relevanceb, below right = 1.4cm and 0.7cm of X_mec] {$\Pi_{D^1}$};
            \node (D2_mec) [relevanceb, right = of D1_mec] {$\Pi_{D^2}$};
            \edge {D2_mec} {D1_mec};
            \edge {U1_mec} {D1_mec};
            \edge {U2_mec} {D2_mec};
        \end{influence-diagram}
        \caption{}
        \label{fig:subgames:1}
    \end{subfigure}
    \begin{subfigure}[b]{0.25\linewidth}
        \centering
        \begin{influence-diagram}
            \node (U1) [utility, player1] {$U^1$};
            \node (U2) [utility, right = of U1, player2] {$U^2$};
            \node (D1) [decision, below left = 1.4cm and 0.7cm of U1, player1] {$D^1$};
            \node (D2) [decision, right = of D1, player2] {$D^2$};
            \edge [information] {D1} {D2};
            \edge {D1} {U1,U2};
            \edge {D2} {U1,U2};
        \end{influence-diagram}

        \vspace{0.5cm}

        \begin{influence-diagram}
            \node (U1_mec) [relevancew] {$\Theta_{U^1}$};
            \node (U2_mec) [relevancew, right = of U1_mec] {$\Theta_{U^2}$};
            \node (D1_mec) [relevanceb, below left = 1.4cm and 0.7cm of U1_mec] {$\Pi_{D^1}$};
            \node (D2_mec) [relevanceb, right = of D1_mec] {$\Pi_{D^2}$};
            \edge {D2_mec} {D1_mec};
            \edge {U1_mec} {D1_mec};
            \edge {U2_mec} {D2_mec};
        \end{influence-diagram}
        \caption{}
        \label{fig:subgames:2}
    \end{subfigure}
    \begin{subfigure}[b]{0.2\linewidth}
        \centering
        \begin{influence-diagram}
            \node (U1) [] {$U^1$};
            \node (U2) [utility, right = of U1, player2] {$U^2$};
            \node (D2) [decision, below right = 1.4cm and 0.7cm of U1, player2] {$D^2$};
            \edge {D2} {U1,U2};
        \end{influence-diagram}
        
        \vspace{0.5cm}

        \begin{influence-diagram}
            \node (U1_mec) [relevancew] {$\Theta_{U^1}$};
            \node (U2_mec) [relevancew, right = of U1_mec] {$\Theta_{U^2}$};
            \node (D2_mec) [relevanceb, below right = 1.4cm and 0.7cm of U1_mec] {$\Pi_{D^2}$};
            \edge {U2_mec} {D2_mec};
        \end{influence-diagram}
        \caption{}
        \label{fig:subgames:3}
    \end{subfigure}
    \begin{subfigure}[b]{0.15\linewidth}
        \centering
        \begin{influence-diagram}
            \node (U2) [utility, player2] {$U^2$};
            \node (D2) [decision, below left = 1.4cm and 0.7cm of U2, player2] {$D^2$};
            \edge {D2} {U2};
        \end{influence-diagram}

        \vspace{0.5cm}

        \begin{influence-diagram}
            \node (U2_mec) [relevancew] {$\Theta_{U^2}$};
            \node (D2_mec) [relevanceb, below left = 1.4cm and 0.7cm of U2_mec] {$\Pi_{D^2}$};
            \edge {U2_mec} {D2_mec};
        \end{influence-diagram}

        \caption{}
        \label{fig:subgames:4}
    \end{subfigure}
    \caption{(a) A MAID $\model = (\graph, \theta)$ for the modified version of Example \ref{ex:job_market} -- in which the firm's profits are also function of the worker's decision but not of their temperament -- and resulting $s$-relevance graph. The graph $\graph$ is also an (improper) $s$-subdiagram and the full game an (improper) $s$-subgame. Figures (b), (c), and (d) illustrate the remaining (proper) $s$-subdiagrams of $\graph$ and their $s$-relevance graphs.}
    \label{fig:subgames}
\end{figure}

$\R$-subgames can be found for any $\R$ using only $\R$-reachability. In particular, $s$-reachability produces $s$-subgames, which in many ways are the most natural form of subgame in MAIDs (because of their connection to subgames in EFGs, as shown formally in Section \ref{sec:equivalences}). In the remainder of the paper, unless otherwise specified, we therefore focus our attention on this case.
Note also that any MAID is an $\R$-subgame of itself, and so an $\R$-subgame on a strictly smaller set of variables is called a \emph{proper} $\R$-subgame. For example, the MAID for Example \ref{ex:job_market} (shown in Figure \ref{fig:job_market_game:MAID}) has no proper $\R^{\BR}$-subgames because $\Pi_{D^1}$ and $\Pi_{D^2}$ are both $\R^{\BR}$-reachable from one another.

\subsubsection*{Identifying More Subgames in MAIDs}
Before continuing, we note that the conditional dependencies captured in MAIDs allow for a richer and stronger notion of subgames than in EFGs. Not only can different notions of $\R$-subgame be introduced for different rationality relations, but it is often possible to identify more subgames (and hence rule out more non-credible threats) in a MAID than in the corresponding EFG. This can be seen by considering a minor variation on Example \ref{ex:job_market}. Suppose that the firm has a new vacancy in which what is important is not whether the worker is hard-working or lazy, but whether they have studied at university. In other words, $U^2$ is a function of $D^1$ and $D^2$ but not $T$. 

The game graph and $s$-relevance graph for this example are shown in Figure \ref{fig:subgames:1}. Note that the \emph{structure} of the EFG (shown in Figure \ref{fig:job_market_game:EFG}) does not change, only the payoffs for the firm. As such, there are no proper subgames in the EFG because when the hiring system is making its decision $(D^2)$, it cannot observe the value of $T$ and so no proper subtree is closed under both descendants and information sets. In contrast, we can recognise three proper $s$-subdiagrams (shown in Figures \ref{fig:subgames:2}, \ref{fig:subgames:3}, and \ref{fig:subgames:4}) of the equivalent MAID. Each of these $s$-subdiagrams has two $s$-subgames associated with it owing to the two values that $T$ can take (for the $s$-subdiagram in Figure \ref{fig:subgames:2}) and the two values that $D^1$ can take (for the $s$-subdiagrams in Figures \ref{fig:subgames:3} and \ref{fig:subgames:4}).

\subsection{Equilibrium Refinements}
\label{sec:equilibrium_refinements}
    
{When more than one NE exists, it is useful to specify additional criteria to rule out less plausible outcomes}. This corresponds to making additional assumptions about the rationality of agents, which can be encoded as stricter rationality relations. Below, we provide definitions of two of the most important equilibrium refinements -- subgame perfect equilibria \cite{selten1965spieltheoretische} and trembling hand perfect equilibria \cite{selten1974reexamination} -- within MAIDs. Later, in Section \ref{sec:equivalences}, we provide proofs regarding various equivalences between these definitions and those in EFGs.

\subsubsection*{Subgame Perfect Equilibrium}
The concept of a subgame perfect equilibrium (SPE) was introduced into EFGs in order to eliminate NEs containing \emph{non-credible threats} -- choices made by an agent in a sequential game that would not be in their best interest to carry out if given the opportunity \cite{selten1965spieltheoretische, selten1974reexamination}. In MAIDs, we can rule out non-credible threats by ensuring that each agent plays a best response in every feasible $s$-subgame. In games with sufficient recall, an SPE (in behavioural policies) can always be constructed by performing backwards induction over the $s$-subgames and finding an NE in each; {it is this technique in Appendix \ref{app:subgamecompute} that allows an NE to be computed much more efficiently than in a corresponding EFG.} However, if even one agent in the game has insufficient recall, an SPE may not exist even when allowing for mixed policies {(see Appendix \ref{app:non-existence_proofs} for an example)}.

\begin{definition}
    \label{def:SPE}
    A policy profile ${\bm{\pi}}$ in a MAID $\model$ is a \textbf{subgame perfect equilibrium (SPE)} if ${\bm{\pi}}$ is an NE in every feasible $s$-subgame of $\model$, when restricted to that subgame.\footnote{Note that this notion of subgame perfectness can be generalised to other choices of rationality relations.}
\end{definition}
\begin{proposition}
    \label{prop:SPEexist}
    Any MAID $\model$ with sufficient recall has at least one SPE in behavioural policies.
\end{proposition}

Recall $\bm{\pi}^\NE$ in Example \ref{ex:warehouse}, introduced in Section \ref{sec:NE}, in which robot one chooses $D_1 = q$ and robot two chooses $D_2 = p$ whatever the value of $D_1$. We can immediately see that in the feasible $s$-subgame where $D_1 = \neg q$ -- the $s$-subdiagram for the smallest such $s$-subgame is shown in Figure \ref{fig:warehouse:b} -- then choosing $D_2 = p$ is a non-credible threat, resulting in expected utility $5$ instead of $6$ for robot two. Instead, an example of an SPE is the policy profile ${\bm{\pi}}^{\text{SPE}}$ in which robot one chooses $D_1 = \neg q$, and robot two chooses $D_2 = p$ if and only if $D_1 = q$. Under such a policy profile, robot two achieves its optimal expected utility in any of the feasible $s$-subgames, and given that robot two is following this policy, robot two receives expected utility $2$ regardless of whether they move quickly or not.

Before continuing, we note that rationality relations allow us to capture arbitrary sets of policy profiles as rational outcomes, including the equilibrium refinements in this section. For example, (when each agent has at most one decision) the rational outcomes $\R^\BR(\mec\model)$ are simply the NEs of $\model$, as can easily be seen via inspection of Equation (\ref{eq:BR_relation}) and Definition \ref{def:NE}. Similarly, for a MAID $\model$ let us denote by $\model(D)$ the set of all feasible $s$-subgames containing $D$, and define:
\begin{equation*}
    \label{eq:SP_relation}
    \pi_{D} \in r^{\SP}_{D}(\pa_{\Pi_{D}}) ~~\Leftrightarrow~~ \pi_{D} \in \argmax_{\hat{\pi}_D \in \dom(\Pi_{D})} \sum_{U \in \bm{U}^i \cap \bm{V}'} \expect_{(\hat{\pi}_{D}, \bm{\pi}'_{-D})} [ U ] ~\forall~ \model' \in \model(D),
\end{equation*} 
expressing that each agent plays a best response for their decision rule in every $s$-subgame containing that decision. In other words, $\bm{\pi} \in \R^\SP(\mec\model)$ if $\bm{\pi}' \in \R^\BR(\mec\model')$ for each $s$-subgame $\model'$ of $\model$, where $\bm{\pi}'$ is $\bm{\pi}$ restricted to the decision variables in $\model'$. If $\model$ has sufficient recall, $\R^{\SP}(\mec\model)$ are the SPEs of the game, as stated formally below.
While such representations may sometimes be slightly more cumbersome, encoding equilibria via rationality relations facilitates the use of $\R$-relevance and hence $\R$-reachability. This offers a principled way to identify independencies that can be useful both for causal and game-theoretic reasoning.

\begin{proposition}
    Suppose that a MAID $\model$ has sufficient recall. Then, the set of SPEs of $\model$ is equal to the set of rational outcomes $\R^\SP(\mec\model)$.
\end{proposition}

\subsubsection*{Trembling Hand Perfect Equilibrium}
In an SPE, agents make decisions on the assumption that an SPE will be played in all proper subgames. As a result, however, the optimality of their strategies may not be robust to events in which other agents make mistakes, or `tremble', with some small probability. To solve this problem, we can stipulate that each agent must play a best response (leading to an NE) in each \emph{perturbed game} \cite{selten1974reexamination}. Let $\zeta_k$ be a perturbation vector containing, for every $D \in \bm{D}$, $d \in \dom(D)$, and decision context $\pa_D$, a value $\epsilon^{\pa_D}_d \in (0,1)$ such that $\sum_{d \in \dom(D)} \epsilon^{\pa_D}_d \leq 1$. Then, given a game $\model$, the perturbed game $\model(\zeta_k)$\label{def:pertMAID} is defined such that each decision rule $\pi_D$ is forced to have $\pi_D(d \mid \pa_D) \geq \epsilon^{\pa_D}_d$.
    
\begin{definition}
    \label{def:THPE}
    A policy profile ${\bm{\pi}}$ is a \textbf{trembling hand perfect equilibrium (THPE)} in a MAID $\model$ if there is a sequence of perturbation vectors $\{\zeta_k\}_{k\in\mathbb{N}}$ such that $\lim_{k \rightarrow \infty}\Vert\zeta_k\Vert_\infty = 0$ and for each perturbed MAID $\model(\zeta_k)$ there is an NE ${\bm{\pi}}_k$ such that $\lim_{k \rightarrow \infty} {\bm{\pi}}_k = {\bm{\pi}}$.
\end{definition}

For example, the policy profile ${\bm{\pi}}^{\text{SPE}}$ is \emph{not} a THPE. To see this, suppose that robot two trembles with probability $\epsilon > 0$ when $D_1 = q$ and $\epsilon' > 0$ when $D_1 = \neg q$. Then, robot one's expected utility when playing $q$ is $2 + 2\epsilon$ and when playing $\neg q$ is $2 - \epsilon'$. Therefore in any NE of the perturbed game robot one will play $q$ with probability one. One THPE is given by the policy profile ${\bm{\pi}}^{\text{THPE}}$ in which robot one chooses $D_1 = q$ and robot two chooses $D_2 = p$ if and only if $D_1 = q$. Note that in any two-agent game, THPEs rule out all weakly dominated policies. In this example, robot one's policy of always choosing $D^1 = \neg q$ is weakly dominated by the policy of always choosing $D^1 = q$ (which itself is not weakly dominated by any other policy).

\section{Connections to EFGs}
\label{sec:EFG_connections}

EFGs are perhaps the most widely studied model of dynamic strategic decision-making. Despite their intuitive appeal, these tree-based models can be less concise and reveal less of the underlying structure of a game than the DAG-based models we use in this paper. With that said, a natural question is whether the game-theoretic definitions for MAIDs in Section \ref{sec:subgames_eqs} successfully capture the familiar concepts of their EFG counterparts, or whether we might lose something by working with MAIDs -- and hence (S)CGs -- instead. 

In this section, we show that such concerns are unwarranted: these game-theoretic concepts are preserved when we convert between representations. We begin by briefly describing two algorithms, \texttt{maid2efg} and \texttt{efg2maid}, that implement these conversions. Using these procedures, we then provide equivalence results for the definitions from Section \ref{sec:subgames_eqs}. Following this, we briefly discuss how tree-based models can also be used for causal reasoning in certain cases.

\subsection{Transformations}
\label{sec:transformations}

We briefly summarise two procedures for converting between MAIDs and EFGs. A more formal treatment of both can be found in Appendix \ref{app:transformations}.

\subsubsection*{From MAID to EFG}
There are many ways to convert a MAID $\model = (\graph, \bm{\theta})$ into an EFG $\efg$, but these differ in their computational costs \cite{koller2003multi, pearl1988probabilistic}. The basic idea, as employed by K\&M, is to use a topological ordering $\prec$ over the variables of $\graph$ to construct $\efg$'s game tree by splitting on each of the variables in order. This splitting is required due to the fact that a variable in $\graph$ defines what happens given any context (a setting of the parent variables), whereas a node in $\efg$ defines what happens only in one context (the path taken to reach that node). As such, we may end up with exponentially more nodes in $\efg$ than variables in $\graph$.

By querying the CPDs of each variable, branches from chance nodes are labelled with probabilities based on the path taken from the root of the tree to that node, and similarly for the utilities assigned to each leaf; branches from decision nodes are labelled with the possible actions available. Given two nodes corresponding to the same decision variable $D$ in $\graph$, we assign them to the same information set if and only if the values taken by their ancestors along the paths from the root of the tree to each node agree on all those variables in $\Pa_D$. Note that because there can be more than one topological ordering $\prec$\label{def:topo}, we regard the output of \texttt{maid2efg} as a \emph{set} of EFGs -- one for each possible ordering.

\begin{remark}
    This procedure can be made more efficient by marginalising out all variables not in  $\bm{U} \cup \Fa_{\bm{D}}$, such as in the reduced EFG for Example \ref{ex:warehouse} (shown in Figure \ref{fig:warehouse:c}), which only has $2^2$ leaves, as opposed to the $2^3$ leaves that would have resulted had we retained all the variables in the original MAID (in Figure \ref{fig:warehouse:a}). Importantly, the information in this reduced EFG is sufficient for computing its equilibria, and can thus offer significant efficiency gains (since the cost of solving an EFG depends on its size, which is exponential in the length of $\prec$). In our codebase \cite{pycid}, we therefore implement this more efficient transformation (originally described by K\&M), which can be used in conjunction with Gambit, a popular tool for solving EFGs \cite{mckelvey2006gambit}.
\end{remark}

\subsubsection*{From EFG to MAID}
By encoding the CPDs for each variable in the MAID using trees as opposed to tables, MAIDs can represent any game using at most the same (but often exponentially less) space than an EFG \cite{koller2003multi}. In general, there are many MAIDs that can represent a given EFG. For instance, upon converting the EFG representation (in Figure \ref{fig:job_market_game:EFG}) of Example \ref{ex:job_market} to a MAID, we could na\"{i}vely create a decision variable for each information set. Alternatively, we could recognise, for example, that $V^1_1$ and $V^1_2$ correspond to the same real-world variable -- the worker's choice to go to university or not -- and thus combine them (as shown in Figure \ref{fig:job_market_game:MAID}, in which the two information sets for the second agent have also been combined).

Whether two nodes represent the same variable in an EFG is a matter of domain knowledge external to the EFG, but this knowledge typically exists when creating a model of a game and is then lost by viewing the interaction as an EFG. By formalising this notion as a question of whether nodes belong to the same \emph{intervention set}, we provide a procedure \texttt{efg2maid} (described fully in Appendix \ref{app:EFG2MAID}), which maps an EFG to a \emph{unique}, canonical MAID.

\begin{definition}
\label{def:intervention_set}
    An \textbf{intervention set} $J$ \label{def:interset} in an EFG $\efg$ is either a set of chance nodes, information sets belonging to the same agent, or leaves, such that:
    \begin{itemize}
        \item Every node $V \in J$ or $V \in I^i_J \in J$ has the same number of children;
        \item No path from the root of $\efg$ to one of its leaves passes through $J$ more than once.
    \end{itemize}
    {Moreover, for a valid partition of information sets and chance variables into intervention sets, we require that} any path to a node $V \in J$ or $V \in I^i_J \in J$ passes through the same intervention sets {before reaching $J$}.
\end{definition}

In \texttt{efg2maid}, we assume that intervention sets are given; in the absence of such knowledge, one can simply choose each intervention set to be a singleton. The resulting MAID will remain equivalent to the EFG (in the sense described in the following subsection), but may not be as compact as it otherwise could be. The basic idea behind the conversion is to first view the game tree as a MAID, then to add missing edges to each node from its ancestors whilst merging nodes that are in the same intervention set into single variables. Incoming edges from any variable $V$ whose value is not observed by a decision $D$ are then removed from $\Pa_D$, along with any duplicate edges produced by merging nodes. Each leaf intervention set is split into one utility variable for each agent. 
The distributions over each variable $V$ are then formed for each context $\pa_V$ by summing the relevant probabilities of sets of merged edges.

\subsection{Equivalences}
\label{sec:equivalences}

Using these transformations, we derive equivalence results between EFGs and MAIDs to demonstrate that the fundamental game-theoretic notions of subgames and various equilibrium refinements in Section \ref{sec:subgames_eqs} are \emph{preserved} when converting between representations. We begin by defining what we mean for two game representations to be `equivalent'; the underlying idea is that the (behavioural) policy and strategy spaces in each game should be the same and each corresponding (behavioural) policy and strategy profile should lead to the same expected utility for each agent in both games. {An immediate consequence of this is that NEs are also preserved.}

\begin{definition}
    \label{def:equiv}
    We say that a MAID $\model$ is \textbf{equivalent} to an EFG $\efg$ (and vice versa) if there is a bijection {for each agent $f^i: \Sigma^i \rightarrow \dom(\bm{\Pi}^i) / \sim$} between {their} {strategies} in $\efg$ and a partition of {their} {policies} in $\model$ (the quotient set of {$\dom(\bm{\Pi}^i)$} by an equivalence relation $\sim$ where {${\bm{\pi}^i}\sim{\hat{\bm{\pi}}^i}$} if and only if {$\bm{\pi}^i$} and {$\hat{\bm{\pi}}^i$} differ only on infeasible decision contexts) such that for every ${\bm{\pi}} \in f(\sigma)$ and every agent $i$, we have $\expect_{\sigma} \big[ U(\rho[\bm{L}])[i] \big] = \sum_{U \in \bm{U}^i} \expect_{{\bm{\pi}}} [ U ]${, where $f(\sigma) \coloneqq \bigtimes_{i \in N} f^i(\sigma^i)$}. We refer to any such $f$ as a \textbf{natural mapping}\label{def:natmap} between $\efg$ and $\model$.
\end{definition}

\begin{lemma}
    \label{lem:NEpreservation}
    {Let $f$ be a natural mapping between $\efg$ and $\model$. Then $\sigma$ is an NE in $\efg$ if and only if every ${\bm{\pi}} \in f(\sigma)$ is an NE in $\model$.}
\end{lemma}

The reason we use an equivalence relation on the space of policies is that \texttt{efg2maid} can introduce additional infeasible decision contexts: those $\pa_D$ such that $\Pr^{\bm{\pi}}(\pa_D) = 0$ for all $\bm{\pi}$, corresponding to paths in the EFG that do not exist. An agent's choice in such decision contexts has no bearing on the outcome of the game, meaning we may safely abstract away from them. Any natural mapping is therefore sufficient for preserving the essential game-theoretic features of each representation, as we show below, though it may not be unique. We begin with a supporting lemma that justifies the correctness of the procedures \texttt{maid2efg} and \texttt{efg2maid}, and forms the basis of our other results.

\begin{lemma}
    \label{lem:correspondence}
    If $\efg\in \mathtt{maid2efg} (\model)$ or
    $\model =\mathtt{efg2maid}(\efg)$, then $\efg$ and $\model$ are equivalent.
\end{lemma}

This lemma follows directly from the construction of a natural mapping $f$ using the two procedures, \texttt{maid2efg} and \texttt{efg2maid} respectively. The intuition is that the information sets in an EFG correspond to the feasible decision contexts in a MAID, and thus a behavioural strategy profile $\sigma$ in the EFG corresponds to a behavioural policy profile ${\bm{\pi}}$ in the MAID, and vice versa. {The following result is a direct corollary of Lemma \ref{lem:NEpreservation} and Lemma \ref{lem:correspondence}.}

\begin{corollary}
\label{prop:BE}
    If $\efg\in\mathtt{maid2efg}(\model)$ or $\model = \mathtt{efg2maid}(\efg)$, then there is a natural mapping $f$ between $\efg$ and $\model$ such that $\sigma$ is an NE in $\efg$ if and only if every ${\bm{\pi}} \in f(\sigma)$ is an NE in $\model$.
\end{corollary}

For a subgame $\efg'$ in an EFG, it can be shown that the variables outside $\efg'$ are not $s$-relevant to those in the corresponding $s$-subgame $\model'$. This means that subgames in EFGs have equivalent counterparts in their equivalent MAID, as established by the following proposition. A minor caveat here is that some utility variables may not occur in $\model'$, and so we must add their value (under the setting of the variables outside $\graph'$ that defines $\model'$) to the payoff for the agents in order to equate their expected utilities with that of $\efg'$. For any subgame $\model'$, however, this change in value is constant for each agent under any policy profile, and so has no effect on their decisions.

\begin{proposition}
\label{prop:subgames}
    If $\efg\in\mathtt{maid2efg}(\model)$ or $\model = \mathtt{efg2maid}(\efg)$, then there is a natural mapping $f$ between $\efg$ and $\model$ such that, for every subgame $\efg'$ in $\efg$ there is an $s$-subgame $\model'$ in $\model$ that is equivalent (modulo a constant difference between the utilities for each {agent} under any policy in $\model'$) to $\efg'$ under the natural mapping $f$ restricted to the strategies of $\efg'$.
\end{proposition}

This restriction of $f$ to the strategies in $\efg'$ can be made precise by considering only those feasible decision contexts that correspond to the information sets contained in $\efg'$.
{By first applying Proposition \ref{prop:subgames} and then Lemma \ref{lem:NEpreservation} to each of the resulting subgames, we see that any SPE in $\model$ is carried over to $\efg$.} We note, however, that as there may be more $s$-subgames in a MAID than in its equivalent EFG, the criterion of subgame perfectness may be slightly stronger in the MAID. In other words, not all SPEs in an EFG may be SPEs in the equivalent MAID. This additional strength can be useful in ruling out non-credible threats even when they do not fall under a particular subgame in the EFG.

\begin{corollary}
\label{prop:SPE}
    If $\efg\in\mathtt{maid2efg}(\model)$ or $\model = \mathtt{efg2maid}(\efg)$, then there is a natural mapping $f$ between $\efg$ and $\model$ such that if every ${\bm{\pi}} \in f(\sigma)$ is an SPE in $\model$, then $\sigma$ is an SPE in $\efg$.
\end{corollary}
 
Finally, we derive an equivalence between the THPEs in EFGs and those in MAIDs. In order to do so, it suffices to prove an equivalence between perturbed versions of the corresponding games $\efg(\zeta_k)$ and $\model(\zeta_k)$, which can easily be done by construction using \texttt{maid2efg} and \texttt{efg2maid}, and then by applying Lemma \ref{lem:correspondence} {and Lemma \ref{lem:NEpreservation} to each of these perturbed games}.
 
\begin{proposition}
\label{prop:THPE}
    If $\efg\in\mathtt{maid2efg}(\model)$ or $\model = \mathtt{efg2maid}(\efg)$, then there is a natural mapping $f$ between $\efg$ and $\model$ such that $\sigma$ is a THPE in $\efg$ if and only if every ${\bm{\pi}} \in f(\sigma)$ is a THPE in $\model$.
\end{proposition}

This series of equivalence results serves to justify MAIDs as an appropriate choice of game representation. Not only do they provide computational advantages over EFGs, they preserve some of the most fundamental game-theoretic concepts commonly employed in EFGs.

\subsection{Causality in EFGs}
\label{sec:EFG_causality}

One alternative to causal games would be to define causal concepts directly via EFGs. In this section, we explain how this can be done for a limited set of causal queries by extending methods designed for probability trees \cite{Genewein2020}, and explain some shortcomings of this approach relative to the models we focus on in this paper.

\subsubsection*{Interventions}
Intuitively, interventions in EFGs correspond to replacing the probability distribution(s) governing some node(s) in the game tree.
More formally, consider an EFG $\efg$ with intervention sets (if each intervention set is a singleton, we recover the case where each EFG variable is treated separately). As in CGs, we can apply both pre- and post-strategy interventions, and both operations have the same form. We make an intervention $\mathcal{I}$ on an intervention set $J$ \label{def:intint}by replacing $P_j$ or $\sigma^i_j$ with an arbitrary probability distribution $P^*_j(V_j \mid \cap_{V_j \in J}\Anc_{V_j})$ or ${\sigma^i_j}^*(V_j \mid \cap_{V_j \in J}\Anc_{V_j})$, for each $V_j \in J$. Hard interventions represent the special case when this distribution is given by $\delta(V_j {,} v_j)$, where $v_j$ corresponds to an outgoing edge from each node $V_j \in J$.
Pre-strategy interventions result from performing an intervention on the game tree \emph{before} agents choose their strategies -- in essence, forming a new game $\efg_{\mathcal{I}}$ -- whereas post-strategy interventions result from intervening on the probability tree that results \emph{after} agents have chosen their strategies -- in other words, forming a new distribution $P^\sigma_{\mathcal{I}}$ \label{def:poststratefg} over paths in the tree.

While the graph of an EFG is typically viewed as encoding informational constraints upon the decisions of agents, as opposed to temporal or causal structure, it can also be given a causal interpretation in much the same way that a MAID or BN can, by restricting the EFG to be consistent with the set of all possible (hard) interventional distributions. This restriction produces a level two model in which the interventional queries described in the paragraph above are semantically meaningful. Due to the total ordering imposed by EFGs over the variables in a game, however, only a restricted subset of interventions are available (those whose distributions $\mathcal{I}$ condition only on \emph{ancestors} of the node in question). Similarly, some queries (be they probabilistic or causal) over EFGs are not well-defined, due to the fact that, in general, not all variables are assigned a value on a path $\rho$ from the root to a leaf.

\subsubsection*{Counterfactuals}
Answering counterfactual queries is slightly more complicated. We briefly describe an existing procedure for probability trees, and refer the interested reader to this work for further details \cite{Genewein2020}. Suppose that we wish to calculate $P^\sigma(\bm{x}_{\I} \mid z)$ -- the value that $\bm{X} = \bm{V} \setminus \{Y,Z\}$ would have taken if $Y$ were distributed according to $\I$, given that in fact $z$ is true. The first step is to condition on $z$ by setting $P_Z(z) = 1$ or $\sigma^i_Z(z) = 1$, which is equivalent to first performing a hard intervention $\Do(Z = z)$ and then renormalising the probability distributions over the branches \emph{upstream} of this intervention. 
The second step is to perform the soft intervention $\I$ on $Y$, without any renormalising. Thirdly and finally, we restore the probability distributions over the branches \emph{downstream} of this second intervention to the original settings given by $P^\sigma$. The resulting probability tree describes the \emph{post-strategy} counterfactual distribution $P^\sigma(\bm{x}_{\I} \mid z)$. 
By modifying this procedure such that the normalisation upstream of $Z$ and the restoring of distributions downstream of $Y$ takes place according to the actual strategy profile $\sigma$, then allowing agents to play a new counterfactual strategy profile $\sigma'$, we produce a \emph{pre-strategy} counterfactual distribution (though if conditioning is performed downstream of differences between $\sigma$ and $\sigma'$, the resulting distribution will not, in general, be correct).

Counterfactual queries in EFGs are \emph{not} invariant to the choice of tree representation, even if the distributions $P^\sigma$ are the same for any strategy profile $\sigma$. This difficulty arises because standard EFGs do not reside at level three of the causal hierarchy, but may be overcome (as in CBNs and CGs) by relegating all stochasticity to a set of exogenous chance nodes that appear at the top of the game tree, and allowing agents only to select deterministic strategies (over endogenous decision nodes) as a function of some exogenous noise. This also resolves the difficulties of pre-policy counterfactuals in EFGs, as updates made due to conditioning on $z$ will now be \emph{upstream} of differences between $\sigma$ and $\sigma'$. Existing work on probability trees \cite{Genewein2020}, on the other hand, treats stochasticity as endogenous, and only possible to learn about (in the counterfactual world) for those variables that are \emph{upstream} of the intervention $\I$. The ordering of variables in the tree therefore implicitly represents a substantive modelling assumption that resolves the subtle complexities surrounding counterfactuals in models with intrinsic stochasticity \cite{pearl2009causality,Dawid2000}.

\section{Applications}
\label{sec:applications}

Having described and analysed the theoretical features of causal games, we now consider potential domains of application. {In the first subsection, we provide a case study based on the UK home insurance market. In the second, we discuss how several existing concepts for the analysis of (artificial) agents can be embedded and further developed within our framework. We provide worked examples of these embeddings in Appendix \ref{app:previous_concepts}.} 
More substantive results{, however,} are beyond the scope of this paper{, whose fundamental aim is to lay the theoretical foundations for causal games, rather than to explore specific applications.}

\subsection{{Case Study: Insurance Pricing}}

{While our primary motivation is to use causal games to analyse AI systems, another natural domain of application is in economics, where both causal and game-theoretic modelling are common. For example, when a regulator in the UK proposes a policy intervention, they are required to conduct a cost-benefit analysis to justify it. This entails making a comparison between the consequences of the intervention and the consequences of no intervention and/or alternative interventions. Some regulators, such as the Financial Conduct Authority (FCA), include a `causal chain' in their analyses, though these are only heuristic diagrams rather than formal causal models. In this case study, we demonstrate the usefulness of causal games for economic analysis with an example based on the UK home insurance market, worth £6.31 billion in 2021 \cite{GlobalData2022}, which was the focus of a recent FCA report \cite{Authority2020}.}

\subsubsection*{{The Model}}

{In this setting: some customers are inert and do not change their insurance provider from year to year; firms charge a low price to attract customers in the first year, then increase their prices upon renewal; inert customers, therefore, end up paying more, while savvy customers spend time negotiating or switching provider each year. This allows firms to increase profits at customers' expense, either from increased prices (if inert), or switching costs (if savvy). These practices are especially troubling because it is often more vulnerable customers (such as older or less educated individuals) who are likely to be exploited \cite{LondonEconomics2019}. While this is not the focus of the aforementioned FCA report, it is conceivable that the use of sophisticated AI algorithms to set insurance prices based on customer details might potentially lead to an even greater degree of exploitation. A simplified model of this setting is as follows.}

\begin{example}[{Insurance Market}]
    \label{ex:insurance}
    {A customer in a duopolistic insurance market must choose an insurance provider, and aims to minimise their costs. They already have a contract with one of the two providers, and so must decide whether to renew or switch. The customer is either savvy (in which case they are modelled as having a low switching cost) or inert (in which case they are modelled as having a high switching cost). The two firms use internal pricing algorithms to offer their quotes simultaneously, with the knowledge of which firm the customer has an existing contract with, but without the knowledge of the customer's type (savvy or inert), in order to maximise their profits.}
\end{example}

{We represent this setting as a causal game $\model$ with three agents (in Figure \ref{fig:insurance_model}) and parameterise the model using real-world data on pricing strategies and consumer tendencies in the UK home insurance market within the period 2013-18 \cite{Authority2019,Authority2020,LondonEconomics2019}. The two firms (agents one and two) use their pricing algorithms to decide prices $d^1,d^2 \in \mathbb{N}$ for their insurance premiums based on which firm the customer (agent three) currently has a contract with, $c \in \{1,2\}$.\footnote{{In practice, the firms' prices will be bounded between the marginal cost of supplying the policy, and the point at which the customer exits the market due to prices being too high.}} The customer uses this information, along with their type $T$ -- savvy ($s$) or inert ($\neg s$) -- to choose a firm $d^3 \in \{1,2\}$ to purchase insurance from, or to exit the market ($\bot$). 

Firm $i$'s utility (their profit) is assumed to be given by their policy price $d^i$ minus the marginal cost of providing the policy, if selected by the customer, and zero otherwise. The customer's utility is defined as their valuation of the policy, minus the price they pay and switching costs, with an extra $-\infty$ term if the customer is inert (i.e., to model the searching and switching cost being prohibitively high). The customer's utility is zero if they exit the market. The chance variables $C$ (the customer's current firm) and $T$ (the customer's type) are parameterised following the values in Table \ref{tab:insurance_stats}.}

\begin{figure}[h]
    \centering
    \begin{subfigure}[b]{0.3\linewidth}
        \centering
        \begin{influence-diagram}
            \node (C) [] {$C$};
            \node (T) [below = of C] {$T$};
            \node (D1) [decision, above right = 0.7cm and 1.4cm of C, player1] {$D^1$};
            \node (D3) [decision, below = of D1, player3] {$D^3$};
            \node (D2) [decision, below = of D3, player2] {$D^2$};
            \node (U1) [utility, right = of D1, player1] {$U^1$};
            \node (U3) [utility, right = of D3, player3] {$U^3$};
            \node (U2) [utility, right = of D2, player2] {$U^2$};
            \draw (C) edge[dashed, ->] (D1);
            \draw (C) edge[dashed, ->] (D2);
            \edge {C} {D3};
            \edge {T} {D3};
            \edge {D1} {U1,D3};
            \edge {D2} {U2,D3};
            \edge {D3} {U1,U2,U3};
            \path (D1) edge[->, bend right=-13] (U3);
            \path (D2) edge[->, bend right=13] (U3);
            \path (C) edge[->, bend right=-20] (U3);
            \path (T) edge[->, bend right=20] (U3);
        \end{influence-diagram}
        \caption{}
        \label{fig:insurance_model}
    \end{subfigure}
    \begin{subfigure}[b]{0.6\linewidth}
        \centering
        \begin{tabular}{l l}
            \toprule
            \textbf{Average Quantity} & \textbf{Amount}\\
            \midrule
            Marginal cost of supplying policy\footnotemark{} & £206 \\
            Customer valuation of policy & £296 \\
            Switching cost for customer & £38 \\
            Savvy customers\footnotemark{} & 72\% \\
            Customers currently with first firm & p\% \\
            \bottomrule
        \end{tabular}
        \caption{}
        \label{tab:insurance_stats}
    \end{subfigure}
    \caption{{(a) A causal game representing Example \ref{ex:insurance}. (b) Data from the UK home insurance market within the period 2013-18 \cite{Authority2019,Authority2020,LondonEconomics2019}, used to parameterise the game.}}
    \label{fig:insurance_game}
\end{figure}

\footnotetext[21]{{This quantity is estimated as the average amount received per policy by firms (£231) minus the average profit per policy (£25) \cite{Authority2020}. In practice, there are other small fees that affect a firm's revenue.}}
\footnotetext[22]{{Taken as those who self-report as `actively' searching for new policies \cite{Authority2020}.}}

\subsubsection*{{The Proposed Intervention}}

{In the aforementioned report, the FCA proposes to ban insurance companies from setting different prices for new and existing customers \cite{Authority2020}. Because firms are made aware of the intervention before deploying their pricing algorithms, we represent this as a pre-policy intervention on $\Pi_{D^1}$ and $\Pi_{D^2}$. More concretely, this corresponds to changing the rationality relations $r_{D^1}$ and $r_{D^2}$ by restricting their codomains -- $\dom(\Pi_1)$ and $\dom(\Pi_2)$ -- to be such that $\pi_{D^1}(D^1 \mid C) = \pi_{D^1}(D^1)$ and $\pi_{D^2}(D^2 \mid C) = \pi_{D^2}(D^2)$, effectively removing the edges $C \to D^1$ and $C \to D^2$ (shown as dashed in Figure \ref{fig:insurance_model}). In what follows we denote this intervention by $\I$.}
{The FCA offers a number of hypotheses about the likely effects of this intervention, namely that:
\begin{itemize}
    \item[i)] The percentage of customers switching will reduce;
    \item[ii)] Switching customers will end up paying higher prices, and renewing customers will end up paying lower prices;
    \item[iii)] Overall, customers will end up better off on average.
\end{itemize}
Representing the intervention within a causal game $\model$ allows us to express these hypotheses formally by measuring and quantifying the causal effects of $\I$, which cannot, in general, be done without a formal causal model. For the purpose of this analysis, we set $p = 0.6$ and assume that agents play a pure THPE, which, while not truly realistic, illustrates the real-world applications of causal games (which can also be used to capture boundedly rational agents). In what follows, we denote the rationality relations describing this assumption as $\R$.}

{Before addressing the three hypotheses, it is a simple exercise to see that for each $\bm{\pi} \in \R(\model)$, firms do indeed end up setting minimal prices (i.e., equal to their marginal cost) for new customers and higher prices for renewing customers, up to the point of driving away sales due to switching costs from customers. In other words, we have that $\pi_{D^i}(D^i \mid C = 3-i) = \delta(D^i,206)$ and $\pi_{D^i}(D^i \mid C = i) = \delta(D^i,244)$ for each $i \in \{1,2\}$. For their part, customers always seek the lowest price, taking into account switching costs, and are indifferent when the price at an alternative firm is the same as the price at their current firm plus the switching cost. After the intervention $\I$, the customer's policy stays the same, but we have $\pi_{D^1}(D^1 \mid C) = \delta(D^1,288)$ and $\pi_{D^2}(D^2 \mid C) = \delta(D^2,250)$ for the two pricing algorithms, respectively, for any $\bm{\pi} \in \R(\model_{\I})$.}

{The rational outcomes of the original game correspond to whether savvy customers (at each of the two firms) switch to the other firm or not. Depending on whether customers with neither firm, the first firm, the second firm, or both firms switch, then we have 0\%, 43.2\%, 28.8\%, or 72\% of customers switching, respectively. After intervention $\I$, customers already with the second firm have no incentive to switch, and so the percentage of customers switching is either 0\% or 43.2\%. If we were to assume a uniform prior over rational outcomes, for instance, then we would see that the causal effect of $\I$ on the percentage of customers switching insurance provider is given by:
$$\frac{1}{\vert \R(\model_{\I}) \vert} \sum_{\bm{\pi} \in \R(\model_{\I})} \text{switch}(\bm{\pi}) - \frac{1}{\vert \R(\model) \vert} \sum_{\bm{\pi} \in \R(\model)} \text{switch}(\bm{\pi}) = 21.6\% - 36\% = -14.4\%,$$
where $\text{switch}(\bm{\pi}) \coloneqq \sum_i \Pr^{\bm{\pi}}(D^3 = i \mid C = 3 -i)$ is the probability that a customer switches provider. Thus, hypothesis i) is confirmed.}

{For the first part of hypothesis ii), we wish to measure the causal effect of $\I$ on the quantity $\expect_{\bm{\pi}} [ D^i \mid C = i, D^3= 3 - i]$ for each $i \in \{1,2\}$, in settings such that $\Pr^{\bm{\pi}}(C = i, D^3=3-i) > 0$. For $\bm{\pi} \in \R(\model)$, this quantity is equal to 206, which may occur when either $i = 1$ or $i = 2$. For $\bm{\pi} \in \R(\model_{\I})$ this quantity is equal to 250, which occurs when $i = 2$; the customer never switches to firm one. This confirms the first half of the hypothesis. For the second part, a similar analysis shows that renewing customers pay £244 before the intervention either £250 or £288 after the intervention, which is \emph{contrary} to the hypothesis. We discuss the reasons for this in the following paragraph, as well as how another proposal can be employed to achieve the desired result.}

{Hypothesis iii) concerns simply the causal effect of $\I$ on $\expect_{\bm{\pi}}[U^3]$. The customer's expected utility is the same under any of the rational outcomes and interventional rational outcomes respectively, so we can write simply:
$$\expect_{\bm{\pi}}[U^3_{\I}] - \expect_{\bm{\pi}}[U^3] = 23.12 - 52 = -28.88.$$
As in the case for the second half of hypothesis ii), this \emph{disconfirms} the hypothesis (in our model). The reason for both of these results is that the relatively high switching costs for customers (£38) and the relatively high percentage of inert customers (28\%) mean that the firms' pricing algorithms are better off sticking to higher prices to exploit customers who will not switch, rather than lower prices in order to attract customers away from other firms.}

{To achieve the desired effects, we can instead consider another intervention proposed by the FCA, which is to reduce switching costs for customers (for example, by allowing them to stop their insurance policy from auto-renewing more easily) \cite{Authority2020}. If we therefore consider a second intervention $\I'$ that makes the same change as $\I$ but also modifies $\Theta_{U^3}$ to set the switching costs of customers to £0, for example, then we have that:
$$\expect_{\bm{\pi}}[U^3_{\I'}] - \expect_{\bm{\pi}}[U^3] = 72 - 52 = 20,$$
as $\pi_{D^i}(D^i \mid C) = \delta(D^i,224)$ for $i \in \{1,2\}$ in any $\bm{\pi} \in \R(\model_{\I'})$. Returning to hypothesis ii), this implies that all customers now pay £224 when renewing as opposed to £244, thus completing the set of all desired outcomes.}

\subsubsection*{{Extensions}}

{As shown above, not only can our analysis be used to investigate the qualitative hypotheses put forwards by the FCA, it provides quantitative results too. Nevertheless, we emphasise that the model we present is highly simplified, and considers only a very simple intervention. More complex models and queries are beyond the scope of the present paper, but represent a natural avenue for further applied work. Before continuing, we briefly remark on possible extensions.}

{First, we could extend the model to make it more realistic. This might include: increasing the number of firms; different marginal costs for different firms; different customer valuations of insurance policies from different firms; small costs to firms from customers switching; including intermediary firms; more explicit modelling of the firms' pricing algorithms; not making savviness a binary variable; and modelling the market over time instead of as a one-shot game.
With these added complexities to the model come greater opportunities to exploit the richness of causal games.
}

{In particular, in the simple model above, the lack of more intricate causal dependencies between the agents' actions means that applying a pre-policy intervention essentially reduces to re-solving a slightly different game (where the dashed edges in Figure \ref{fig:insurance_model} are removed). In larger games, it is both easier and more important to avoid this cost, which cannot (in general) be avoided in models such as EFGs that do not explicitly represent the causal structure of a game. We could also use the model for more advanced analyses, such as answering counterfactuals, identifying subgames, and computing mixtures of pre- and post-policy queries.}

\subsection{{Blame, Intent, Incentives, and Fairness}}
\label{sec:biif}

{Many of the most important concepts for reasoning about the safe and ethical use of AI systems are implicitly causal in nature. Moreover, in recent years, substantial progress has been made on formalising these concepts using causal models \cite{Halpern2018,Everitt2021,Kusner2017,Miller2021,Richens2022}. Because causal games are the first causal framework to explicitly capture multi-agent and game-theoretic reasoning, they open up the possibility of further work in these directions. In what follows, we explain how causal games might be applied to existing concepts in order to arrive at richer and more general results. Concrete examples of how these concepts can be modelled using causal games are provided in Appendix \ref{app:previous_concepts}.}

\subsubsection*{Blame and Intent}

SCMs have been fruitfully employed in order to formally define the notion of \emph{actual causation} (i.e., what it means for some event $c$ to, in fact, cause some event $e$) \cite{Halpern2005,Halpern2015}. Such ideas are not only of philosophical interest, but have been argued as being crucial for building AI systems that can automatically generate \emph{explanations}, either of their own workings or that of other agents \cite{Halpern2005a,Miller2021}. In recent years, this underlying theory has been used to formalise the concepts of blameworthiness and intention \cite{Halpern2018,Friedenberg2019}. As these works have highlighted, issues of blame and intent become particularly acute in multi-agent settings. 

{While} these concepts {can be} accommodated within (structural) causal games {(see Appendix \ref{app:previous_concepts}), they also} allow for an even richer formalism of blame and intent in multi-agent settings. More specifically, there are at least four generalisations we might make. Firstly, the use of sets of utility variables that define \emph{multiple} utility functions, one for each agent, would mean that costs (and hence both blame and intent) can be viewed from the perspective of multiple agents. Secondly, we could consider blame and intent with respect to \emph{strategies} rather than single actions, which corresponds to considering soft (post-policy) interventions as opposed to hard (post-policy) interventions. Thirdly, we might wish to consider \emph{pre-policy} interventions by using mechanised games. For example, if a pre-policy change to one agent's decision means that the only rational response for the second agent is something that causes a bad outcome, we might think that the first agent is at least somewhat to blame for this turn of events (although these concepts become more nuanced when there are cyclic dependencies between decision rule variables). Finally, all of these extensions can also be viewed in terms of \emph{coalitions} of agents and their strategies{, which has already begun to be explored in prior work \cite{Friedenberg2019}}.

\subsubsection*{Incentives}

In order to build intelligent agents that do not behave in undesirable ways, it is useful to be able to reason about their \emph{incentives} \cite{Deletang2021,Everitt2021,Langlois2021} and \emph{reasoning patterns} \cite{pfeffer2007reasoning,antos2012identifying}. For example, given a (possibly multi-agent) decision-making scenario, we might be interested in asking whether an agent has an incentive to base its policy on a protected attribute, or to influence a variable that we would prefer to be left alone. Similarly, notions of fairness are often naturally expressed in terms of whether a different outcome would result in some counterfactual situation.

Recent work has sought to formalise incentives by modelling such scenarios as single-agent causal games (i.e., causal influence diagrams). This work has developed a number of sound and complete graphical criteria for identifying Value of Information (VoI) \cite{Howard1966} and Value of Control (VoC) \cite{Shachter1986}. Sound and complete graphical criteria have also been introduced for two novel concepts: \emph{response incentives}, which indicate when a decision-maker can benefit from responding to a variable, and \emph{instrumental control incentives}, which establish whether an agent can benefit from manipulating a variable \cite{Everitt2021}.

{Existing work, however, has been} limited to the single-agent, single-decision setting; it is our hope that causal games will provide the basis for generalising this work to the multi-agent, multi-decision setting and that the use of mechanised games may lead to additional insights. For instance, for some incentive concept $C$ we may wish to ask questions such as `is there any $\R$-rational outcome in which agent $i \in N$ has a $C$ incentive with respect to a particular variable $X$ in the game?'. Given the fact that many proposals for safe AI systems are multi-agent \cite{everitt2019modeling}, this generalisation to the multi-agent setting marks an important next step in analysing agent incentives. {We provide an example of such a proposal -- cooperative inverse reinforcement learning \cite{HadfieldMenell2016} -- modelled as a causal game, in Appendix \ref{app:previous_concepts}.}

\subsubsection*{Fairness}

Another important and useful application of incentives is for reasoning about fairness, a number of popular and influential definitions of which are based explicitly on causal frameworks \cite{Kilbertus2017,Kusner2017,Zhang2018a,Nabi2018,ashurst2022fair}. Indeed, it can be shown that all optimal policies ${\bm{\pi}}^*$ in a single-decision SCIM are \emph{counterfactually unfair} \cite{Kusner2017} with respect to a protected attribute $A$ (meaning that a change to the protected attribute would change the decision made) if and only if there is an RI on $A$ \cite{Everitt2021}. 

The question of fairness arguably becomes even more important and interesting in the multi-agent setting, in which one not only has to ask whether or not a process is fair, but fair to \emph{whom}, and whether we might be forced to trade off fairness with respect to different agents. Interestingly, despite the large number of existing works investigating fairness in games (see, e.g., the seminal papers of Rabin \cite{Rabin1993} or Fehr and Schmidt \cite{Fehr1999}) and the recent insights gained from using causal definitions of fairness, to the best of our knowledge there has been no application of these definitions to the multi-agent setting. One possible (partial) explanation for this is that researchers have, until now, been lacking precisely the kind of models that we introduce.

\section{Discussion}
\label{sec:discussion}

In this paper, we have introduced a framework that we argue provides a unifying formalism for reasoning about causality in games. As mentioned in Section \ref{sec:introduction}, combinations of causal and game-theoretic reasoning have long been considered by various research communities, and so we conclude with a brief {summary of the relative advantages and disadvantages of causal games compared to other models,} before offering some final thoughts about future directions.

\subsection{{Advantages and Disadvantages of Causal Games}}
\label{sec:DAG_advantages}

{The primary benefit of causal games compared to prior work is that no previous framework captured both game-theoretic and causal features in a general, principled way. Alongside our discussion of related work in Section \ref{sec:related}, we expand briefly upon this claim by considering other causal and game-theoretic models in turn.}

\subsubsection*{{Causal Models}}

{Standard causal models do not take into account the presence of rational, self-interested agents. Doing so requires more than a simple re-labelling of some of the variables as decisions and utilities, as the presence of strategic, decision-making agents violates the standard assumption of independent causal mechanisms \cite{peters2017elements}, represented by the edges between mechanism variables in mechanised games. Alternative causal models that do include agents, such as CIDs \cite{howard2005influence,dawid2002influence,Everitt2021}, typically only consider a single agent.}

{To the best of our knowledge, the only exception to this rule is settable systems \cite{White2009,White2014}, though these models do not capture the multiplicity of equilibria that arise in game-theoretic analyses, and thus do not naturally support the analysis we seek to do in this paper. The focus in these works is instead on capturing lower-level algorithmic details, which may make them more appropriate when reasoning about the data and attributes of, say, an optimisation or machine learning process.}

{Our use of \emph{non-deterministic} mechanisms (arising whenever agents are indifferent between alternatives) also serves to generalise existing work on cyclic causal models \cite{Halpern1998,Bongers2016} to the relational setting \cite{Getoor2007a,Maier2013,Ahsan2022}. However, the cyclic dependencies in causal games are of a specific form due to the fact that they represent game-theoretic equilibria, which means that causal games are not necessarily an appropriate model for analysing the more arbitrary dynamical systems and sets of simultaneous equations studied in these works.}

{The fact that we build on top of MAIDs means that causal games inherit some of the existing game-theoretic concepts introduced in these models, such as NEs, while being consistent with the more standard probabilistic graphical models upon which MAIDs are based. This further allows us to derive notions of subgames and other equilibrium refinements, which are critical for game-theoretic reasoning and yet are not supported by other causal models. Such concepts have been historically underexplored in the context of graphical games when compared to, say, EFGs or strategic-form games.}

\subsubsection*{Game-Theoretic Models}
The most natural game-theoretic model with which to compare causal games is EFGs, which are tree-based models as opposed to DAG-based. For our purposes, the most important benefit of DAG-based models is that they can be more readily used to natively support a wide range of causal queries. This is bolstered by the use of $\R$-relevance graphs, which form an explicit representation of the causal dependencies between agents' decision rules; there is no such representation for these dependencies in EFGs.\footnote{{This further implies that we cannot take advantage of the structure of these dependencies in order to answer pre-strategy queries more efficiently.}} Mechanised games can be used to define a wide range of both pre- and post-policy queries in games. In contrast, as explained in Section \ref{sec:EFG_causality}, there are many causal queries that \emph{cannot} be {natively} answered in EFGs.

{DAG-based models also more compactly and explicitly represent} dependencies between variables, which can often only be understood in an EFG through inspecting the parameterisation of the game. Moreover, a DAG-based representation of a game need never be bigger than the corresponding EFG and can be exponentially smaller. On the other hand, if a game is highly asymmetric in its play paths (such as when if play proceeds down one path the game stops immediately, and on the other path it continues for several more moves), then this structure is not immediately observable from a {causal game}, and may effectively make many valuations of the variables irrelevant due to their inconsistency with the shorter game path.

In addition to their transparency, MAIDs allow us to further exploit the conditional independencies between variables using d-separation. Using $\R$-reachability, we may construct the $\R$-relevance graph for any game and find more subgames in a MAID than in its equivalent EFG. This can {significantly} reduce the computational complexity of solving a game {(as shown by the example in Appendix \ref{app:subgamecompute})}, offers analytical benefits, and provides a way to define a stronger subgame perfectness condition. On the other hand, in the case of \emph{context-specific independencies} -- such as when $A$ sometimes depends on $B$, and $B$ sometimes depends on $A$ -- it is well-known that DAG-based models are a less natural choice than tree-based models \cite{Boutilier1996}.\footnote{Though it is possible to support said independencies in DAGs via tree-based representations of the CPDs, which can graphically capture different independencies on different branches.}

Finally, though we have significantly extended the number of standard game-theoretic concepts for MAIDs {(and hence causal games)} and proved their equivalence to their EFG counterparts, EFGs remain a more well-investigated representation. Thus, if one is interested in more exotic equilibrium refinements, for example, EFGs are likely to be a more suitable model. It is our hope that further research on MAIDs and causal games will reduce this last difference.

\subsection{Future Work}
\label{sec:future}

Our priority is to use causal games to further analyse incentives in multi-agent systems, which has important applications in ensuring that we build AI systems that are safe and fair. As mentioned in Appendix \ref{sec:biif}, existing work has already characterised some of these incentives using CIDs. Therefore, a natural next step is to extend these definitions to multi-agent (and multi-decision) scenarios, though this is by no means trivial. For example, in single agent-settings VoI is always non-negative, whereas in multi-agent settings this need not be the case (if other agents are aware of the information gain) \cite{Ponssard1976}. Further, we might wish to rule certain incentives in or out based on whether or not they occur under all or some policy profiles satisfying a particular equilibrium refinement, or more generally, falling within the set of rational outcomes.

Other specific applications for which causal games may prove fruitful are, for example: designing mechanisms for auctions and other multi-agent systems, or analysing possible interventions on those mechanisms; generalising counterfactual fairness to multi-agent settings; providing artificial agents with the means to more easily provide explanations and reason about qualitative concepts (such as blame and intent or reasoning patterns) that can be defined using causal models of games; and deriving new definitions for similar concepts. We might also hope to extend the framework presented here with: model variations that can more easily capture dynamic settings, fine-grained subjective beliefs, or optimisation; definitions capturing other classic equilibrium refinements such as perfect Bayesian equilibrium \cite{Fudenberg1991a} or sequential equilibrium \cite{Kreps1982a}; and methods of causal discovery for games.

Given that we propose this paper and our accompanying codebase as a robust foundation for reasoning about causality in games, we believe our work presents many other interesting avenues for further research. We hope that the advantages causal games confer based on their generality, explainability, and succinctness (not to mention their compatibility with existing mainstream models) make them an attractive choice for researchers and practitioners alike who are interested in the intersection of causality and game theory.

\section*{Acknowledgements}
This paper is a significantly expanded version of a previous publication \cite{Hammond2021}. We thank Zac Kenton, Jon Richens, Ilya Shpitser, Colin Rowat, Chris van Merwijk, Patrick Forré, David Reber, Joe Halpern, Paul Harrenstein, Will Lee, Vincent Conitzer{, and several anonymous reviewers} for their helpful comments and discussions while completing this work. Hammond was supported by an EPSRC Doctoral Training Partnership studentship (Reference: 2218880), Fox was supported by the EPSRC Centre for Doctoral Training in Autonomous Intelligent Machines and Systems (Reference: EP/S024050/1), and Wooldridge was supported by a UKRI Turing AI World Leading Researcher Fellowship (Reference: EP/W002949/1).

\printbibliography

\appendix
\setcounter{lemma}{1}
\setcounter{proposition}{2}
\setcounter{theorem}{0}
\setcounter{corollary}{0}

\section{Proofs}
\label{app:proofs}

\subsection{Transformations between Game Representations}
\label{app:transformations}

\subsubsection*{MAID to EFG}
\label{app:MAIDtoEFG}

In this section, we provide an encoding transforming a MAID, $\model = (\graph, \bm{\theta})$, into an EFG, $\efg = (N, T, P, A, \lambda, I, U)$. One can decide on whether one wants to keep record of the structure of the original MAID, or if one wants to optimise the encoding by splitting only only some of $\graph$'s variables. If one wants to be able to preserve the full structure of $\model$ then the set of splitting variables in the resulting tree is $\bm{S} = \bm{X} \cup \bm{D}$.\footnote{We assume that $\desc_{\bm{U}} = \varnothing$.} If instead, however, one wishes to minimise the complexity of computing equilibria, then one only needs to split on the MAID's decision variables and their parents, $\bm{S} = \Fa_{\bm{D}}$ \cite{pearl1988probabilistic}. This is because the EFG's tree size will be exponential in $\vert \bm{S} \vert$. The following procedure, which we refer to as \texttt{maid2efg}, is based on that of K\&M \cite{koller2003multi}. Given a MAID $\model = (\graph, \bm{\theta})$ we define an equivalent EFG $\efg = (N, T, P, A, \lambda, I, U)$ as follows:
\begin{itemize}
    \item Choose a topological ordering $S_1 \prec \dots \prec S_n$ over all variables in $\bm{S}$ such that if $S_k$ is a descendent of $S_j$, then $S_j \prec S_k$.\footnote{This topological ordering will, in general, be non-unique.} Further, define $\bm{S}_k^\prec = \{S' \in \bm{S} : S' \prec S_k\}$.
    
    \item The set of agents, $N$, in $\efg$ is the same as in ${\model}$.
    
    \item The tree $T$ is a symmetric tree with each path containing splits over all the variables in $\bm{S}$ in the order defined by $\prec$.
    
    \item For all variables $S \in \bm{S}$, if $S$ is a chance variable $S \in \bm{X}$, add it to $\efg$'s set of chance nodes $\bm{V}^0$. Else, if $S$ is a decision variable $S \in \bm{D}^i$, add $S$ to agent $i$'s set of nodes in $\efg$, $\bm{V}^i$.
    
    \item Label each node $V$ in $T$ with an instantiation $\mu(V)$ corresponding to the values taken by each EFG node (i.e., the branch followed from this node) on the path from the tree's root node $R$ to $V$.
    
    \item For every node $V$ in $T$ such that its corresponding variable $S_V \in \bm{S}$ is a chance variable in ${\model}$, we determine the probability distribution $P_V : \Ch_{V} \rightarrow [0,1]$ over its children $V' \in \Ch_V$ (each child corresponds to a value $s_v \in \dom(S_V)$) by querying that variable's CPD with the instantiation label at $V$. In other words, $P_V(V') \coloneqq \Pr\big(s_v \mid \mu(V)\big)$. This equation can be simplified because in a BN the CPD for each variable $S_V$ only depends on the values of its parents, $\pa_{S_V}$. Therefore, we have:
    $$P_V(V') \coloneqq \Pr\big(s_v \mid \mu(V)\big) = \Pr\big(s_v \mid \mu(V)[\Pa_{S_V}]\big),$$
    where $\mu(V)[\Pa_{S_V}])$ is the restriction of $\mu(V)$ to the values of $\Pa_{S_V}$. If $P_V(V') = 0$ for some child $V'$, we remove that child from the EFG.
    
    \item The set of available actions $A^i_j$ for agent $i$ at node $V^i_j$ in the EFG is given by the domain of the corresponding decision variable $D = S_{V^i_j}$ in the MAID, i.e. $A^i_j \coloneqq \dom(D)$.
    
    \item Define $\lambda: \bm{V} \times \bm{V} \rightarrow A$ for each decision node $V \in \bm{V}^1 \cup \dots \cup \bm{V}^n$ and $V' \in \Ch_V$ such that $\lambda(V,V') = \mu(V')[S_V]$ to label the outcome of the decision.
    
    \item Define an equivalence relation $\sim$ over  $\bm{V}^1 \cup \dots \cup \bm{V}^n$ such that $V \sim V'$ if and only if $S_V = S_{V'} = S$ and $\mu(V)[\Pa_{S_V}] = \mu(V')[\Pa_{S_{V'}}]$. Then the set of information sets $I^i$ for each agent $i \in N$ is simply the quotient set $\bm{V}^i/\sim$ -- the set of $\sim$ equivalence classes partitioning $\bm{V}^i$. In other words, two nodes $V, V' \in \bm{V}$ which correspond to the same decision variable $S$ in $\model$ are in the same information set $I^i_j$ for agent $i$ if and only if their instantiation labels, restricted to their parents in $\model$, are the same.
    
    \item The payoff $U: \bm{L} \rightarrow \mathbb{R}^n$ for each leaf $L\in \bm{L}$ of the EFG's tree is $U(L) = (u^1,\dots,u^n)$, where $u^i \coloneqq \sum_{U \in \bm{U}^i} \expect [ U \mid \mu(L) ]$ is the expected utility of agent $i$ in $\model$, given the instantiation $\mu(L)$ of the variables $\bm{S}$.
\end{itemize}

\subsubsection*{EFG to MAID}
\label{app:EFG2MAID}

We now detail the construction \texttt{efg2maid}. Unlike in the case for \texttt{maid2efg}, we show how every EFG can be converted to a \emph{unique}, canonical equivalent MAID. Given an EFG $\efg = (N, T, P, A, \lambda, I, U)$ (including intervention sets) we define an equivalent MAID $\model = (\graph, \bm{\theta})$ as follows:

\begin{itemize}
    \item The set of agents $N$ remains the same in $\model$ as in $\efg$.
    \item Initialise the MAID's graph $(\bm{V}, \mathscr{E})$ as $T$, where edges are directed from parents to children.
    \item For each of $\efg$'s chance nodes $V \in \bm{V}^0$, label each outgoing edge from $V$ with a value $v$. Every other node with an outgoing edge is labelled, by $\lambda$, with the decision corresponding to that edge. Thus, let $\dom(V)$ in $\model$ contain the labels of the outgoing edges from $V$.
    \item For each variable $V \in \bm{V}$ let $\rho_V$ be the unique path formed by the sequence of labels from the root $R$ of $\efg$ to the corresponding EFG node for $V$ and let $\rho_V[V']$ denote the label of the outgoing edge from node $V'$ on the path $\rho_V$.
    \item For each information set $I^i_j$ in $\efg$, let $\rho^\prec(I^i_j)$ be the set of paths from $R$ into the nodes of $I^i_j$.
    Next, define a function $\mu : \bigcup_{i \in N} I^i \rightarrow 2^{\bm{V}}$ that maps each information set $I^i_j$ to the set of variables in $\Anc_{I^i_j}$ whose outgoing label is the same in every path in $\rho^\prec(I^i_j)$. Note that by Definition \ref{def:intervention_set}, if information sets $I$ and $I'$ are in the same intervention set, then the nodes whose values are in $\mu(I)$ are the same as the nodes whose values are in $\mu(I')$.
    \item We then consider $\efg$'s chance, decision, and leaf nodes in turn:
    \begin{itemize}
        \item For every outgoing edge with some label $v_j$ from a chance node $V \in \bm{V}^0$, add the label $(v_j \mid \rho_{V_j}) : p$ where $p = P_j(v_j)$.
        \item For every information set $I^i_j$ and each variable $V \in I^i_j$, add the label $(v \mid \rho_V[\mu(I^i_j)] ) : \_$ where $\rho_V[\mu(I^i_j)]$ denotes the labels of $\rho_D$ restricted to those in $\mu(I^i_j)$ and $\_$ is a placeholder to be parameterised by a decision rule.
        \item For each leaf variable $L \in \bm{L}$ with payoff vector $U(L) = (u^1,\ldots,u^n)$, split $L$ into $n$ utility variables $U^1,\dots,U^n$ (duplicating incoming edges) with labels $(u^i \mid \rho_L) : 1$ respectively.
    \end{itemize}
    \item Given these labellings we proceed by merging variables according to the intervention sets:
    \begin{itemize}
         \item Merge each information set $I^i_j$ into a single variable $D_j \in \bm{D}^i$, collecting the labels $(v \mid \rho_V[\mu(I^i_j)] ) : \_$ for each $V \in I^i_j$ and retaining all incoming and outgoing edges.
        \item Begin by adding a directed edge from every $V' \in \Anc_V$ to $V$ for every variable $V$. Then for every variable $D_j$ corresponding to an information set $I^i_j$, remove all the incoming edges from variables that are not in $\mu(I^i_j)$.
        \item Merge each group of variables, collecting their incoming and outgoing labelled edges, that belong to the same intervention set.
        \item Merge any utility variables $U^i_j$ and $U^i_k$ belonging to the same agent $i$ that have the same sets of incoming edges, and collect their labels.
    \end{itemize}
    \item We then let $\bm{V} \coloneqq \bm{X} \cup \bm{D} \cup \bm{U}$ where $\bm{X} = \bm{V}^0$, $\bm{D}$ is the union of variable sets $\bm{D}^i$ defined above, and $\bm{U}$ is the collection of utility variables, also defined above.
    \item $\mathscr{E}$ is the set of edges in the graph defined above.
    \item We conclude by defining the CPDs $\Pr(V \mid \Pa_V)$, for every $V \in \bm{X} \cup \bm{U}$. 
    \begin{itemize}
        \item For certain instantiations, $\pa_V$, the CPD over $V$ is undefined because $\pa_V$ does not represent a path through the original game tree $T$ (as mentioned in Section \ref{sec:equivalences}). 
        For non-decision variables, this is dealt with by simply adding a null value $\perp$ for every variable in $\bm{X}$ and a value of 0 for every variable in $\bm{U}$.
        \item Recall the labels of the form $(v \mid \rho_V) : p$ and let $l(V)$ denote the set of $V$'s labels. 
        \begin{itemize}
            \item For each variable $V \in \bm{V} \setminus \bm{D}$, we define $\Pr(v \mid \pa_V) \coloneqq \sum_{(v \mid \pa_V) : p ~\in~ l(V)} p$.
            \item For any $\pa_V$ such that for all $v \in \dom(V)$, $\Pr(v\mid\pa_V) = 0$, we set $\Pr(\perp \mid \pa_V) = 1$ if $V \in \bm{X}$ and $\Pr( 0\mid \pa_V) = 1$ if $V \in \bm{U}$.
        \end{itemize}
    \end{itemize}
    \item By construction, $\prod_{V \in \bm{V} \setminus \bm{D}} \Pr(v \mid \pa_V)$ therefore forms a partial distribution over the non-decision variables in $\model$.
    \item For decision variables, we now see that the only changes in parameterisations of $\pi_D(D \mid \Pa_D)$ that can have any effect on any of the other variables are those that occur under settings $\pa_D$ such that there exists a strategy $\sigma$ and path $\rho$ in $\efg$ capturing all values in $\pa_D$ with $\Pr^\sigma(\rho) > 0$. In other words, those values of $\pi_D(d \mid \pa_D)$ corresponding to a parametrisation of $(d \mid \rho_D[\mu(I^i_j)] ) : \_$.
\end{itemize}

\subsection{Theoretical Results}
\label{app:theory}

\subsubsection*{Section 3}

\begin{proposition}
    $\mecvar_{V}$ is $\R^{\BR}$-relevant to $\Pi_{D}$ if and only if $\mecvar_{V} \not\perp_{\meczero{\graph}} \bm{U}^i \cap \Desc_{D} \mid D, \Pa_{D}$ or $\mecvar_{V} \not\perp_{\meczero{\graph}} \Pa_D$, where if $D \in \bm{D}^i$, then $\bm{U}^i \cap \Desc_{D} \neq \varnothing$.
\end{proposition}

\begin{proof}
First, recall from Definition \ref{def:relevance}, that $\mecvar_{V} \in \Pa_{\Pi_D}$ is $\R^{\BR}$-\textbf{relevant} to $\Pi_D$ if there exists $\pa_{\Pi_D} \neq \pa_{\Pi_D}'$ such that $r^{\BR}_D(\pa_{\Pi_D}) \neq r^{\BR}_D(\pa_{\Pi_D}')$, where $\pa_{\Pi_D}$ and $\pa_{\Pi_D}'$ differ only on $\mecvar_{V}$. Moreover, recall that for each $D \in \bm{D}^i$, $\pi_{D} \in r^{\BR}_{D}(\pa_{\Pi_{D}})$ if and only if: 
$$\pi_{D} \in \argmax_{\hat{\pi}_{D} \in \dom(\Pi_{D})} \sum_{U \in \bm{U}^i} \expect_{(\hat{\pi}_{D}, \bm{\pi}_{-D})} [ U ].$$ 
Given some setting $\mecvals_{-D}$ of all mechanism variables apart from $\Pi_D$, then for each $U \in \bm{U}^i$ and any $\hat{\pi}_{D} \in \dom(\Pi_D)$ we have that:
\begin{align*}
    \Pr^{(\hat{\pi}_{D}, \bm{\pi}_{-D})}(u)
    &= \sum_{d \in \dom(D)} \sum_{\pa_{D} \in \dom(\Pa_{D})} \Pr^{(\hat{\pi}_{D}, \bm{\pi}_{-D})}(u, d, \pa_{D})\\
    &= \sum_{d \in \dom(D)} \sum_{\pa_{D} \in \dom(\Pa_{D})} \Pr^{\bm{\pi}_{-D}}(u \mid d, \pa_{D}) \hat{\pi}_D(d \mid \pa_{D}) \Pr^{\bm{\pi}_{-D}}(\pa_{D})\\
    &= \sum_{d \in \dom(D)} \sum_{\pa_{D} \in \dom(\Pa_{D})} \Pr^{\bm{\pi}}(u \mid d, \pa_{D}) \hat{\pi}_{D}(d \mid \pa_{D}) \Pr^{\bm{\pi}}(\pa_{D})
\end{align*}
The first equality simply rewrites the marginal distribution as a sum over a joint distribution; the second factorises this joint distribution; and the third results from the observations that the distribution over $\Pa_{D}$ cannot depend on the CPD $\pi_{D}(D \mid \Pa_{D})$, and nor can the distribution over $U$ \emph{given} some $d, \pa_{D}$. Next, note that in the definition of $r^{\BR}$ we quantify over $\hat{\pi}_{D}$ via the $\argmax$ operation, and that the choice of $\hat{\pi}_{D}$ (given $\bm{\pi}_{-D}$) can only impact the expected value of those utility variables $U \in \bm{U}^i$ that are also in $\Desc_{D}$. Therefore, any decision rule that is in $r^{\BR}_{D}(\pa_{\Pi_{D}})$ satisfies:
\begin{equation}
\label{eq:argmax}
    \argmax_{\hat{\pi}_{D} \in \dom(\Pi_{D})}\bigg[ \sum_{U \in \bm{U}^i \cap \Desc_{D}}  \sum_{d \in \dom(D)} \sum_{\pa_{D} \in \dom(\Pa_{D})} \Pr^{\bm{\pi}}(u \mid d, \pa_{D}) \hat{\pi}_{D}(d \mid \pa_{D}) \Pr^{\bm{\pi}}(\pa_{D}) \cdot  u \bigg].
\end{equation}
Therefore, if the expression (\ref{eq:argmax}) is independent of the value $\mecval_{V}$, then $\mecvar_{V}$ is not $\R^{\BR}$-\textbf{relevant} to $\Pi_D$. As explained in Section \ref{sec:relevance}, this is equivalent to asking whether $V$ is a requisite probability node for any distribution in $\mathcal{Q}^\BR_D = \{ \Pr^{{\bm{\pi}}}(\bm{u}^i \cap \desc_{D} \mid d, \pa_{D}), \Pr^{\bm{\pi}}(\pa_{D}) \}$. The graphical criterion for checking whether something is a requisite probability node is given in Lemma \ref{lem:reachability}, and results in the criteria (which we refer to as $\R^{\BR}$-reachability) $\mecvar_{V} \not\perp_{\meczero{\graph}} \bm{U}^i \cap \Desc_{D} \mid D, \Pa_{D}$ or $\mecvar_{V} \not\perp_{\meczero{\graph}} \Pa_D$, where if $D \in \bm{D}^i$, then $\bm{U}^i \cap \Desc_{D} \neq \varnothing$. 

We begin by proving soundness, i.e., that if $\mecvar_{V}$ is not $\R^{\BR}$-reachable from $\Pi_D$, then $\mecvar_{V}$ is not $\R^{\BR}$-relevant to $\Pi_{D}$. Let $\mecvals \neq \mecvals'$ differ only on the value of $\mecvar_V \in \Pa_{\Pi_D}$, and let $\pa_{\Pi_D} \neq \pa_{\Pi_D}'$ be the respective settings of $\Pa_{\Pi_D}$. Additionally, let us write $\mecvals = (\bm{\pi}, \bm{\theta})$ and $\mecvals' = (\bm{\pi}', \bm{\theta}')$. If $\mecvar_{V}$ is not $\R^{\BR}$-reachable from $\Pi_D$ then we have that both $\mecvar_{V} \perp_{\meczero{\graph}} \bm{U}^i \cap \Desc_{D} \mid D, \Pa_{D}$ and $\mecvar_{V} \perp_{\meczero{\graph}} \Pa_D$, implying that both $\Pr^{\bm{\pi}}(u \mid d, \pa_{D}; \bm{\theta}) = \Pr^{\bm{\pi}'}(u \mid d, \pa_{D} ; \bm{\theta}')$ and $\Pr^{\bm{\pi}}(\pa_{D}; \bm{\theta}) = \Pr^{\bm{\pi}'}(\pa_{D} ; \bm{\theta}')$. Thus, the value of expression (\ref{eq:argmax}) is independent of the value of $\mecvar_{V}$ and so $r^{\BR}_D(\pa_{\Pi_D}) = r^{\BR}_D(\pa_{\Pi_D}')$, meaning that $\mecvar_{V}$ is not $\R^{\BR}$-relevant to $\Pi_{D}$, as required.

We now turn to completeness, i.e., that if $\mecvar_{V}$ is $\R^{\BR}$-reachable from $\Pi_D$ in some mechanised graph $\mec{\graph}$, where $D \in \bm{D}^i$ and $\bm{U}^i \cap \Desc_{D} \neq \varnothing$, then for some mechanised MAID $\mec{\model} = (\mec{\graph}, \bm{\theta}, \R^{\BR})$, $\mecvar_{V}$ is $\R^{\BR}$-relevant to $\Pi_{D}$. Our goal is therefore to find some parameterisation of the mechanised graph $\mec{\graph}$ such that if the parameterisation of $\Pi_{D}$'s parents $\pa_{\Pi_D}$ is changed to $\pa_{\Pi_D}'$, where these only differ on $\mecvar_V$, then $r_D(\pa_{\Pi_D}) \neq r_D(\pa_{\Pi_D}')$.

If $\mecvar_{V} \not\perp_{\meczero{\graph}} \bm{U}^i \cap \Desc_{D} \mid D, \Pa_{D}$, then we can simply follow K\&M's proof of completeness for $s$-reachability to construct a parameterisation of the mechanised graph such that $\mecvar_{V}$ is $s$-relevant, and hence $\R^{\BR}$-relevant, to $\Pi_{D}$ in the mechanised MAID. If, instead, we only have $\mecvar_{V} \not\perp_{\meczero{\graph}} \Pa_{D}$, then we proceed as follows. Below, we restrict our attention to those decision nodes that satisfy $\bm{U}^i \cap \Desc_{D} \neq \varnothing$ for $D \in \bm{D}^i$. This is because if $\bm{U}^i \cap \Desc_{D} = \varnothing$, then agent $i$ will be indifferent between all of its decision rules for $D$ and so trivially we will have $r^{\BR}_D(\pa_{\Pi_D}) = r^{\BR}_D(\pa_{\Pi_D}')$ for any $\pa_{\Pi_D}$ and $\pa_{\Pi_D}'$; in this case, $\mecvar_{V}$ is never $\R^{\BR}$-relevant to $\Pi_{D}$.

Let us assume that $\mecvar_{V} \not\perp_{\meczero{\graph}} Z$ for some $Z \in \Pa_{D}$. Then, as $\mecvar_{V}$ has no parents in $\meczero{\graph}$ and all paths between $\mecvar_{V}$ and $Z$ that contain colliders are blocked, there must be a directed path $\mecvar_{V} \pathto Z$ in $\meczero{\graph}$. Let us denote the variables along this directed path by $X_0 \pathto X_{n+1}$, where $X_0 = \mecvar_V$ and $X_{n+1} = Z$. Next, following our remarks above, we know that $\bm{U}^i \cap \Desc_{D} \neq \varnothing$, and so there is a directed path $D \pathto U$ for some $U \in \bm{U}^i$. We denote the variables along this path by $Y_0 \pathto Y_{m+1}$, where $Y_0 = D$ and $Y_{m+1} = U$. Thus, in $\meczero{\graph}$ we have a directed path
$X_0 \pathto X_{n+1} \to Y_0 \pathto Y_{m+1}$
where the segments $X_1 \pathto X_{n}$ and $Y_1 \pathto Y_{m}$ might be empty.

We now give two parameterisations of $\mec{\graph}$ that are identical except for the two settings $\pa_{\Pi_D} \neq \pa_{\Pi_D}'$ that differ only on $\mecvar_V$. We begin with what the two parameterisations have in common. First, we let all non-utility variables be binary, although we will show how this assumption can easily be relaxed at the end. Next, we set the CPDs for all nodes along the paths $X_2 \pathto X_{n+1}$ and $Y_1 \pathto Y_{m}$ to copy the value of their parent in this path. That is, for $i \in \{2,n\}$, $x_{i} = x_{i-1}$ with probability 1, and for $i \in \{1,m\}$, $y_{i} = y_{i-1}$ with probability 1. Note that the value of $X_1 = V$ is determined according to the value of $\mecvar_V$. Finally, we set the distributions of the utility functions so that $U$ copies the value of $Y_m$ and any other $U' \in \bm{U}^i$ takes value 0 regardless of the value of $\Pa_{U'}$.

For $\pa_{\Pi_D}$ we let $\mecval_V$ set the distribution over $V$ to $\delta(V {,} 1)$, implying that $X_{n+1} = Z$ takes value 1. Given some decision rule $\pi_D \in r^{\BR}_D(\pa_{\Pi_D})$, note that the decision rule $\bar{\pi}_D \in r^{\BR}_D(\pa_{\Pi_D})$ too, where $\bar{\pi}_D$ is identical to $\pi_D$ apart from the fact that $\bar{\pi}_D(D=0 \mid Z=0, \pa'_D) = 1$ where $\pa'_D$ is any setting of the variables $\Pa'_D = \Pa_D \setminus \{Z\}$. When $D = 0$ then, by the construction above, $Y_{m + 1} = U = 0$, but as $Z=0$ with probability zero given $\mecval_V$, then this outcome never occurs, and so $\bar{\pi}_D$ remains a rational response to $\pa_{\Pi_D}$.

For $\pa_{\Pi_D}'$, we let $\mecval'_V$ set the distribution over $V$ to be such that $\Pr^{\bm{\pi}}(V=0 \mid \pa_V) > 0$ for any $\pa_V$. In this case, $\bar{\pi}_D$ is \emph{not} a rational response, as the context $Z=0$ occurs with non-zero probability, and thus so to does $U=0$, due to the construction above and the fact that $\bar{\pi}_D(D=0 \mid Z=0, \pa'_D) = 1$. {Agent} $i$'s expected utility could instead be strictly improved by selecting the decision rule $\delta(D {,} 1)$. Thus, $\bar{\pi}_D \notin r^{\BR}_D(\pa'_{\Pi_D})$, and hence $r^{\BR}_D(\pa_{\Pi_D}) \neq r^{\BR}_D(\pa'_{\Pi_D})$, i.e., $\mecvar_{V}$ is $\R^{\BR}$-relevant to $\Pi_{D}$, as required.
To extend the above proof to the case where variables have arbitrary (as opposed to binary) domains, we simply label one value in the domain of each variable by `1' and another by `0', making sure that for the utility variables the value labelled as `1' is strictly greater than the value labelled as `0'.
\end{proof}

\subsubsection*{Section 4}

\begin{proposition}
    For distributions over $\exovar_D$ and $D$ as governed by equations (\ref{eq:dist_def}), there is a one-to-one correspondence between the set of stochastic decision rules $\Delta(\dom(D) \mid \dom(\overline{\Pa}_D))$ and the set of deterministic decision rules $\dom(\dot{\Pi}_D) \subset \Delta(\dom(D) \mid \dom(\Pa_D))$. Moreover, given two such corresponding decision rules $\pi_D$ and $\dot{\pi}_D$, then $\int_{\dom(\exovar_D)} \dot{\pi}_D(d \mid \overline{\pa}_D, \exoval_D)\Pr(\exoval_D) \,d \exoval_D = \pi_D(d \mid \overline{\pa}_D)$.
\end{proposition}
\begin{proof}
    Recall from Section \ref{sec:counterfactuals} that we define $\overline{\Pa}_D \coloneqq \Pa_D \setminus \{\exovar_D\}$. Let us first examine the correspondence between the stochastic and non-stochastic decision rules. Let $f : \Delta(\dom(D) \mid \dom(\overline{\Pa}_D)) \to \dom(\dot{\Pi}_D)$ be a function where for $\pi_D \in \Delta(\dom(D) \mid \dom(\overline{\Pa}_D))$ we define $f(\pi_D) = \dot{\pi}_D$ such that $\dot{\pi}_D({d} \mid \overline{\pa}_D, \exoval_D) = \delta(d {,} \exoval^{\pi_D, \overline{\pa}_D}_D)$. Given our definition of $\exovar_D$ then for any $\pi_D \neq \pi'_D$ then we have that $f(\pi_D) = f(\pi'_D)$ if and only if $\exovar^{\pi_D, \overline{\pa}_D}_D = \exovar^{\pi'_D, \overline{\pa}_D}_D$ for all $\overline{\pa}_D$. Clearly, however, for $\pi_D \neq \pi'_D$ then there exists some $\overline{\pa}_D$ and $d \in \dom(\exovar^{\pi_D, \overline{\pa}_D}_D) = \dom(\exovar^{\pi'_D, \overline{\pa}_D}_D)$ such that:
    $$\Pr(\exovar^{\pi_D, \overline{\pa}_D}_D = d) \coloneqq \pi_D(d \mid \overline{\pa}_D) \neq \pi'_D(d \mid \overline{\pa}_D) \eqqcolon \Pr(\exovar^{\pi'_D, \overline{\pa}_D}_D = d),$$
    and so $f(\pi_D) \neq f(\pi'_D)$, hence $f$ is an injection. Moreover, any $\dot{\pi}_D$ of this form can be identified with the set $\{ \exovar^{\pi_D, \overline{\pa}_D}_D \}_{\overline{\pa}_D \in \dom(\overline{\Pa}_D)} \subseteq \exovar_D$ and hence some $\pi_D \in \Delta(\dom(D) \mid \dom(\overline{\Pa}_D))$ such that $f(\pi_D) = \dot{\pi}_D$, hence $f$ is a surjection, and therefore a bijection, giving us our one-to-one correspondence. 
    
    Now take some $\pi_D \in \Delta(\dom(D) \mid \dom(\overline{\Pa}_D))$ and $\dot{\pi}_D = f(\pi_D)$. Let $\exovar^{- \pi_D, \overline{\pa}_D}_D \coloneqq \exovar_D \setminus \{ \exovar^{\pi_D, \overline{\pa}_D}_D \}$. Then, by our definition of $\exovar_D$ and its joint product of probability distributions \cite{Bauer1996}, we have that:
    \begin{align*}
        &\int_{\dom(\exovar_D)} \dot{\pi}_D(d \mid \overline{\pa}_D, \exoval_D)\Pr(\exoval_D) \,d \exoval_D\\
        = &\int_{\dom(\exovar_D)} \delta(d {,} \exoval^{\pi_D, \overline{\pa}_D}_D) \Pr(\exoval_D) \,d \exoval_D\\
        = &\sum_{\exoval^{\pi_D, \overline{\pa}_D}_D \in \dom(\exovar^{\pi_D, \overline{\pa}_D}_D)} \delta(d {,} \exoval^{\pi_D, \overline{\pa}_D}_D) \Pr(\exoval^{\pi_D, \overline{\pa}_D}_D) \cdot \int_{\dom(\exovar^{- \pi_D, \overline{\pa}_D}_D)} \Pr(\exoval^{- \pi_D, \overline{\pa}_D}_D) \,d \exoval^{- \pi_D, \overline{\pa}_D}_D\\
        = &\sum_{\exoval^{\pi_D, \overline{\pa}_D}_D \in \dom(\exovar^{\pi_D, \overline{\pa}_D}_D)} \delta(d {,} \exoval^{\pi_D, \overline{\pa}_D}_D) \pi_D(\exoval^{\pi_D, \overline{\pa}_D}_D \mid \overline{\pa}_D) \cdot 1\\
        = &\pi_D(d \mid \overline{\pa}_D).
    \end{align*}
\end{proof}

\subsubsection*{Section 5}
\begin{proposition}
    {A (behavioural policy) NE is not guaranteed to exist in a MAID.}
\end{proposition}
\begin{proof}
    {See Appendix \ref{app:non-existence_proofs} for an example of a MAID that does not have a (behavioural policy) NE.}
\end{proof}
\begin{lemma}
    Let ${\bm{\pi}}^{-i} \in \Delta(\dom(\bm{D}^{-i}) \mid \dom(\Pa_{\bm{D}^{-i}}))$ be a partial (behavioural or mixed) policy profile for agents $N \setminus \{i\}$ in a MAID $\model$. If agent $i$ has perfect recall in $\model$, then for every mixed policy $\mu^i$ there exists a behavioural policy ${\bm{\pi}}^i$ such that $\Pr^{(\mu^i, {\bm{\pi}}^{-i})}(\bm{v}) = \Pr^{({\bm{\pi}}^i, {\bm{\pi}}^{-i})}(\bm{v})$.
\end{lemma}
\begin{proof}
    For the purposes of this proof, let us abuse notation by viewing each pure decision rule $\dot{\pi}_D$ as a function, so that $\dot{\pi}_D(\pa_D) = d$ just in case $\dot{\pi}_D(d \mid \pa_D) = 1$. Further, for some pure policy $\dot{{\bm{\pi}}}^i \in \bigtimes_{D\in\bm{D}^i}\dot{\Pi}_D$, let us write $\dot{{\bm{\pi}}}^i(\pa_{\bm{D}^i}) = \bm{d}^i$ just in case $\dot{\pi}_D(\pa_D) = d$ for every $D \in \bm{D}^i$. Then, given a fixed MAID $\model$ and a partial (behavioural or mixed) policy profile ${\bm{\pi}}^{-i}$, we have the partial distribution $\Pr^{{\bm{\pi}}^{-i}}(\bm{v}, \bm{d}^{-i} : \bm{d}^i) = \prod_{V \in \bm{V} \setminus \bm{D}^i} \Pr^{{\bm{\pi}}^{-i}}(v \mid \pa_V)$. It therefore suffices to show that for any $\mu^i \in \Delta(\bigtimes_{D\in\bm{D}^i} \dom(\dot{\Pi}_D))$, there exists some ${\bm{\pi}}^i \in \bigtimes_{D\in\bm{D}^i}\Pi_D$ such that:
    \begin{align*}
        \mu^i(\bm{d}^i \mid \pa_{\bm{D}^i}) 
        \coloneqq \sum_{ \dot{{\bm{\pi}}}^i \in \bigtimes_{D \in\bm{D}^i} \dot{\Pi}_D ~:~ \dot{{\bm{\pi}}}^i(\pa_{\bm{D}^i}) = \bm{d}^i} \mu^i(\dot{{\bm{\pi}}}^i)
        = \prod_{D \in \bm{D}^i} \pi_D(d \mid \pa_{D})
        \eqqcolon {\bm{\pi}}^i(\bm{d}^i \mid \pa_{\bm{D}^i}),
    \end{align*}
    where again we abuse notation somewhat by using $\mu^i$ and ${\bm{\pi}}^i$ to denote both policies and the resulting partial distributions over variables in $\model$. In what follows, we also suppress the qualification that $\dot{{\bm{\pi}}}^i \in \bigtimes_{D \in\bm{D}^i} \dot{\Pi}_D$, which we take to be implicit when considering $\dot{{\bm{\pi}}}^i$. Suppose that $\vert \bm{D}^i \vert = m$
    and let us assume, without loss of generality, that the indices $1,\dots,m$ of these decision rules reflect the unique (in virtue of the fact that each agent has perfect recall) topological ordering of variables $\bm{D}^i$ in $\model$; i.e., $D_1 \prec \dots \prec D_{m}$. In the remainder of the proof, we write $\bm{D}^i_{\leq j}$ to denote the set of variables $\{D_1, \ldots, D_j\} \subseteq \bm{D}^i$. Then for each $D_j \in \bm{D}^i$ we have that $\Fa_{\bm{D}_{< j}} \subseteq \Pa_{D_j}$. We now define each decision rule $\pi^i_{D_j}$ as follows: 
    $$\pi^i_{D_j}(d_j \mid \pa_{D_j})  \coloneqq
    \begin{cases}
        \frac{1}{\vert \dom(D_j) \vert} & \text{ if } Z_j = 0\\
        \frac{1}{Z_j} \sum_{\dot{{\bm{\pi}}}^i_{\bm{D}_{\leq j}} ~:~  \dot{{\bm{\pi}}}^i_{\bm{D}_{\leq j}}(\pa_{\bm{D}^i_{\leq j}}) = \bm{d}^i_{\leq j}} \mu^i(\dot{{\bm{\pi}}}_{\bm{D}_{\leq j}}) & \text{ otherwise},
    \end{cases}
    $$
    where:
    $$Z_j  \coloneqq
    \begin{cases}
        1 & \text{ if } j = 1\\
        \sum_{\dot{{\bm{\pi}}}^i_{\bm{D}_{< j}} ~:~  \dot{{\bm{\pi}}}^i_{\bm{D}_{< j}}(\pa_{\bm{D}^i_{< j}}) = \bm{d}^i_{< j}} \mu^i(\dot{{\bm{\pi}}}_{\bm{D}_{< j}}) & \text{ otherwise}.
    \end{cases}
    $$
    Note that this behavioural decision rule is well-formed. Concretely, $\pi^i_{D_j}(d_j \mid \pa_{D_j}) \geq 0$ as $\mu^i(\dot{{\bm{\pi}}}_{\bm{D}_{\leq j}}) \geq 0$ and thus $Z_j > 0$ whenever used to normalise $\pi^i_{D_j}(d_j \mid \pa_{D_j})$, and $\sum_{d_j} \pi^i_{D_j}(d_j \mid \pa_{D_j}) = 1$ as:
    $$\sum_{d_j} \Bigg( \sum_{\dot{{\bm{\pi}}}^i_{\bm{D}_{\leq j}} ~:~  \dot{{\bm{\pi}}}^i_{\bm{D}_{\leq j}}(\pa_{\bm{D}^i_{\leq j}}) = \bm{d}^i_{\leq j}} \mu^i(\dot{{\bm{\pi}}}_{\bm{D}_{\leq j}}) \Bigg) = \sum_{\dot{{\bm{\pi}}}^i_{\bm{D}_{< j}} ~:~  \dot{{\bm{\pi}}}^i_{\bm{D}_{< j}}(\pa_{\bm{D}^i_{< j}}) = \bm{d}^i_{< j}} \mu^i(\dot{{\bm{\pi}}}_{\bm{D}_{< j}}).$$
    Using this definition, we may now conclude the proof, proceeding by cases. First, suppose that $Z_j = 0$ for some $j \geq 2$. Then:
    \begin{align*}
        \sum_{\dot{{\bm{\pi}}}^i_{\bm{D}_{< j}} ~:~  \dot{{\bm{\pi}}}^i_{\bm{D}_{< j}}(\pa_{\bm{D}^i_{< j}}) = \bm{d}^i_{< j}} \mu^i(\dot{{\bm{\pi}}}_{\bm{D}_{< j}}) = 0
        \Rightarrow \pi_{D_{j-1}}(d_{j-1} \mid \pa_{D_{j-1}}) = 0
        \Rightarrow {\bm{\pi}}^i(\bm{d}^i \mid \pa_{\bm{D}^i}) = 0.
    \end{align*}
    But the first antecedent above also implies that $\mu^i(\bm{d}^i \mid \pa_{\bm{D}^i}) \coloneqq \sum_{ \dot{{\bm{\pi}}}^i ~:~ \dot{{\bm{\pi}}}^i(\pa_{\bm{D}^i}) = \bm{d}^i} \mu^i(\dot{{\bm{\pi}}}^i) = 0$, and so we have ${\bm{\pi}}^i(\bm{d}^i \mid \pa_{\bm{D}^i}) = \mu^i(\bm{d}^i \mid \pa_{\bm{D}^i})$, as required. Now suppose that there is no $j$ such that $Z_j = 0$. Then, by telescoping terms we have the following:
    \begin{align*}
        {\bm{\pi}}^i(\bm{d}^i \mid \pa_{\bm{D}^i})
        \coloneqq &\prod_{j = 1}^{m} \pi_{D_j}(d_j \mid \pa_{D_j})\\
        = &\prod_{j = 1}^{m} \frac{1}{Z_j} \sum_{\dot{{\bm{\pi}}}^i_{\bm{D}_{\leq j}} ~:~  \dot{{\bm{\pi}}}^i_{\bm{D}_{\leq j}}(\pa_{\bm{D}^i_{\leq j}}) = \bm{d}^i_{\leq j}} \mu^i(\dot{{\bm{\pi}}}_{\bm{D}_{\leq j}})\\ 
        = &\sum_{\dot{{\bm{\pi}}}^i_{D_{\leq m}} ~:~  \dot{{\bm{\pi}}}^i_{D_{\leq m}}(\pa_{\bm{D}^i_{\leq m}}) = \bm{d}^i_{\leq m}} \mu^i(\dot{{\bm{\pi}}}_{D_{\leq m}})\\
        \eqqcolon &\mu^i(\bm{d}^i \mid \pa_{\bm{D}^i}).
    \end{align*}
\end{proof}
\begin{proposition}
    In any MAID $\model$ with perfect recall, there exists a {(behavioural)} policy profile ${\bm{\pi}}$ that is an NE.
\end{proposition}
\begin{proof}
    First, note that an immediate consequence of Lemma \ref{lem:prequiv} is that, for a given partial policy ${\bm{\pi}}^{-i}$ (behavioural or mixed) in a MAID $\model$ with perfect recall, then under any mixed policy $\mu^i$ and equivalent behavioural policy ${\bm{\pi}}^i$ we have:
    $$\sum_{U \in \bm{U}^k} \expect_{(\mu^i, {\bm{\pi}}^{-i})} [ U ] = \sum_{U \in \bm{U}^k} \expect_{({\bm{\pi}}^i, {\bm{\pi}}^{-i})} [ U ],$$
    for all agents $k \in N$.
    Next, by a straightforward application of Nash's theorem \cite{nash1950equilibrium}, we know that every MAID is guaranteed to have an NE $\mu$ in mixed policies. Given $\mu = (\mu^1,\ldots,\mu^n)$, we can therefore apply Lemma \ref{lem:prequiv} $n$ times to result in a behavioural policy ${\bm{\pi}} = ({\bm{\pi}}^1,\ldots,{\bm{\pi}}^n)$ such that $\sum_{U \in \bm{U}^k} \expect_{\mu} [ U ] = \sum_{U \in \bm{U}^k} \expect_{{\bm{\pi}}} [ U ]$ for all agents $k \in N$.
    Finally, let us assume (for a contradiction) that there is some agent $i \in N$ that may deviate from ${\bm{\pi}}^i$ in order to increase their expected utility. Let this deviation be denoted by $\hat{{\bm{\pi}}}^i$. Then we can construct a mixed policy $\hat{\mu}^i \in \Delta(\bigtimes_{D\in\bm{D}^i} \dom(\dot{\Pi}_D))$ that results in the same joint distribution as $\hat{{\bm{\pi}}}^i$ in the following manner. We define $\hat{\mu}^i(\dot{{\bm{\pi}}}^i) \coloneqq \prod_{j=1}^{m} \hat{\mu}^i (\dot{\pi}^i_{D_j})$ where we set:
    $$\hat{\mu}^i (\dot{\pi}^i_{D_j}) \coloneqq \prod_{\pa'_{D_j} \in \dom(\Pa_{D_j})} \hat{\pi}^i_j \big( \dot{\pi}^i_{D_j}(\pa'_{D_j}) \mid \pa'_{D_j} \big).$$
    Again, note that this results in a well-formed mixed policy. Concretely, $\hat{\mu}^i(\dot{{\bm{\pi}}}^i) \geq 0$ for any $\dot{{\bm{\pi}}}^i$ as $\hat{\pi}^i_j \big( \dot{\pi}^i_{D_j}(\pa'_{D_j}) \mid \pa'_{D_j} \big) \geq 0$, and we have:
    \begin{align*}
        \sum_{\dot{{\bm{\pi}}}^i} \hat{\mu}^i(\dot{{\bm{\pi}}}^i) 
        = &\sum_{\dot{\pi}^i_{D_1}} \ldots \sum_{\dot{\pi}^i_{D_m}} \prod_{j=1}^{m} \hat{\mu}^i (\dot{\pi}^i_{D_j})\\
        = &\prod_{j=1}^{m} \sum_{\dot{\pi}^i_{D_j}} \hat{\mu}^i (\dot{\pi}^i_{D_j})\\
        = &\prod_{j=1}^{m} \sum_{\dot{\pi}^i_{D_j}} \Bigg( \prod_{\pa'_{D_j} \in \dom(\Pa_{D_j})} \hat{\pi}^i_j \big( \dot{\pi}^i_{D_j}(\pa'_{D_j}) \mid \pa'_{D_j} \big) \Bigg)\\
        = &\prod_{j=1}^{m} \Bigg( \prod_{\pa'_{D_j} \in \dom(\Pa_{D_j})} \Bigg( \sum_{d'_j \in \dom(D_j)} \Bigg( \sum_{\dot{\pi}^i_{D_j} ~:~ \dot{\pi}^i_{D_j}(\pa'_{D_j}) = d'_j} \hat{\pi}^i_j \big( \dot{\pi}^i_{D_j}(\pa'_{D_j}) \mid \pa'_{D_j} \big) \Bigg) \Bigg) \Bigg)\\
        = &\prod_{j=1}^{m} \Bigg(  \prod_{\pa'_{D_j} \in \dom(\Pa_{D_j})} 1 \Bigg)\\
        = &1.
    \end{align*}
    In order to show that policy profiles $(\hat{\mu}^i, {\bm{\pi}}^{-i})$ and $(\hat{{\bm{\pi}}}^i, {\bm{\pi}}^{-i})$ result in the same joint distribution over all variables (and hence the same expected payoffs for each agent), it suffices to show the following, which holds for any $j \in \{1,\dots,m \}$:
    \begin{align*}
        \hat{\mu}^i_j(d_j \mid \pa_{D_j})
        &\coloneqq \sum_{\dot{\pi}^i_{D_j} ~:~ \dot{\pi}^i_{D_j}(\pa_{D_j}) = d_j} \hat{\mu}^i (\dot{\pi}^i_{D_j})\\
        &= \sum_{\dot{\pi}^i_{D_j} ~:~ \dot{\pi}^i_{D_j}(\pa_{D_j}) = d_j} \Bigg( \prod_{\pa'_{D_j} \in \dom(\Pa_{D_j})} \hat{\pi}^i_j \big( \dot{\pi}^i_{D_j}(\pa'_{D_j}) \mid \pa'_{D_j} \big) \Bigg)\\
        &= \sum_{\dot{\pi}^i_{D_j} ~:~ \dot{\pi}^i_{D_j}(\pa_{D_j}) = d_j} \hat{\pi}^i_j(d_j \mid \pa_{D_j}) \Bigg( \prod_{\pa'_{D_j} \in \dom(\Pa_{D_j}) \setminus \{ \pa_{D_j} \}} \hat{\pi}^i_j \big( \dot{\pi}^i_{D_j}(\pa'_{D_j}) \mid \pa'_{D_j} \big) \Bigg)\\
        &= \hat{\pi}^i_j(d_j \mid \pa_{D_j}) \cdot \prod_{\pa'_{D_j} \in \dom(\Pa_{D_j}) \setminus \{ \pa_{D_j} \}} \Bigg( \sum_{\dot{\pi}^i_{D_j} ~:~ \dot{\pi}^i_{D_j}(\pa_{D_j}) = d_j} \hat{\pi}^i_j \big( \dot{\pi}^i_{D_j}(\pa'_{D_j}) \mid \pa'_{D_j} \big) \Bigg)\\
        &= \hat{\pi}^i_j(d_j \mid \pa_{D_j}) \cdot \prod_{\pa'_{D_j} \in \dom(\Pa_{D_j}) \setminus \{ \pa_{D_j} \}} 1\\
        &= \hat{\pi}^i_j(d_j \mid \pa_{D_j}).
    \end{align*}
    Thus, if $\hat{{\bm{\pi}}}^i$ is a beneficial deviation against ${\bm{\pi}}^{-i}$, $\hat{\mu}^i$ is also a beneficial deviation against ${\bm{\pi}}^{-i}$. Further, by again appealing to Lemma \ref{lem:prequiv} we can see that:
    $$\sum_{U \in \bm{U}^i} \expect_{(\hat{\mu}^i, {\bm{\pi}}^{-i})} [ U ] = \sum_{U \in \bm{U}^i} \expect_{(\hat{\mu}^i, \mu^{-i})} [ U ],$$
    and so $\hat{\mu}^i$ is also a beneficial deviation against $\mu^{-i}$. In which case, $\mu$ cannot be an NE, giving us our contradiction.
\end{proof}

\begin{proposition}
    If an agent $i$ in a MAID $\model$ has perfect recall, then they also have sufficient recall. However, if an agent has sufficient recall, then they do not always have perfect recall.
\end{proposition}
\begin{proof}
    For the first direction, we recognise that if agent $i$ has perfect recall in $\model$, then this means there exists a unique topological ordering of variables $\bm{D}^i$ in $\model$; i.e., $D_1 \prec \dots \prec D_{m}$. Furthermore, for any $j<k$ we have that $\Fa_{D_{j}} \subseteq \Pa_{D_k}$. This implies that: $$\Pi_{D_j} \perp_{\meczero{\graph}} \bm{U}^i \cap \Desc_{D_k} \mid D_k, \Pa_{D_k},$$
    and so $\Pi_{D_j}$ is not $s$-reachable from $\Pi_{D_k}$. Therefore, there is no edge $\Pi_{D_j} \rightarrow \Pi_{D_k}$ in the $s$-relevance graph $\mathsf{r}_s{\graph}$. This implies that $\mathsf{r}_s{\graph}$, when restricted to just agent $i$'s decision rule variables, must be acyclic (i.e., there can \emph{only} exist edges $\Pi_{D_j} \rightarrow \Pi_{D_k}$ when $j > k$), as required. For the second direction, we provide {an example} in Appendix \ref{app:non-existence_proofs} of a MAID where an agent has sufficient but imperfect recall.
\end{proof}

\begin{proposition}
    Any MAID $\model$ with sufficient recall has at least one SPE in behavioural policies.
\end{proposition}
\begin{proof}
    {Let $\graph_1 \prec \dots \prec \graph_m$ be an ordering of MAID $\model$'s $s$-subdiagrams such that $\graph_j \prec \graph_k$ implies that $\graph_k$ is \emph{not} an $s$-subdiagram of $\graph_j$. If all agents have sufficient recall, then $\graph_1$ contains at most one decision variable for each agent, and for each $s$-subdiagram $\graph_j$ where $1 \leq j < m$, $\graph_{j+1}$ contains \emph{at most} one additional decision variable for each agent. The newly added decision nodes in each subsequent $s$-subdiagram for such an ordering correspond exactly with a topological ordering of the strong connected components in $\model$'s $s$-relevance graph. This means that we can employ the same method as in a similar proof of Theorem 6.2 in K\&M \cite{koller2003multi} to show that if we use a backwards induction procedure to optimise the decision rules of each subgame in this order, then the resulting policy profile is guaranteed to be an NE in every feasible $s$-subgame of $\model$, and hence an SPE.}
\end{proof}

\begin{proposition}
    Suppose that a MAID $\model$ has sufficient recall. Then, the set of SPEs of $\model$ is equal to the set of rational outcomes $\R^\SP(\mec\model)$.
\end{proposition}
\begin{proof}
    As in the previous proof, let $\graph_1 \prec \dots \prec \graph_m$ be an ordering of $\model$'s $s$-subdiagrams such that $\graph_j \prec \graph_k$ implies that $\graph_k$ is \emph{not} an $s$-subdiagram of $\graph_j$. If all agents have sufficient recall, then $\graph_1$ contains at most one decision variable for each agent, and for each $s$-subdiagram $\graph_j$ where $1 \leq j < m$, $\graph_{j+1}$ contains \emph{at most} one additional decision variable for each agent. The fact that $\graph_1$ contains at most one decision variable for each agent means that $\bm{\pi}_1$ is an NE in any (feasible) $s$-subgame $\model_1$ over $\graph_1$ just in case:
    $$\pi_{D} \in \argmax_{\hat{\pi}_D \in \dom(\Pi_{D})} \sum_{U \in \bm{U}^i \cap \bm{V}_1} \expect_{(\hat{\pi}_{D}, \bm{\pi}_{1,-D})} [ U ],$$
    for each $\model_1$, i.e., just in case $\pi_{D} \in r^{\SP}_{D}(\pa_{\Pi_{D}})$. This is akin to the result that the $\R^{\BR}$-rational outcomes of a game where each agent has only one decision variable are simply the NEs of the game.

    Similarly, if we reason analogously to the backwards induction procedure over $s$-subgames as described in the proof of Proposition \ref{prop:SPEexist} and assume that $\bm{\pi}_j$ forms an NE in every feasible $s$-subgame of every $\model_j$, then by fixing $\bm{\pi}_j$ in each (feasible) $s$-subgame $\model_{j+1}$ we induce a new MAID where each agent has at most one decision variable. A partial policy profile $\bm{\pi}'_{j+1}$ is an NE in this game if and only if:
    $$\pi_{D} \in \argmax_{\hat{\pi}_D \in \dom(\Pi_{D})} \sum_{U \in \bm{U}^i \cap \bm{V}_{j+1}} \expect_{(\hat{\pi}_{D}, \bm{\pi}'_{j,-D})} [ U ],$$ 
    for each $D \in \bm{D}_{j+1} \setminus \bigcup_{1\leq k \leq j} \bm{D}_{k}$. Moreover, by our inductive hypothesis the resulting policy profile $\bm{\pi}_{j+1} = (\bm{\pi}'_{j+1}, \bm{\pi}_j)$ forms an NE (and in fact an SPE) in each (feasible) $\model_{j+1}$, and so we have that each such $\pi_{D} \in r^{\SP}_{D}(\pa_{\Pi_{D}})$.
    At the conclusion of this line of argument we see that any (full) policy profile $\bm{\pi} = \bm{\pi}_m$ is an NE in every feasible $s$-subgame of $\model$, when restricted to that subgame, just in case each $\pi_{D} \in r^{\SP}_{D}(\pa_{\Pi_{D}})$, i.e., just in case $\bm{\pi} \in \R^\SP(\mec\model)$.
\end{proof}

\subsubsection*{Section 6}

\begin{lemma}
    {Let $f$ be a natural mapping between $\efg$ and $\model$. Then $\sigma$ is an NE in $\efg$ if and only if every ${\bm{\pi}} \in f(\sigma)$ is an NE in $\model$.}
\end{lemma}
\begin{proof}
    {This result follows straightforwardly from Definition \ref{def:equiv}. Suppose that $\sigma$ is an NE. Then if $f$ is a natural mapping, we have that $\sum_{U \in \bm{U}^i} \expect_{{\bm{\pi}}} [ U ] = \expect_{\sigma} \big[ U(\rho[\bm{L}])[i] \big]$ for every ${\bm{\pi}} \in f(\sigma)$ and every agent $i$. If $\hat{\bm{\pi}}^i$ is a profitable deviation for player $i$ in $\model$ then there must exist some $f^i(\hat{\sigma}^i) \ni {\bm{\pi}^i}$ such that $\hat{\sigma}^i$ is a profitable deviation for player $i$ in $\efg$. I.e., we must have: 
    $$\sum_{U \in \bm{U}^i} \expect_{{\bm{\pi}}} [ U ]
    <
    \sum_{U \in \bm{U}^i} \expect_{(\bm{\pi}^{-i}, \hat{\bm{\pi}}^i)} [ U ]
    =
    \expect_{(\sigma^{-i},\hat{\sigma}^i)} \big[ U(\rho[\bm{L}])[i] \big]
    >
    \expect_{\sigma} \big[ U(\rho[\bm{L}])[i] \big],
    $$
    which contradicts our assumption that $\sigma$ is an NE. The same argument also applies in the opposite direction, thus concluding the proof.}
\end{proof}
\begin{lemma}
    If $\efg\in \mathtt{maid2efg} (\model)$ or
    $\model =\mathtt{efg2maid}(\efg)$, then $\efg$ and $\model$ are equivalent.
\end{lemma}
\begin{proof}
    In what follows we make the trivial assumption that any non-leaf node $V$ in $\efg$ has more than one child (otherwise we could simply remove such nodes), and that if $V$ is a chance node then all of its children are reached with positive probability (otherwise we could delete the subtrees rooted at such children). The proof follows directly by construction using the procedures \texttt{maid2efg} and \texttt{efg2maid} respectively. 
    
    To see this, first suppose we have a MAID, $\model$, and $\efg$ is an EFG resulting from $\text{\texttt{maid2efg}}(\model)$. A decision rule $\pi_D$ defines a probability distribution over $\dom(D)$ conditional on each decision context, $\pa_D$. Following the \texttt{maid2efg} procedure, each feasible decision context is an instantiation of a set of variables which corresponds to one information set $I_j^i$ where $S_V = D$ for all $V \in I_j^i$, and for each $d \in \dom(D)$ there exists precisely one node $V' \in \Ch_V$ such that $\lambda(V,V') = d$. In particular, if $\pa_D$ is infeasible then the probability of reaching $I_j^i$ is zero and it would be removed from $\efg$ (as per our assumption above). 
    
    For a policy profile ${\bm{\pi}}$, let us define $\sigma$ such that $\sigma^i_j(d) \coloneqq \pi_D(d\mid\pa_D)$ for each feasible decision context $\pa_D$. Note that for a given $\pa_D$ then this construction results in a one-to-one correspondence between $\sigma^i_j(D), \pi_D( D \mid\pa_D) \in \Delta(\dom(D))$, and that for any two ${\bm{\pi}} \sim {\bm{\pi}}'$ (i.e., ${\bm{\pi}}$ and ${\bm{\pi}}'$ differ only on infeasible decision contexts) then the same strategy $\sigma$ will be constructed. This construction therefore results in a bijection $f: \Sigma \rightarrow \dom(\bm{\Pi}) / \sim$. By reasoning analogously about the chance and utility variables, and observing that the difference between chosen actions in infeasible decision contexts has no bearing on the expected utility of each agent in $\model$ (because they all occur with probability zero) it follows that the expected utility for each agent $i$, $\expect_{\sigma} \big[ U(\rho[\bm{L}])[i] \big] = \sum_{U \in \bm{U}^i} \expect_{{\bm{\pi}}} [ U ]$ for any $\sigma$ such that $\bm{\pi} \in f(\sigma)$, hence $f$ is a natural mapping and $\model$ and $\efg$ are equivalent.
    
    For the second part of the proof, note first that the deterministic nature of the $\text{\texttt{efg2maid}}$ procedure guarantees uniqueness of the resulting MAID, and so suppose we have some $\model = \text{\texttt{efg2maid}}(\efg)$. In our construction above, the incoming edges for each decision variable $D \in \bm{D}^i$ are precisely those that originate from the variables whose values agent $i$ can determine when in the corresponding information set(s) $I^i_j$. Hence the strategy $\sigma^i_j$, which assigns a probability distribution over the set of available decisions $A^i_j$ at node $V^i_j$, is determined as a function of $\pa_D$, where $\pa_D$ corresponds to a particular information set $I^i_j$ from which $D$ was created. 
    
    Thus, given a strategy profile $\sigma$ in $\efg$ we fix the corresponding policy profiles ${\bm{\pi}}$ in $\model$ to have $\pi_D(d\mid\pa_D) \coloneqq \sigma^i_j(d)$ when $\pa_D$ is feasible, and let $\pi_D(d\mid\pa_D)$ vary otherwise. As before, we therefore have a bijection $f: \Sigma \rightarrow \dom(\bm{\Pi}) / \sim$. The same form of correspondence (i.e., ignoring the infeasible settings of the parents) can analogously be seen to hold between the distributions $P$ of $\efg$ and the CPDs for $\bm{X}$ and $\bm{U}$. Thus, for any ${\bm{\pi}} \in f(\sigma)$, we have that $\sum_{U \in \bm{U}^i} \expect_{{\bm{\pi}}} [ U ] = \expect_{\sigma} \big[ U(\rho[\bm{L}])[i] \big]$ for each agent $i \in N$, and so $f$ is a natural mapping, as required.
\end{proof}
\begin{definition}
    \label{def:EFGNE}
    A strategy profile $\sigma$ is a \textbf{Nash equilibrium (NE)} in an EFG if, for every agent $i \in N$, $\sigma^i \in \argmax_{\hat{\sigma}^i \in \Sigma^i} \expect_{(\hat{\sigma}^i, \sigma^{-i})} \big[ U(\rho[\bm{L}])[i] \big]$.
\end{definition} 
\begin{corollary}
    If $\efg\in\mathtt{maid2efg}(\model)$ or $\model = \mathtt{efg2maid}(\efg)$, then there is a natural mapping $f$ between $\efg$ and $\model$ such that $\sigma$ is an NE in $\efg$ if and only if every ${\bm{\pi}} \in f(\sigma)$ is an NE in $\model$.
\end{corollary}
\begin{proof}
    By Lemma \ref{lem:correspondence} we have that $\efg$ and $\model$ are equivalent. Thus, let $f : \Sigma \rightarrow \dom({\bm{\Pi}}) / \sim$ be a natural mapping  between $\efg$ and $\model$ as defined in Definition \ref{def:equiv}. {By applying Lemma \ref{lem:NEpreservation}, we have that $\sigma$ is an NE in $\efg$ if and only if every ${\bm{\pi}} \in f(\sigma)$ is an NE in $\model$.}
\end{proof}
\begin{proposition}
    If $\efg\in\mathtt{maid2efg}(\model)$ or $\model = \mathtt{efg2maid}(\efg)$, then there is a natural mapping $f$ between $\efg$ and $\model$ such that, for every subgame $\efg'$ in $\efg$ there is an $s$-subgame $\model'$ in $\model$ that is equivalent (modulo a constant difference between the utilities for each agent under any policy in $\model'$) to $\efg'$ under the natural mapping $f$ restricted to the strategies of $\efg'$.
\end{proposition}
\begin{proof}
    We begin by proving existence. Recall that a subgame $\efg'$ in an EFG $\efg$ is a subtree that is closed under information sets and descendants. Let $\bm{V}' \subseteq \bm{V}$ be the set of variables in $\model$ corresponding to the intervention sets overlapping with or contained in $\efg'$. We first show that $\bm{V}'$ forms an $s$-subdiagram, i.e., that $\bm{V}'$ contains every variable $Z$ in $\model$ whose mechanism node $\mecvar_Z$ is $s$-reachable from $\Pi_{D}$ for some $D \in \bm{V}'$, and for every $X,Y \in \bm{V}'$, $\bm{V}'$ contains every variable that lies on a directed path $X \pathto Y$ in $\graph$. 
    
    Beginning with the first condition, it suffices to show that there exists no variable $Z \in \bm{Z} = \bm{V} \setminus \bm{V}'$ such that $\mecvar_Z$ is $s$-reachable from the decision rule $\Pi_{D}$ of some variable $D \in \bm{D}^i \cap \bm{V}'$ for $i \in N'$. Recall that this means we would have $\mecvar_{Z} \not\perp_{\meczero{\graph}} \bm{U}^i \cap \Desc_{D} \mid \Fa_{D}$, where $\meczero{\graph}$ is the mechanised graph with no edges between mechanism variables. Any path supporting such a dependency must have one of the following two forms:
    \begin{itemize}
        \item $\mecvar_Z \rightarrow Z \rightarrow \cdots~ \bm{U}^i \cap \Desc_{D}$. In this case, as $Z \notin \bm{V'}$ then, by either construction \texttt{efg2maid} or \texttt{maid2efg}, any node in $\efg$ corresponding to $Z$ must lie outside $\efg'$.
        Further, as $\efg'$ is closed under information sets, then the value of $Z$ must be observed by any decision node in $\efg'$ corresponding to $D$, and thus there is an information link $Z \rightarrow D$ in $\model$ and so $Z \in \Fa_{D}$. Hence, by conditioning on $\Fa_{D}$ we block the path above, meaning there is no dependency.
        \item $\mecvar_Z \rightarrow Z \leftarrow \cdots~ \bm{U}^i \cap \Desc_{D}$. In this case, let $W$ be the first variable in the path (from left to right) that is a fork variable, i.e., we have $\cdots \leftarrow W \rightarrow \cdots$. Such a variable must exist because we assume that utility variables do not have children. 
        As before, notice that by either construction \texttt{efg2maid} or \texttt{maid2efg} we have that $Z \notin \bm{V'}$ must lie outside $\efg'$, and as there is a directed path $W \pathto Z$ in $\model$ then $W$ must also lie outside $\efg'$. But then the fact that $\efg'$ is closed under information sets means that $W$ is observed by $D$ and so conditioning on $\Fa_{D}$ blocks the path above, meaning there is no dependency.
    \end{itemize}
    We next consider the second condition. Suppose that $X,Y\in \bm{V}'$ and there is a directed path $X \pathto Z \pathto Y$ in $\model$. Then if $\efg \in \texttt{maid2efg}(\model)$ or $\model = \texttt{efg2maid}(\efg)$, any topological ordering $\prec$ over the variables in $\model$ must have $X \prec Z \prec Y$ and hence if there are nodes in $\efg'$ corresponding to both $X$ and $Y$ then as $\efg'$ is closed under descendants we must have a node corresponding to $Z$ in $\efg'$. This means that the intervention set containing $Z$ overlaps with $\efg'$ and so $Z \in \bm{V}'$ by the definition of $\bm{V}'$. 

    For the second part of the existence proof, we show that there is a setting $\bm{z}$ of $\bm{Z}$ which, when combined with $\bm{V'}$, leads to an $s$-subgame $\model'$ of $\model$ that is equivalent to $\efg'$. First, note that any node passed through on the path $\rho_{R'}$ from the root $R$ of $\efg$ to the root $R'$ of $\efg'$ must correspond to a variable in $\bm{Z}$. Let $\bm{z}$ be any setting of $\bm{Z}$ that is consistent with the path from $R$ to $R'$, and let $\model'$ be the resulting $s$-subgame that is obtained by combining $\bm{V'}$ with $\bm{z}$. Observe that by the argument above, the decision variables in $\bm{D'} \subseteq \bm{V'}$ are precisely those corresponding to the information sets in $\efg'$, and moreover that any decision context of any $D \in \bm{D'}$ that does not correspond to some information set in $\efg'$ is, by definition, not feasible. 
    
    Thus, the natural mapping $f$ between $\efg$ and $\model$, when restricted to $\efg'$ and $\model'$, leads to a bijection between the strategies in $\efg'$  and the quotient set of behavioural policy profiles in $\model'$ by $\sim$ (where two policies are in the same equivalence class if and only if they differ only on those decision contexts that are infeasible), as per the arguments in the proof of Lemma \ref{lem:correspondence}. Furthermore, for any equivalent strategy $\bm{\pi} \in f(\sigma)$, it follows directly by construction using \texttt{maid2efg} or \texttt{efg2maid} that if $\expect'_{\sigma} \big[ U(\rho[\bm{L}])[i] \big]$ and $\sum_{U' \in \bm{U}^i \cap \bm{V}'} \expect'_{\bm{\pi}} [ U']$ denote the expected utility of agent $i$ when $\sigma$ and $\bm{\pi}$ are restricted to $\efg'$ and $\model'$ respectively, then:
    \begin{align*}
        \expect'_{\sigma} \big[ U(\rho[\bm{L}'])[i] \big]
        &\coloneqq \sum_\rho {P'}^\sigma(\rho) U(\rho[\bm{L}])[i]\\
        &= \sum_\rho {P}^\sigma(\rho \mid \rho_{R'}) U(\rho[\bm{L}])[i]\\
        &= \sum_{U \in \bm{U}^i}\sum_{u \in \dom(U)} \!\! u \cdot {\Pr}^{\bm{\pi}}(u \mid \bm{z}; \theta)\\
        &= \sum_{U \in \bm{U}^i \cap \bm{Z}} u + \sum_{U' \in \bm{U}^i \cap \bm{V}'}\sum_{u' \in \dom(U')} \!\! u' \cdot {\Pr}^{\bm{\pi}}(u' \mid \bm{z}; \theta)\\
        &= \sum_{U \in \bm{U}^i \cap \bm{Z}} u + \sum_{U' \in \bm{U}^i \cap \bm{V}'}\sum_{u' \in \dom(U')} \!\! u' \cdot {\Pr}^{\bm{\pi}}(u' ; \theta')
        \eqqcolon \sum_{U \in \bm{U}^i \cap \bm{Z}} u + \sum_{U' \in \bm{U}^i \cap \bm{V}'} \expect'_{\bm{\pi}} [ U' ]
    \end{align*}
    for every agent $i \in N'$, where ${P}^\sigma(\rho \mid \rho_{R'})$ is the distribution over paths $\rho$ in $\efg$ conditional on $\rho$ containing $\rho_{R'}$ as a sub-path. Note that because some utility variables $U \in \bm{U}^i$ may not occur in $\bm{V}'$ then we must add their value $u$ according to $\bm{z}$ to the payoff for each agent in order to equate the expected utilities for each agent in both games. Given $\bm{z}$, however, this difference between the utilities for each agent is constant under any policy in $\model'$, and so has no effect on the optimality of policies in $\model'$. Modulo this small difference, $f$ forms a natural mapping between $\efg'$ and $\model'$, as required.
\end{proof}
\begin{definition}
    A strategy profile $\sigma$ in an EFG $\efg$ is a \textbf{subgame perfect equilibrium (SPE)} if $\sigma$ is an NE in every subgame of $\efg$, when restricted to that subgame.
\end{definition}
\begin{corollary}
    If $\efg\in\mathtt{maid2efg}(\model)$ or $\model = \mathtt{efg2maid}(\efg)$, then there is a natural mapping $f$ between $\efg$ and $\model$ such that if every ${\bm{\pi}} \in f(\sigma)$ is an SPE in $\model$, then $\sigma$ is an SPE in $\efg$.
\end{corollary}
\begin{proof}
    The corollary can be seen to follow immediately from combining the definitions of SPEs in EFGs and MAIDs with Proposition \ref{prop:subgames} and {Lemma \ref{lem:NEpreservation}}. Concretely, if $\efg \in \texttt{maid2efg}(\model)$ or $\model = \texttt{efg2maid}(\efg)$, and if any ${\bm{\pi}} \in f(\sigma)$ is an SPE in $\model$ (for some natural mapping $f$ between $\efg$ and $\model$ satisfying Proposition \ref{prop:subgames}) then ${\bm{\pi}}$ is an NE in any $s$-subgame of $\model$, and as every subgame of $\efg$ is equivalent to an $s$-subgame of $\model$ under $f$ (modulo a constant difference between the utilities for each agent under any policy in $\model'$, which does not affect whether a given policy profile in $\model'$ is an NE or not), then we must have that $\sigma$ is an NE in each subgame of $\efg$ (by {Lemma \ref{lem:NEpreservation}}).
\end{proof}
\begin{definition}
    A perturbation vector $\eta_k$ contains perturbations $\epsilon^{i,j}_a \in (0,1)$ with $\sum_{a \in A^i_j} \epsilon^{i,j}_a \leq 1$ for every information set $I^i_j$ such that in a perturbed game $\efg(\eta_k)$, each strategy is forced to have $\sigma^i_j(a) \geq \epsilon^{i,j}_a$.
    A strategy profile $\sigma$ is a \textbf{trembling hand perfect equilibrium (THPE)} in an EFG $\efg$ if there is a sequence of perturbation vectors $\{\eta_k\}_{k\in\mathbb{N}}$ such that $\lim_{k \rightarrow \infty}\Vert\eta_k\Vert_\infty = 0$ and for each perturbed EFG $\efg(\eta_k)$ there is an NE $\sigma_k$ such that $\lim_{k \rightarrow \infty} \sigma_k = \sigma$.
\end{definition}
\begin{proposition}
    If $\efg\in\mathtt{maid2efg}(\model)$ or $\model = \mathtt{efg2maid}(\efg)$, then there is a natural mapping $f$ between $\efg$ and $\model$ such that $\sigma$ is a THPE in $\efg$ if and only if every ${\bm{\pi}} \in f(\sigma)$ is a THPE in $\model$.
\end{proposition}
\begin{proof}
    Given Lemma \ref{lem:correspondence} and {Lemma \ref{lem:NEpreservation}}, it suffices to show that the perturbed MAID resulting from $\text{\texttt{efg2maid}}(\efg(\eta_k))$ is equivalent to $\efg(\eta_k)$, and similarly that any perturbed EFG resulting from $\text{\texttt{maid2efg}}(\model(\zeta_k))$ is equivalent to $\model(\zeta_k)$. For the first part, given $\eta_k$ we define the entries of $\zeta_k$ for each decision $d$ and decision context $\pa_D$ as:
    $$\epsilon^{\pa_D}_d \coloneqq
    \begin{cases}
        \epsilon^{i,j}_a & \text{ if $d = a$ and $\pa_D$ is feasible}\\
        \min_{i,j,a} \epsilon^{i,j}_a & \text{ otherwise,}
    \end{cases}$$
    where $\pa_D$ corresponds to the information set $I^i_j$ as per our $\texttt{efg2maid}$ construction. Clearly for any $\bm{\pi} \in f(\sigma)$ where $f$ is a natural mapping between $\efg$ and $\model$ then $\sigma^i_j(a) \geq \epsilon^{i,j}_a$ if and only if $\pi_D(d \mid \pa_D) \geq \epsilon^{\pa_D}_d$. Based on this, it can easily be seen that any such $f$ is also a natural mapping between the perturbed games $\efg(\eta_k)$ and $\model(\zeta_k)$, and hence they are equivalent.
    For the second part, given $\zeta_k$ we construct $\eta_k$ such that $\epsilon^{i,j}_a \coloneqq \epsilon^{\pa_D}_d$ if $a = d$ and $\pa_D$ is feasible. A similar argument to the above shows that any natural mapping $f$ between $\model$ and $\efg$ induces a natural mapping between the perturbed games $\model(\zeta_k)$ and $\efg(\eta_k)$, as required.
\end{proof}

\section{Further Examples}
\label{app:examples}

\subsection{Counterfactuals Using the Closest Possible World Principle}
\label{app:closest_possible_world}

When computing counterfactuals in SCGs under the `closest possible world' principle, our basic desideratum is to consider those counterfactual rational outcomes ${\bm{\pi}}' \in \R(\mec{\model}_{\I})$ that are consistent with the decision rules used in some actual rational outcome $\bm{\pi} \in \R(\mec{\model} \mid \bm{z})$ whenever those decision rules are \emph{not} affected by $\I$. In other words, we wish to keep the set of invariant decision rules ${\bm{\Pi}}(\I)$ as large as possible while propagating changes due to $\I$.

As a first attempt, we might set ${\bm{\Pi}}(\I)$ to ${\bm{\Pi}} \setminus (\bm{Y} \cup \Desc_{\bm{Y}})$. It can easily be seen, however, that this choice is flawed. Consider an $\R$-relevance graph with two variables $\Pi_D, \Pi_{D'}$, with $\Pa_{\Pi_D'} = \{\Pi_{D}\}$, to which we apply an intervention $\Do(\Pi_D = \pi_D)$. Then it is perfectly possible that $r_{D'}(\pi_D) = \bigcup_{\hat{\pi}_D \in r_D()} r_{D'}(\hat{\pi}_D)$ -- i.e., the set of rational responses for decision rule $\Pi_D'$ does not change -- and in which case there is no sense in which $\Pi_D'$ is affected by $\I$, even though $\Pi_D' \in \{\Pi_D\} \cup \Desc_{\Pi_D}$. 

Instead, we can compute ${\bm{\Pi}}(\I)$ by propagating the effects of $\I$ through the \emph{maximal strongly connected components} (MSCCs) of the $\R$-relevance graph when restricted to decision rule variables.\footnote{Recall that a strongly connected component (SCC) is a subgraph containing a directed path between every pair of nodes and a maximal SCC is an SCC that is not a strict subset of any other SCC.} Let $\conr\graph$\label{def:con-graph} denote the condensation of the $\R$-relevance graph $\relr\graph$ when restricted to the decision rule variables -- called the `component graph' by K\&M \cite{koller2003multi} -- with topological ordering $C_1 \prec \dots \prec C_m$\label{topo} over the vertices in $\conr \graph$ (where each $C_j \subseteq {\bm{\Pi}}$ is an MSCC of $\relr\graph$ restricted to the decision rule variables), and let us write:
$$r_{C_j}(\pa_{C_j}) \coloneqq \big\{ c_j : \pi_{D} \in r_D(\pa_{\Pi_D}) ~\forall~ {\Pi}_{D} \in C_j \big\},$$ to denote the rational responses $c_j = (\pi_D)_{\Pi_D \in C_j}$. Then we may compute ${\bm{\Pi}}(\I)$ using Algorithm \ref{algo:Pi_I}.

\begin{algorithm}[h]
    \caption{}
    \label{algo:Pi_I}
    \begin{algorithmic}[1]
        \Statex \textbf{Input:} $\mec{\model}$, $\bm{z}$, $\I$
        \Statex \textbf{Output:} ${\bm{\Pi}}(\I)$
        \State $\Pr(\bm{\theta}) \gets \Pr(\bm{\theta}_{\I})$
        \State ${\bm{\Pi}}(\I) \gets {\bm{\Pi}} \setminus (\bm{Y} \cup \Desc_{\bm{Y}})$
        \State form $\conr \graph$ from $\relr \graph$ with topological ordering $C_1 \prec \dots \prec C_m$
        \For{$j = 1, \dots, m$}
            \If{$C_j \subseteq {\bm{\Pi}}(\I)$}{ continue}
            \Else{}
                \State $\texttt{actual} \gets \R(\mec{\model} \mid \bm{z})$
                \State $\texttt{counterfactual} \gets \bigcup_{{\bm{\pi}} \in \texttt{actual} } \big\{ {\bm{\pi}}' \in \R(\mec{\model}_{\I}) : {\bm{\pi}}(\I) = {\bm{\pi}}'(\I) \big\}$
                \State $\bm{\Theta}_j \gets \Pa_{C_j} \setminus {\bm{\Pi}}$
                \State ${\bm{\Pi}}_j \gets \Pa_{C_j} \cap {\bm{\Pi}}$
                \If{$\bigcup_{{\bm{\pi}} \in \texttt{actual}} r_{C_j}(\bm{\theta}_j, {\bm{\pi}}_j) = \bigcup_{{\bm{\pi}'} \in \texttt{counterfactual}} r_{C_j}(\bm{\theta}_j, {\bm{\pi}'}_j) $}{ ${\bm{\Pi}}(\I) \gets {\bm{\Pi}}(\I) \cup C_j$}
                \EndIf
            \EndIf
        \EndFor
        \State \textbf{return} ${\bm{\Pi}}(\I)$
    \end{algorithmic}
\end{algorithm}

Algorithm \ref{algo:Pi_I} functions by incrementally expanding the set ${\bm{\Pi}}(\I)$ -- which is at first initialised to ${\bm{\Pi}} \setminus (\bm{Y} \cup \Desc_{\bm{Y}})$ -- by computing the rational responses for each MSCC to both the actual and counterfactual rational outcomes, adding the variables in the component to ${\bm{\Pi}}(\I)$ if and only if the set of such responses remains the same. It can therefore be seen as generalising interventions to models containing both cycles and non-determinism, while maintaining a maximally justifiable similarity to an observed outcome. In other words, we maintain our updated beliefs (due to observation $\bm{z}$) about the values of a decision rule variable $\Pi_D$ if and only if the set of all values $\pi_D$ takes in any rational outcome $\bm{\pi}$ remains the same after the intervention $\I$ is performed, i.e., if and only if $\Pi_D \in {\bm{\Pi}}(\I)$. 
We may then use this set ${\bm{\Pi}}(\I)$ in Definition \ref{def:counterfactual_query} to define counterfactual queries in games under this principle.

\subsection{Non-Existence Results}
\label{app:non-existence_proofs}

\subsubsection*{No Nash Equilibrium in Behavioural Policies}

Proposition \ref{prop:NEexistance} in Section \ref{sec:NE} says that in a MAID with sufficient recall, there must exist an NE in behavioural policies. The example presented here, adapted from \cite{Wichardt2008}, demonstrates that if even one agent in a MAID does \emph{not} have sufficient recall, then although there will be an NE in mixed policies, the same may not hold true for behavioural policies.

\begin{figure}[h]
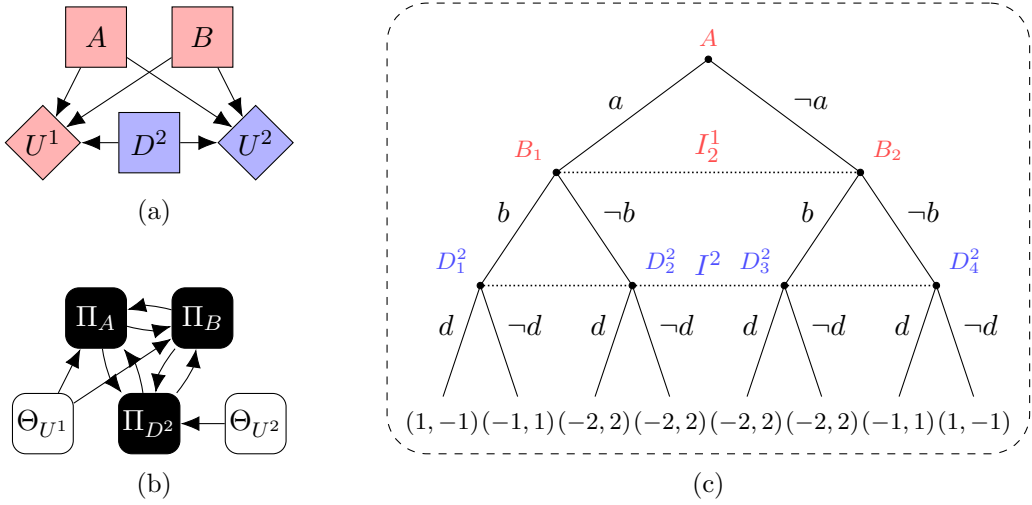

    \centering
    \begin{subfigure}[b]{0.3\linewidth}
        \vspace{0pt}
        \centering
        \begin{influence-diagram}
            \node (A) [decision, player1] {$A$};
            \node (B) [decision, right = of A, player1] {$B$};
            \node (D) [decision, below right = 1.4cm and 0.7cm of A, player2] {$D^2$};
            \node (U1) [utility, left = of D, player1] {$U^1$};
            \node (U2) [utility, right = of D, player2] {$U^2$};
            \edge {A, B, D} {U1};
            \edge {A, B, D} {U2};

        \end{influence-diagram}
        \caption{}
        \label{fig:noeq:a}
        \vspace{0.75cm}
        \begin{influence-diagram}
            \node (A_mec) [relevanceb] {$\Pi_{A}$};
            \node (B_mec) [relevanceb, right = of A_mec] {$\Pi_{B}$};
            \node (D_mec) [relevanceb, below right = 1.4cm and 0.7cm of A_mec] {$\Pi_{D^2}$};
            \node (U1_mec) [relevancew, left = of D_mec] {$\Theta_{U^1}$};
            \node (U2_mec) [relevancew, right = of D_mec] {$\Theta_{U^2}$};
            \path (A_mec) edge[->, bend right=15] (B_mec);
            \path (A_mec) edge[->, bend right=15] (D_mec);
            \path (B_mec) edge[->, bend right=15] (A_mec);
            \path (B_mec) edge[->, bend right=15] (D_mec);
            \path (D_mec) edge[->, bend right=15] (A_mec);
            \path (D_mec) edge[->, bend right=15] (B_mec);
            \edge {U2_mec} {D_mec};
            \edge {U1_mec} {A_mec, B_mec};
        \end{influence-diagram}
        \caption{}
        \label{fig:noeq:b}
    \end{subfigure}
    \begin{subfigure}[b]{0.6\linewidth}
        \vspace{0pt}
        \centering
        \begin{istgame}[]
            \xtdistance{15mm}{40mm}
            \istroot(0)<90, red!70>{$A$}
            \istb{a}[al]
            \istb{\neg a}[ar] 
            \endist
            \xtdistance{15mm}{20mm}
            \istroot(1)(0-1)<135, red!70>{$B_1$}
            \istb{b}[al]
            \istb{\neg b}[ar] 
            \endist
            \istroot(2)(0-2)<45, red!70>{$B_2$}
            \istb{b}[al]
            \istb{\neg b}[ar] 
            \endist
            \xtdistance{15mm}{10mm}
            \istroot(3)(1-1)<135, blue!70>{$D^2_1$}
            \istb{d}[al]{(1,-1)}
            \istb{\neg d}[ar]{(-1,1)} 
            \endist
            \istroot(4)(1-2)<45, blue!70>{$D^2_2$}
            \istb{d}[al]{(-2,2)}
            \istb{\neg d}[ar]{(-2,2)} 
            \endist
            \istroot(5)(2-1)<135,blue!70>{$D^2_3$}
            \istb{d}[al]{(-2,2)}
            \istb{\neg d}[ar]{(-2,2)} 
            \endist
            \istroot(6)(2-2)<45,blue!70>{$D^2_4$}
            \istb{d}[al]{(-1,1)}
            \istb{\neg d}[ar]{(1,-1)} 
            \endist
            \xtInfoset(3)(4)
            \xtInfoset(4)(5){\textcolor{blue!70}{$I^2$}}[above]
            \xtInfoset(5)(6)
            \xtInfoset(1)(2){\textcolor{red!70}{$I^1_2$}}[above]
            \xtSubgameBox(0){(0)(1)(3-1)(3-2)(4-1)(4-2)(2)(5-1)(5-2)(6-1)(6-2)}[black,inner sep = 20pt, xshift=0pt, yshift=0pt]
        \end{istgame}
        \caption{}
        \label{fig:noeq:c}
    \end{subfigure}
    \caption{(a) A MAID representing the {NE non-existence example} and (b) its $s$-relevance graph. (c) An equivalent EFG representing the same game.
    }
    \label{fig:noeq}    
\end{figure}

The {example} is shown as both a MAID $\model$ and an EFG $\efg$ in Figure \ref{fig:noeq}, in which agent 1 has two binary decision variables ($A$ and $B$), agent 2 has one binary decision variable ($D$), and the utilities are assigned according to the leaves in $\efg$. 
For the sake of simplicity, we identify each pure policy with the decisions chosen, i.e., $\dom(\dot{{\bm{\Pi}}}^1) = \{\neg a \neg b, \neg a b, a \neg b, ab\}$ and $\dom(\dot{{\bm{\Pi}}}^2) = \{\neg d, d\}$. Note that agent 1 has insufficient recall as the $s$-relevance graph restricted to contain just $\Pi_{A}$ and $\Pi_{B}$ is not acyclic.

In order to show that $\model$ has no NE in behavioural policies, we first show that there exists no NE where agent $1$ uses a pure policy. Suppose, by contradiction, that there exists some policy profile $\dot{{\bm{\pi}}}=(\dot{{\bm{\pi}}}^1, \dot{{\bm{\pi}}}^2)$ where $\dot{{\bm{\pi}}}^1$ is pure and $\dot{{\bm{\pi}}}$ is an NE. 
First, $\dot{{\bm{\pi}}}^1$ cannot be $\neg ab$ or $a \neg b$ because under any choice of $\dot{\bm{\pi}}^2$ agent 1 receives expected utility $-2$ and so could improve their expected utility by playing $\neg a \neg b$ or $ab$. 
Second, if $\dot{{\bm{\pi}}}^1 = \neg a \neg b$, then agent 2's best response is $\dot{{\bm{\pi}}}^2 = d$, and if $\dot{{\bm{\pi}}}^1 = ab$, then agent 2's best response is $\dot{{\bm{\pi}}}^2 = \neg d$. However, agent 1's best response to $\dot{{\bm{\pi}}}^2=d$ is $\dot{{\bm{\pi}}}^1=ab$ and agent $1$'s best response to $\dot{{\bm{\pi}}}^2= \neg d$ is $\dot{{\bm{\pi}}}^1=\neg a \neg b$. 
Therefore, there exists no choice of pure policy $\dot{{\bm{\pi}}}^1$ for any $\dot{{\bm{\pi}}}^2 \in \dot{{\bm{\Pi}}}^2$ such that both agents are simultaneously playing a best response. Hence, no NE in pure policies exists in $\model$.

We now consider the case where agent 1 is using a behavioural policy in which $\pi^1_{A}$ and/or $\pi^1_{B}$ is stochastic. Suppose, by contradiction, that there exists some policy profile ${\bm{\pi}}=({\bm{\pi}}^1, {\bm{\pi}}^2)$ that is an NE in behavioural policies. We let agent 1's decision rules be parameterised by $p, q \in [0,1]$ such that $\pi^1_{A}(a)=p$ and $\pi^1_{B}(b)=q$. 
First, consider the case where agent 2 plays $d$ or $\neg d$ with probability $1$, then agent $1$'s best response is to play the pure policy $\dot{{\bm{\pi}}}^1 = ab$ or $\dot{{\bm{\pi}}}^1 = \neg a \neg b$ with probability 1 respectively. Since we know that no NE in pure policies exists, agent 2 must instead select a stochastic decision rule $\pi^2_{D^2}$ and so agent 2 must be indifferent between $d$ and $\neg d$. We thus obtain two constraints on $p$ and $q$. On the one hand, agent 1's behavioural policy ${\bm{\pi}}^1$ must result in $\pi^1(\neg a, \neg b) = \pi^1(a, b)$, and hence we have:
$$(1-p)(1-q) = pq \Rightarrow p + q = 1.$$
On the other hand, agent 1 receives utility $-2$ if its policy ${\bm{\pi}}^1$ selects $\neg a$ and $b$ or $a$ and $\neg b$ (whatever the choice of ${\bm{\pi}}^2$).
Therefore, we must have that $\pi^1(\neg a, b) + \pi^1(a, \neg b)  < \pi^1(a, b) + \pi^1(\neg a, \neg b)$ and thus (by substituting in the result that $p + q = 1$):
$$(1-p)q + p(1-q) < pq + (1-p)(1-q) \Rightarrow (2p - 1)^2 < 0.$$
This contradiction implies that $\model$ has no NE in behavioural policies.
However, as expected by Nash's theorem \cite{nash1950equilibrium}, there does exist an NE of $\model$ in mixed policies (i.e., where both agents randomise over their pure policies). To find this mixed policy NE, we know from the principle of indifference that the expected utility from all of an agent's pure policies in the support of their mixed policy must be the same. This means that we can immediately rule out the pure policies $\neg a b$ and $a \neg b$ being played with any positive probability. Assuming that agent $1$ plays mixed policy $\mu^1_r$, which we define as $ab$ with probability $r$ and $\neg a \neg b$ with probability $1-r$, then 
$\expect_{(\mu^1_r, d)} [ U^2 ] = \expect_{(\mu^1_r, \neg d)} [ U^2 ]$ implies $r = \frac{1}{2}$. Furthermore, assuming that agent $2$ plays mixed policy $\mu^2_s$, which we define as $d$ with probability $s$ and $\neg d$ with probability $1-s$, then 
$\expect_{(ab, \mu^2_s)} [ U^1 ] = \expect_{(\neg a \neg b, \mu^2_s)} [ U^1 ]$ implies $s = \frac{1}{2}$. Thus, the policy profile $\mu = (\mu^1_{\frac{1}{2}}, \mu^2_{\frac{1}{2}})$ is an NE of $\model$ in mixed policies.

\subsubsection*{No Subgame Perfect Equilibrium}

We {now} extend the MAID from the previous {example} (in Figure \ref{fig:noeq:a}) to construct a MAID that has no SPE. Proposition \ref{prop:SPEexist} in Section \ref{sec:equilibrium_refinements} says that in a MAID with sufficient recall, there must exist an SPE in behavioural policies. The {example} presented here demonstrates that if even one agent in a MAID does \emph{not} have sufficient recall, then there may not exist an SPE, even if we allow mixed policies.

Figure \ref{fig:nospe:a} shows the graph $\graph$ of the MAID $\model$ for our {example} and Figure \ref{fig:nospe:d} shows its $s$-relevance graph (where we can observe that agent 1 has insufficient recall). The proper $s$-subdiagrams -- $\graph_1$, $\graph_2$, and $\graph_3$ -- of $\graph$ are shown in Figures \ref{fig:nospe:b}, \ref{fig:nospe:c}, and \ref{fig:nospe:e} respectively. To parameterise $\model$, all non-utility variables are given a Boolean domain (which we interpret as integers, e.g., $a=1$ and $\neg a = 0$) and $U^1$ and $U^2$ inherit the same parameterisation as the previous {example} (shown for reference via the EFG in Figure \ref{fig:noeq:c}). We let $U^3 = D^3$ and let $U^4=1$ if and only if $D^4=B$, otherwise $U^4=0$. Finally, we set the CPD of $X$ such that $X=A$ if $D^3=1$ and $X=1-A$ if $D^3=0$.

\begin{figure}[h]
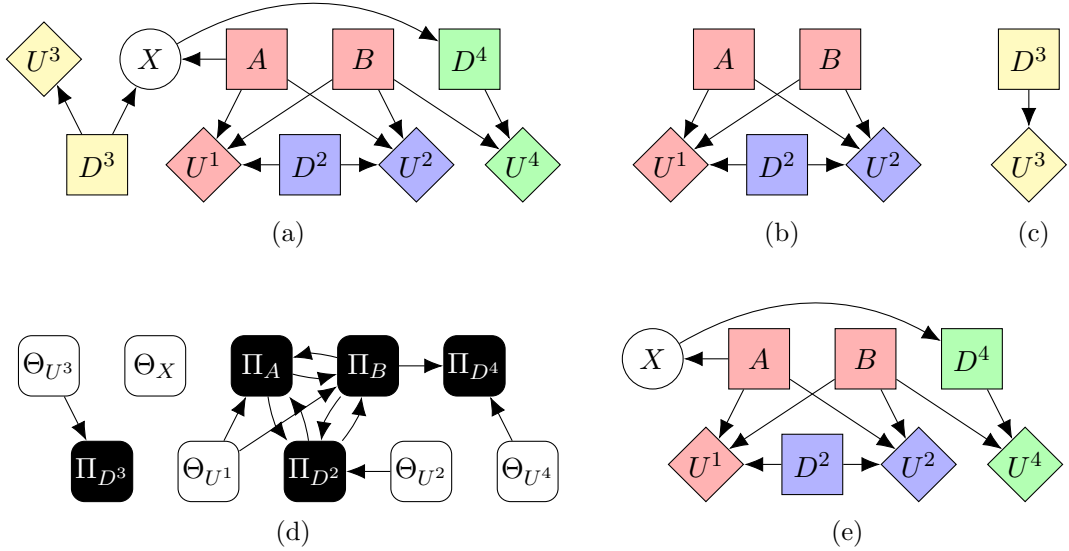

    \centering
    \begin{subfigure}[b]{0.5\linewidth}
        \vspace{0pt}
        \centering
        \begin{influence-diagram}
            \node (A) [decision, player1] {$A$};
            \node (B) [decision, right = of A, player1] {$B$};
            \node (D2) [decision, below right = 1.4cm and 0.7cm of A, player2] {$D^2$};
            \node (U1) [utility, left = of D2, player1] {$U^1$};
            \node (U2) [utility, right = of D2, player2] {$U^2$};
            \edge {A, B, D2} {U1};
            \edge {A, B, D} {U2};
            \node (X) [left = of A] {$X$};
            \node (D3) [decision, left = of U1, player3] {$D^3$};
            \node (U3) [utility, left = of X, player3] {$U^3$};
            \node (U4) [utility, right = of U2, player4] {$U^4$};
            \node (D4) [decision, right = of B, player4] {$D^4$};
            \edge {D3, A} {X};
            \edge {D3} {U3};
            \edge {D4, B} {U4}; 
            \path (X) edge[->, bend left=30] (D4);
        \end{influence-diagram}
        \caption{}
        \label{fig:nospe:a}
    \end{subfigure}
    \begin{subfigure}[b]{0.3\linewidth}
        \vspace{0pt}
        \centering
        \begin{influence-diagram}
            \node (A) [decision, player1] {$A$};
            \node (B) [decision, right = of A, player1] {$B$};
            \node (D2) [decision, below right = 1.4cm and 0.7cm of A, player2] {$D^2$};
            \node (U1) [utility, left = of D2, player1] {$U^1$};
            \node (U2) [utility, right = of D2, player2] {$U^2$};
            \edge {A, B, D2} {U1};
            \edge {A, B, D2} {U2};
        \end{influence-diagram}
        \caption{}
        \label{fig:nospe:b}
    \end{subfigure}
    \begin{subfigure}[b]{0.1\linewidth}
        \vspace{0pt}
        \centering
        \begin{influence-diagram}
            \node (D3) [decision, player3] {$D^3$};
            \node (U3) [utility, below = of D3, player3] {$U^3$};
            \edge {D3} {U3};
        \end{influence-diagram}
        \caption{}
        \label{fig:nospe:c}
    \end{subfigure}

    \vspace{0.5cm}

    \begin{subfigure}[b]{0.5\linewidth}
        \vspace{0pt}
        \centering
        \begin{influence-diagram}
            \node (A_mec) [relevanceb] {$\Pi_{A}$};
            \node (B_mec) [relevanceb, right = of A_mec] {$\Pi_{B}$};
            \node (D2_mec) [relevanceb, below right = 1.4cm and 0.7cm of A_mec] {$\Pi_{D^2}$};
            \node (U1_mec) [relevancew, left = of D2_mec] {$\Theta_{U^1}$};
            \node (U2_mec) [relevancew, right = of D2_mec] {$\Theta_{U^2}$};
            \path (A_mec) edge[->, bend right=15] (B_mec);
            \path (A_mec) edge[->, bend right=15] (D2_mec);
            \path (B_mec) edge[->, bend right=15] (A_mec);
            \path (B_mec) edge[->, bend right=15] (D2_mec);
            \path (D2_mec) edge[->, bend right=15] (A_mec);
            \path (D2_mec) edge[->, bend right=15] (B_mec);
            \edge {U2_mec} {D2_mec};
            \edge {U1_mec} {A_mec, B_mec};
            \node (Y_mec) [relevancew, left = of A_mec] {$\Theta_{X}$};
            \node (D_mec) [relevanceb, left = of U1_mec] {$\Pi_{D^3}$};
            \node (U3_mec) [relevancew, left = of Y_mec] {$\Theta_{U^3}$};
            \node (U4_mec) [relevancew, right = of U2_mec] {$\Theta_{U^4}$};
            \node (E_mec) [relevanceb, right = of B_mec] {$\Pi_{D^4}$};
            \edge {B_mec} {E_mec};
            \edge {U4_mec} {E_mec};
            \edge {U3_mec} {D_mec};
        \end{influence-diagram}
        \caption{}
        \label{fig:nospe:d}
    \end{subfigure}
    \begin{subfigure}[b]{0.4\linewidth}
        \vspace{0pt}
        \centering
        \begin{influence-diagram}
            \node (A) [decision, player1] {$A$};
            \node (B) [decision, right = of A, player1] {$B$};
            \node (D2) [decision, below right = 1.4cm and 0.7cm of A, player2] {$D^2$};
            \node (U1) [utility, left = of D2, player1] {$U^1$};
            \node (U2) [utility, right = of D2, player2] {$U^2$};
            \edge {A, B, D2} {U1};
            \edge {A, B, D2} {U2};
            \node (X) [left = of A] {$X$};
            \node (U4) [utility, right = of U2, player4] {$U^4$};
            \node (E) [decision, right = of B, player4] {$D^4$};
            \edge {A} {X};
            \edge {E, B} {U4}; 
            \path (X) edge[->, bend left=30] (E);
        \end{influence-diagram}
        \caption{}
        \label{fig:nospe:e}
    \end{subfigure}
    
    \caption{(a) A MAID $\model$ representing the {SPE non-existence example}.
    (b), (c), and (e) show the three proper $s$-subdiagrams of this MAID, $\graph_1$, $\graph_2$, and $\graph_3$ respectively.
    (d) The $s$-relevance graph of $\model$.}
    \label{fig:nospe}    
\end{figure}

From the previous {example} we know that any $s$-subgame defined over $\graph_1$ has no behavioural NEs, and that the only mixed NE is given by $\mu_1 = (\mu^1_{\frac{1}{2}}, \mu^2_{\frac{1}{2}})$ where $\mu^1_{\frac{1}{2}}$ plays the pure strategies $ab$ and $\neg a \neg b$ each with probability $\frac{1}{2}$, and $\mu^2_{\frac{1}{2}}$ plays pure strategy $d^2$ with probability $\frac{1}{2}$. Thus, any mixed SPE $\mu = (\mu^1, \mu^2, \mu^3, \mu^4)$ in $\model$ must be such that $\mu^1 = \mu^1_{\frac{1}{2}}$ and $\mu^2 = \mu^2_{\frac{1}{2}}$.

With $\mu^1$ and $\mu^2$ fixed, we now show that there is no choice of $\mu^4$ that leads to an NE in every feasible subgame defined over $\graph_3$. First, note that there are effectively only two such subgames, $\model_3$ and $\model'_3$ (defined by setting $D^3 = d^3$ and $D^3 = \neg d^3$ respectively), both of which are clearly feasible. In $\model_3$ the only policy {$\mu^4$ that forms an NE with $\mu^1$ and $\mu^2$ is one that assigns all probability mass to the pure policy that plays $d^4$ if and only if $X=1$. However, in $\model'_3$, the only best response for agent 4 is to play $\neg d^4$ if and only if $X=1$}. As such, there is no policy profile that forms an NE in \emph{every} $s$-subgame of $\model$, and hence no SPE in $\model$.

    The reason for this result is that the use of a mixed policy by agent 1 means that variables $A$ and $B$ become correlated. This can be modelled graphically by the introduction of a shared parent $C$ of $A$ and $B$. Given this correlation, $\Pi_{D^3}$ is now $s$-relevant to $\Pi_{D^4}$ due to the new path:
    {$$\Pi_{D^3} \to D^3 \to X \gets A \gets C \to B \to U^4,$$} in the independent mechanised graph, which is active given $\Fa_{D^4} = \{X, D^4\}$.

\subsubsection*{No Equivalent Mixed and Behavioural Policies}

{Finally,} we provide an example to demonstrate why sufficient recall is not a sufficient condition for every mixed policy to have an equivalent behavioural policy that results in the same probability distribution over game outcomes. It therefore serves as {a proof for the second part of} Proposition \ref{prop:sr_pr}.

\begin{figure}[h]
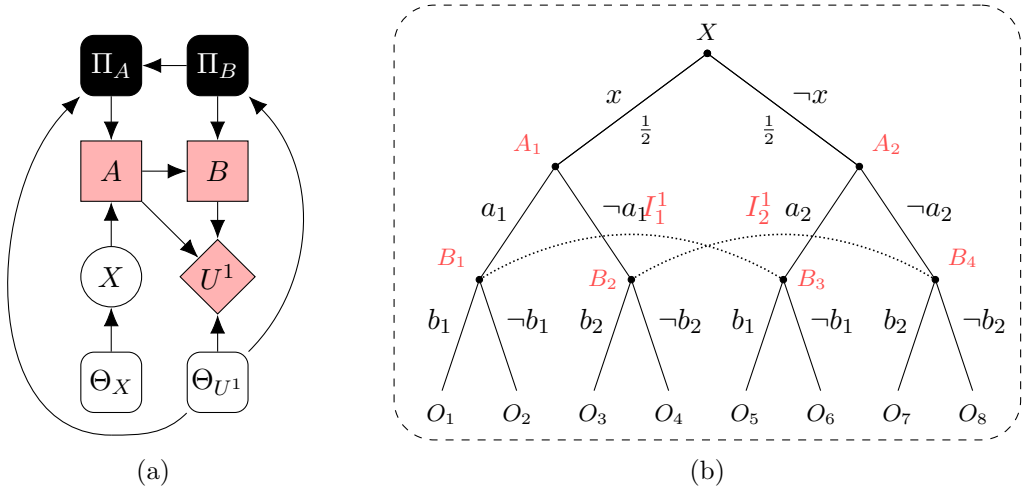

    \centering
    \begin{subfigure}[b]{0.3\linewidth}
        \vspace{0pt}
        \centering
        \begin{influence-diagram}
            \node (A1) [decision, player1] {$A$};
            \node (B1) [decision, right = of A1, player1] {$B$};
            \node (X) [below = of A1] {$X$};
            \node (U1) [utility, right = of X, player1] {$U^1$};
            \edge {A1, B1} {U1};
            \edge {X} {A1};
            \edge {A1} {B1};

            \node (U1_mec) [relevancew, below = of U1] {$\Theta_{U^1}$};
            \node (X_mec) [relevancew, below = of X] {$\Theta_{X}$};
            \node (A1_mec) [relevanceb, above = of A1] {$\Pi_{A}$};
            \node (B1_mec) [relevanceb, above = of B1] {$\Pi_{B}$};

            \edge {U1_mec} {U1};
            \edge {X_mec} {X};
            \edge {A1_mec} {A1};
            \edge {B1_mec} {B1};
            \edge {B1_mec} {A1_mec};
            
            \path (U1_mec) edge[->, bend right=45] (B1_mec);
            
            \node (space1) [minimum size=0mm, node distance=2mm, below = 0.7cm of X_mec, draw=none] {};
            \draw (U1_mec) edge[in=0,out=-135] (space1.center)
            (space1.center) edge[->,out=180,in=-135] (A1_mec);
        \end{influence-diagram}
        \caption{}
        \label{fig:sr_noeq:a}
    \end{subfigure}
    \begin{subfigure}[b]{0.6\linewidth}
        \vspace{0pt}
        \centering
        \begin{istgame}
        \xtdistance{15mm}{40mm}
        \istroot(0)<90>{$X$}
        \istb{x}[al]
        \istb{\neg x}[ar]

        \endist
        \istroot(0)
        \istb{\text{\footnotesize $\frac{1}{2}$}}[br]
        \istb{\text{\footnotesize $\frac{1}{2}$}}[bl]
        \endist
        \xtdistance{15mm}{20mm}
        \istroot(1)(0-1)<135, red!70>{$A_1$}
        \istb{a_1}[al]
        \istb{\neg a_1}[ar] 
        \endist
        \istroot(2)(0-2)<45, red!70>{$A_2$}
        \istb{a_2}[al]
        \istb{\neg a_2}[ar] 
        \endist
        \xtdistance{15mm}{10mm}
        \istroot(3)(1-1)<135, red!70>{$B_1$}
        \istb{b_1}[al]{O_1}
        \istb{\neg b_1}[ar]{O_2} 
        \endist
        \istroot(4)(1-2)<180, red!70>{$B_2$}
        \istb{b_2}[al]{O_3}
        \istb{\neg b_2}[ar]{O_4} 
        \endist
        \istroot(5)(2-1)<0,red!70>{$B_3$}
        \istb{b_1}[al]{O_5}
        \istb{\neg b_1}[ar]{O_6} 
        \endist
        \istroot(6)(2-2)<45,red!70>{$B_4$}
        \istb{b_2}[al]{O_7}
        \istb{\neg b_2}[ar]{O_8} 
        \endist
        \xtCInfoset(3)(5){\textcolor{red!70}{$I^1_1$}}[above right]
        \xtCInfoset(4)(6){\textcolor{red!70}{$I^1_2$}}[above left]
        
        \xtSubgameBox(0){(0)(1)(3-1)(3-2)(4-1)(4-2)(2)(5-1)(5-2)(6-1)(6-2)}[black,inner sep = 17pt, xshift=0pt, yshift=0pt]
        \end{istgame}
        \caption{}
        \label{fig:sr_noeq:b}
    \end{subfigure}
    \caption{(a) An $s$-minimal mechanised MAID representing the {equivalent behavioural and mixed policies non-existence example}. (b) An equivalent EFG representing the same game.
    }
    \label{fig:sr_noeq}
\end{figure}

Figure \ref{fig:sr_noeq:a} shows the mechanised MAID for this {example}, $\model$, in which the single agent has sufficient, but imperfect recall. We assume that all variables are binary, and further that $\Pr(x) = \frac{1}{2}$. 
To aid our reasoning, Figure \ref{fig:sr_noeq:b} shows the corresponding EFG for $\model$. The outcomes of the game (equivalent to a setting of all the variables in $\model$) are denoted by $\{O_1,\ldots,O_8\}$, e.g., $O_3$ represents $(x,\neg a, b, u^1_3)$.
We use the same notation as the previous {example} to denote a pure policy. For example, $\dot{{\bm{\pi}}}^1 = a_1a_2b_1\neg b_2$ is the pure policy where the agent selects $a$ whatever the value of $X$, and $b$ if $a$, and $\neg b$ if $\neg a$. Because there are two decision contexts for $A$ and $B$ respectively (corresponding to the four information sets in the EFG), there are $2^4$ pure policies. A mixed policy is then a distribution over these pure policies.

We will now show that there is no behavioural policy equivalent to the mixed policy $\mu^i = [\frac{1}{2}(a_1\neg a_2b_1\neg b_2), \frac{1}{2}(a_1a_2\neg b_1\neg b_2)]$. Let use begin by parameterising a general behavioural policy as: 
\begin{align*}
    &\pi_A(a \mid x) = p & &\pi_B(b \mid a) = r\\
    &\pi_A(a \mid \neg x) = q & &\pi_B(b \mid \neg a) = s
\end{align*}
where $p,q,r,s \in [0,1]$. Suppose, for a contradiction, that a behavioural policy equivalent to $\mu^i$ exists. This behavioural policy must then induce the same probability distribution over the outcomes of the game $\{O_1, \ldots O_8\}$ as $\mu^i$, and so the following equalities must hold:

\begin{center}
    \begin{minipage}{0.49\textwidth}
        \begin{itemize}
            \setlength\itemsep{1em}
            \item[$(O_1)$] $\frac{1}{2} \cdot p \cdot r = \frac{1}{4}$
            \item[$(O_2)$] $\frac{1}{2} \cdot p \cdot (1-r) = \frac{1}{4}$
            \item[$(O_3)$] $\frac{1}{2} \cdot (1-p) \cdot s = 0$
            \item[$(O_4)$] $\frac{1}{2} \cdot (1-p) \cdot (1-s) = 0$
        \end{itemize}
    \end{minipage}
    \begin{minipage}{0.49\textwidth}
        \begin{itemize}
            \setlength\itemsep{1em}
            \item[$(O_5)$] $\frac{1}{2} \cdot q \cdot r = 0$
            \item[$(O_5)$] $\frac{1}{2} \cdot q \cdot (1-r) = \frac{1}{4}$
            \item[$(O_7)$] $\frac{1}{2} \cdot (1-q) \cdot s = 0$
            \item[$(O_8)$] $\frac{1}{2} \cdot (1-q) \cdot (1-s) = \frac{1}{4}$
        \end{itemize}
    \end{minipage}\\
\end{center}

From $(O_5)$, we must have $q=0$ or $r=0$, but $q=0$ is ruled out by $(O_6)$. Taking $r=0$ implies that $q=\frac{1}{2}$ using $(O_6)$, and so $s=0$ using $(O_7)$. If $s=0$, then $p=1$ using $(O_4)$; however, this means that $r=\frac{1}{2}$ from $(O_1)$ or $(O_2)$ and yet we were forced to set $r=0$. This contradiction implies that there is no behavioural policy equivalent to $\mu^i$ in $\model$.

\subsection{{Reasoning about Existing Concepts using Causal Games}}
\label{app:previous_concepts}

\subsubsection*{{Blame}}

The following definition is adapted from \cite{Halpern2018}, and formalises the extent to which an agent is blameworthy for causing an event.

\begin{definition}[\citenum{Halpern2018}]
    \label{def:blame}
    Let $\model$ be an SCG, ${\bm{\pi}}$ a policy profile, $D \in \bm{D}^i$ a decision variable with $d, d' \in \dom(D)$, and $\varphi$ a Boolean combination of terms $V = v$. The \textbf{degree of blameworthiness} of $d$ for $\varphi$ relative to $d'$ is denoted $db_{S}(d, d', \varphi)$, and is defined as:
    $$db_{S}(d, d', \varphi)\coloneqq \delta_{d,d',\varphi} \cdot \frac{S - \max \big(c^i(d') - c^i(d), 0 \big)}{S},$$
    where:
    \begin{itemize}
        \item $\delta_{d,d',\varphi} \coloneqq \max\big(0, \Pr^{\bm{\pi}}(\varphi_d) - \Pr^{\bm{\pi}}(\varphi_{d'})\big)$ captures how much more likely it is that $\varphi$ will result from decision $d$ than from decision $d'$;
        \item $c^i(d) \coloneqq \sum_{U \in \bm{U}^i} \expect_{{\bm{\pi}}} [ U ] - \sum_{U \in \bm{U}^i} \expect_{{\bm{\pi}}} [ U_d ]$ captures the cost to agent $i$ of performing $d$;
        \item $S > \max_{d \in \dom(D)}c^i(d)$ is a given measure of cost-sensitivity.
    \end{itemize}
	The \textbf{overall degree of blameworthiness} of $d$ for $\varphi$ is $db_{S}(d, \varphi) \coloneqq \max_{d' \in \dom(D)} db_{S}(d, d', \varphi)$.
\end{definition}

The factor $\delta_{d,d',\varphi}$ captures the extent to which the agent's decision causes the likelihood of $\varphi$ to increase. The $\max$ operation ensures that no blame is possible if the likelihood \emph{decreases}. The second factor of $db_{S}(d, d', \varphi)$ captures the costs of actions -- the idea being that an agent is less liable to be blamed for $\varphi$ if taking an alternative action would have been particularly costly. The extent to which these costs are taken into account is determined by $S$, where note that $\lim_{S \rightarrow \infty} db_{S}(d, d', \varphi) = \delta_{d,d',\varphi}$.

In Example \ref{ex:warehouse}, for instance, one might wish to compute $db_{S}(q, \neg q, B = b)$  -- the degree of blame we should assign to robot one for moving quickly (as opposed to not moving quickly) with respect to breaking an item. Let us assume the policy profile in question is ${\bm{\pi}}^{\text{THPE}}$ as described in Section \ref{sec:equilibrium_refinements} (in which robot one moves quickly and robot two patrols if and only if it observes robot one moving quickly). Then, we have the following:
\begin{align*}
    db_{S}(q, \neg q, B = b)
    &= \delta_{q, \neg q, B = b} \cdot \frac{S - \max\big(c^1(\neg q) - c^1(q), 0\big)}{S}\\
    &= \max\big(0, \Pr^{\bm{\pi}}(b_q) - \Pr^{\bm{\pi}}(b_{\neg q})\big) \cdot \frac{S - \max\big(0, \expect_{\bm{\pi}} [U^1_{q}] - \expect_{\bm{\pi}} [U^1_{\neg q}]\big)}{S}\\
    &= \max(0, \frac{1}{3} - 0) \cdot \frac{S - \max(0, 2 - 2)}{S}\\
    &= \frac{1}{3}.
\end{align*}

Before continuing, we note that definitions {(and the definitions of intent below)} were originally formalised in SCMs with a single `action' variable $A$ and a single utility function $\mathcal{U} : \dom(\bm{V}) \rightarrow \mathbb{R}$, resulting in some small differences. Firstly, in an SCM, every variable must have a default value and actions are viewed as interventions on said variables, whereas, in SCGs, the values of the endogenous variables are typically undefined before actions are chosen. This is because SCMs model sequences of events and their causal connections in a broad sense, whereas SCGs represent a game-theoretic model to be analysed. These two framings can easily be reconciled if one posits a default value for every decision variable in an SCG, such as `do nothing'. Relatedly, the use of utility variables in SCGs strictly generalises $\mathcal{U}$, and can be used to describe a more fine-grained structure. Finally, previous work views SCMs as representing a subjective, epistemic state possessed by the intervening agent.\footnote{{The philosophical motivation here is that arguably it only makes sense to judge an agent's, e.g. intent, relative to that agent's own beliefs. In our work we make a common prior assumption and assume that all agents are aware of the structure of the game, though relaxations of this assumption ought to support more subjective definitions of intent.}} In contrast, SCGs represent a more objective view of a sequential strategic interaction that is assumed to be common knowledge among all agents in the game. This perspective may be more appropriate when, say, assessing blame and intention from the point of view of an arbiter, system designer, or other third party.

\subsubsection*{{Intent}} 
In the same work as above, the authors also formalise the question of whether or not an agent intends to bring about an event, which we adapt below for use with SCGs.

\begin{definition}[\citenum{Halpern2018}]
    \label{def:intent}
    Let $\model$ be an SCG, ${\bm{\pi}}$ a policy profile, $D \in \bm{D}^i$ a decision variable with $d \in \dom(D)$ and $\text{alt}(d) \subseteq \dom(D)$ be a set of alternative decisions under consideration.
    An agent \textbf{intends to bring about} $\bm{Y} = \bm{y}$ by performing $d$ if and only if:
    \begin{itemize}
        \item There exists some $\bm{Z} \supseteq \bm{Y}$ such that $\sum_{U \in \bm{U}^i} \expect_{{\bm{\pi}}} [ U_d ] \leq \max_{d' \in \text{alt}(d)} \sum_{U \in \bm{U}^i} \expect_{{\bm{\pi}}} [ U_{d', \bm{z}_d} ]$, and $\bm{Z}$ is minimal with respect to this inequality;
        \item $\Pr^{\bm{\pi}}(\bm{y}_d) > 0$, where recall that $\Pr^{\bm{\pi}}(\bm{y}_d)$ denotes $\Pr^{\bm{\pi}}(\bm{Y}_d = \bm{y})$;
        \item For all $\bm{y}' \in \dom(\bm{Y})$ such that $\Pr^{\bm{\pi}}(\bm{y}'_d) > 0$ we have $\sum_{U \in \bm{U}^i} \expect_{{\bm{\pi}}} [ U_{\bm{y}'} ] \leq \sum_{U \in \bm{U}^i} \expect_{{\bm{\pi}}} [ U_{\bm{y}} ]$.
    \end{itemize}
    Here $\bm{z}_d$ is used in a \emph{nested counterfactual} to denote the value that $\bm{Z}$ would take under the intervention $\Do(D = d)$ (evaluated with respect to some marginalised setting $\exovals$ of the exogenous variables).
\end{definition}
The notion of intent here addresses the problem of differentiating desirable, intended effects from undesirable, unintended effects. The first condition says that if the variables $\bm{Z}$ were set to the values that they would take under $d$, then another decision $d' \in \text{alt}(d)$ becomes at least as good, and further that $\bm{Z}$ is the minimal set of variables that the agent is intending to affect in this way. The second condition says that it is possible to bring about $\bm{y}$ by doing $d$ (where it is assumed that the agent knows this). Finally, the third condition says that $\bm{y}$ is an optimal value of $\bm{Y}$ -- among those possible under the intervention $\Do(D = d)$ -- for agent $i$.

For instance, we can see that robot two does not intend to obstruct robot one (meaning that it receives utility $U^1 = 0$) by patrolling. This is because the first property in Definition \ref{def:intent} fails to hold. Taking $\text{alt}(p) = \{\neg p\}$, then the minimal set of variables $\bm{Z}$ such that:
$$c^2(p) \leq \max_{p' \in \text{alt}(p)} \sum_{U \in \bm{U}^1} \expect_{{\bm{\pi}}} [ U_{p', \bm{z}_p} ],$$
is $\{U^2\} \not \ni U^1$. In other words, given that robot one is moving quickly, $\{U^2\}$ is the minimal set of outcomes that robot two is trying to affect by performing $p$. If the values of $U^2$ were restricted to those that occur when $D^2 = p$, then there would be no incentive for robot two to patrol.

\subsubsection*{{Incentives}} 

In what follows, structural causal influence models (SCIMs) refer to SCGs with only one {agent}.

\begin{definition}[\citenum{Everitt2021}]
    \label{def:RI}
    A policy ${\bm{\pi}}$ in a single-decision SCIM $\model$ \textbf{responds} to a variable $X$ if there exists $x \in \dom(X)$ and $\exovals \in \dom(\exovars)$ such that $\Pr^{\bm{\pi}}(D_x \mid \exovals) \neq \Pr^{\bm{\pi}}(D \mid \exovals)$. $X$ has a \textbf{response incentive} (RI) if all optimal policies ${\bm{\pi}}^*$ respond to $X$, and a graph $\graph$ admits a RI on $X$ if there is a SCIM over $\graph$ that has a RI on $X$.
\end{definition}

\begin{definition}[\citenum{Everitt2021}]
    \label{def:ICI}
    In a single-decision SCIM $\model$, there is an \textbf{instrumental control incentive} (ICI) on a variable $X$ in decision context $\pa_D$ {if for all optimal policies ${\bm{\pi}}^*$, there exists some $d \in \dom(D)$ such that:} 
    $$\sum_{U \in \bm{U}} \expect_{{\bm{\pi}}^*} [ U_{x_d} \mid \pa_D ] \neq \sum_{U \in \bm{U}} \expect_{{\bm{\pi}}^*} [ U \mid \pa_D ].$$
    A graph $\graph$ admits an ICI on $X$ if there is a SCIM over $\graph$ that has an ICI on $X$ for some $\pa_D$.
\end{definition}

{These definitions can help us to understand what algorithmic frameworks lead to agents having undesirable incentives, and thus to build safer AI systems. One such proposal for building safe AI systems} is cooperative inverse reinforcement learning (CIRL) \cite{HadfieldMenell2016}, which seeks to formalise the alignment problem \cite{Bostrom2014,Russell2019} as an \emph{assistance game} in which a human $H$ and a robot $R$ attempt to cooperatively perform a task in an unknown environment, but where the robot is uncertain about the human's true reward function (parameterised by the value of $P^H$) and so must infer it through observing the human's actions. This setup is formalised using a decentralised partially observable Markov decision process (Dec-POMDP) and is shown as a MAID in Figure \ref{fig:cirl} for a finite horizon game of three timesteps. At each timestep $t$, both the human and the robot perform an action $A^H_t$ and $A^R_t$ from their shared state $S_t$, after which they transition to a new state $S_{t+1}$ and receive a reward $R_{t+1}$, which is only observable by the human.

\begin{figure}[h]
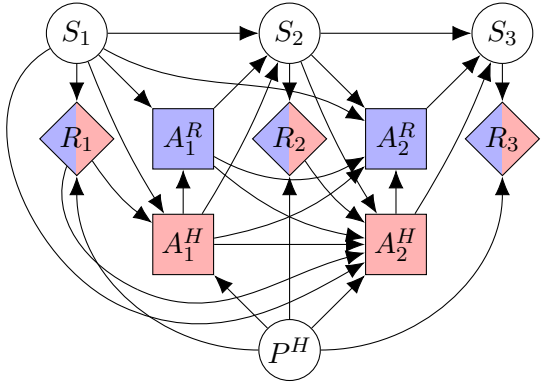

    \centering
    \begin{influence-diagram}
        \node (R1) [utility, vertical fill={red!30}{blue!30}] {$R_1$};
        \node (S1) [above = of R1] {$S_1$};
        \node (A1) [right = of R1, decision, player2] {$A^R_1$};
        \node (R2) [right = of A1, utility, vertical fill={red!30}{blue!30}] {$R_2$};
        \node (S2) [above = of R2] {$S_2$};
        \node (A2) [right = of R2, decision, player2] {$A^R_2$};
        \node (R3) [right = of A2, utility, vertical fill={red!30}{blue!30}] {$R_3$};
        \node (S3) [above = of R3] {$S_3$};
        
        \edge {S1} {R1};
        \edge {S2} {R2};
        \edge {S3} {R3};
        \edge {A1,S1} {S2};
        \edge {A2,S2} {S3};
        \edge[information] {S1} {A1};
        \edge[information] {S2} {A2};

        \draw[information,->] (S1) edge[out=330,in=150] (A2);
        
        \node (A1h) [below = of A1, decision, player1] {$A^H_1$};
        \node (A2h) [below = of A2, decision, player1] {$A^H_2$};
        
        \node (help) [below = of A1h, draw=none] {};
        
        \node (theta) at (R2 |- help) [] {$P^H$};
        
        \node (space) [minimum size=0mm, node distance=2mm, below left = 3em of A1h, draw=none] {};
        \draw (S1) edge[information,in=155,out=-150] (space.center)
        (space.center) edge[information,->,out=-25,in=-150] (A2h);

        \node (space2) [minimum size=0mm, node distance=2mm, below left = 2.2em of A1h, draw=none] {};
        \draw (R1) edge[information,in=155,out=-110] (space2.center)
        (space2.center) edge[information,->,out=-25,in=-165] (A2h);
          
        \path 
        (A1h) edge[->, bend right=5] (S2)
        (A1h) edge[->, information] (A1)
        (A1h) edge[->, information, bend right=15] (A2)        
        (A2h) edge[->, bend right=5] (S3)
        (A2h) edge[->, information] (A2)

        (R1) edge[->, information, bend right=10] (A1h)
        (R2) edge[->, information, bend right=10] (A2h)
        
        (A1) edge[->, information, bend right] (A2)
        
        (S1) edge[->, information, bend right=5] (A1h)
        (S2) edge[->, information, bend right=5] (A2h)
        (A1) edge[->, bend right=15, information] (A2h)
        (A1h) edge[->, information] (A2h)
        (theta) edge[->, information] (A1h)
        (theta) edge[->, information] (A2h)        
        
        (theta) edge[->, out=180, in=270] (R1)
        (theta) edge[->] (R2)
        (theta) edge[->, out=0, in=-90] (R3)
        ;

    \end{influence-diagram}
    \caption{A MAID representing an assistance game \cite{everitt2019modeling}, played by agents $H$ (a human) and $R$ (a robot); both want to maximise the same rewards, indicated by the shared utility variables.}
    \label{fig:cirl}
\end{figure}

If we assume that, in the example given by Figure \ref{fig:cirl}, actions $A^H_1$, $A^R_1$, and $A^H_2$ have already been taken (and thus that the MAID reduces to a single-decision ID with decision variable $A^R_2$), then we may apply the sound and graphical criteria derived in previous work to detect RIs and ICIs in CIRL \cite{Everitt2021}. A variable $X$ in a graph admits an RI (with respect to decision variable $D$) if and only if there is a directed path $X \pathto D$ in the graph that results when all information links to $D$ from variables $Y$ satisfying $Y \not\perp_\graph \bm{U} \cap \Desc_{D} \mid D, \Pa_{D}$ are removed.\footnote{The observations available at $D$ satisfying this criterion are known as \emph{non-requisite} observations \cite{Lauritzen2001}. Note the similarity to $s$-reachability, defined in Proposition \ref{prop:s-reachability}.} As $A^H_2$ does not satisfy this criterion, the path $P^H \rightarrow A^H_2 \rightarrow A^R_2$ exists in this new graph, and hence we see that the robot $R$ has an RI to act according to $P^H$, as we would hope. A variable $X$ admits an ICI (with respect to decision variable $D$) if and only if there exists a directed path $D \pathto X \pathto U$ for some utility variable $U$. Hence we see, due to the path $A^R_2 \rightarrow S_3 \rightarrow R_3$, that the robot has an ICI to influence $S_3$ as, again, we would expect.
\section{Codebase}
\label{app:code}

In this section, we briefly describe PyCID \cite{pycid}, our open source Python library that implements MAIDs (and their causal variants). PyCID has a range of classes, methods, and functions for handling IDs at level one and level two of the causal hierarchy. For our work in this paper, the \texttt{MAID} class is of primary interest; however, note that PyCID also has significant other functionality including functions for computing incentives in single-agent causal IDs \cite{Everitt2021} and reasoning patterns in MAIDs \cite{pfeffer2007reasoning}. This makes the codebase well-suited as a testbed for future research and applications. 

{We begin by showing} how to instantiate MAIDs for Examples \ref{ex:job_market} and \ref{ex:warehouse} in PyCID. {We then provide empirical results demonstrating how MAIDs can be used to compute NEs faster than in the equivalent EFGs. We refer the reader to an existing tool paper} \cite{pycid} and our online codebase for further details, including several tutorials.

\subsection{Creating MAIDs}

Listings \ref{lst:job_market} and \ref{lst:warehouse} show how to instantiate a MAID for Examples \ref{ex:job_market} and \ref{ex:warehouse} as instances of PyCID's \texttt{MAID} class, which inherits from pgmpy's \texttt{BayesianModel} class \citep{ankan2015pgmpy}. A \texttt{MAID} is initialised using a list of edges as its first argument, and then dictionaries to specify each agent's decision and utility variables. The method \texttt{draw} plots the graphs of these MAIDs (shown in Figure \ref{fig:pycidjob_market} and \ref{fig:pycidwarehouse}) with chance variables as grey circles, decision variables as rectangles, utility variables as diamonds, and colouring to denote different agents (each agent is assigned a unique colour). Recall that MAIDs are syntactically the same as CGs. Therefore, both can be defined as \texttt{MAID} objects using PyCID. In the case of MAIDs, a level one model, all class methods are permitted except those that involve causal interventions.

A MAID is parameterised by assigning domains to the decision variables and CPDs to every chance and utility variable. CPDs in PyCID are \texttt{StochasticFunctionCPD} objects, and there are multiple ways to define them. Listing \ref{lst:job_market} shows how to instantiate a MAID for Example \ref{ex:job_market}. The CPD for $T$ follows a Bernoulli distribution with success probability $\frac{1}{2}$; $D^1$ and $D^2$ are decision variables for the worker and firm's hiring system respectively with binary domains; and $U^1$ and $U^2$ are defined as described in Section \ref{sec:EFGs}.

{\centering
\begin{lstlisting}[xleftmargin=0.05\textwidth, language=Python, label={lst:job_market}, caption=An instantiation of Example \ref{ex:job_market} in PyCID.]
import pycid

job_market = pycid.MAID(
    [
        ("T", "D1"),
        ("T", "U1"),
        ("T", "U2"),
        ("D1", "D2"),
        ("D1", "U1"),
        ("D2", "U1"),
        ("D2", "U2"),
    ],
    agent_decisions={
        1: ["D1"],
        2: ["D2"],
    },
    agent_utilities={
        1: ["U1"],
        2: ["U2"],
    },
)

job_market.draw()

prob = 1/2

job_market.add_cpds(
    T = pycid.bernoulli(prob), # T = 1 corresponds to T = h (hard-working)
    D1 = [0,1], # D1 = 1 corresponds to D1 = g (going to university)
    D2 = [0,1], # D2 = 1 corresponds to D2 = j (offering a job)
    U1 = lambda d1, d2, t: 5*d2 - t*d1 - 2*d1*(1-t),
    U2 = lambda d2, t: 3*t*d2 - 2*(1-t)*d2 - (1-d2)*t,
)
\end{lstlisting}
}

Listing \ref{lst:warehouse} shows how to instantiate a MAID for Example \ref{ex:warehouse}. $D^1$ and $D^2$ are decision variables for the two robots with binary domains; $U^1$ and $U^2$ are defined as described in Section \ref{sec:subgames_eqs}; and $B$ is a chance variable defined as a function of its parent $D^1$ -- it takes value $\neg b$ with probability 1 if $D^1=\neg q$ and probability $\frac{2}{3}$ if $D^1=q$.

\begin{lstlisting}[xleftmargin=0.05\textwidth, language=Python, label={lst:warehouse}, caption=An instantiation of Example \ref{ex:warehouse} in PyCID.,]
warehouse_robots = pycid.MAID(
    [
        ("D1", "D2"),
        ("D1", "U1"),
        ("D1", "B"),
        ("B", "U2"),
        ("B", "U1"),
        ("D2", "U2"),
        ("D2", "U1"),
    ],
    agent_decisions={
        1: ["D1"],
        2: ["D2"],
    },
    agent_utilities={
        1: ["U1"],
        2: ["U2"],
    },
)

warehouse_robots.draw()
    
warehouse_robots.add_cpds(
    B = lambda d1 : {0: 1 if d1==0 else 2/3, 1: None}, # B = 1 corresponds to B = b (breaking something)
    D1 = [0,1], # D1 = 1 corresponds to D1 = q (moving quickly)
    D2 = [0,1], # D2 = 1 corresponds to D2 = p (patrolling)
    U1 = lambda d1, d2, b: (1 - 0.5*d2)*((1-d1)*2 + d1*5 - 3*b),
    U2 = lambda d2, b: 6*(1 - (1-d2)*b) - d2,
)
\end{lstlisting}

\begin{figure}[h]
    \centering
    \begin{subfigure}[b]{0.45\linewidth}
        \includegraphics[scale=.6]{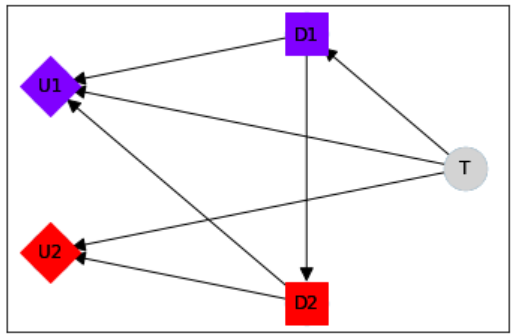}
        \centering
        \caption{}
        \label{fig:pycidjob_market}
    \end{subfigure}
    \begin{subfigure}[b]{0.45\linewidth}
        \includegraphics[scale=.6]{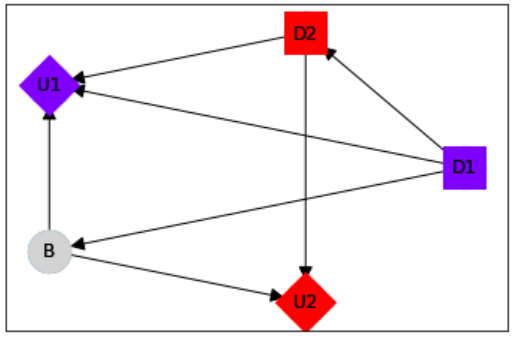}
        \centering
        \caption{}
        \label{fig:pycidwarehouse}
    \end{subfigure}
    \caption{MAIDs for (a) Example \ref{ex:job_market} and (b) Example \ref{ex:warehouse} drawn in PyCID.}
\end{figure}

\subsection{Computing Equilibria}
\label{app:subgamecompute}

PyCID finds all pure NEs and SPEs in a MAID natively, and finds all behavioural NEs and SPEs in two agent games by converting the MAID into a normal form game and interfacing with Nashpy.\footnote{Available at \href{https://github.com/drvinceknight/nashpy}{\texttt{https://github.com/drvinceknight/nashpy}}.} We refer the interested reader to our online codebase for up-to-date syntax showing how to compute NEs. All SPEs in a MAID are found by adapting Algorithm 6.2 from \cite{koller2003multi} to iterate backwards through a topological ordering of the MAID's $s$-subdiagrams, finding NEs in every $s$-subgame in turn. More specifically, the method follows the procedure outlined in the proof of Proposition \ref{prop:SPEexist}. In contrast to K\&M's algorithm, our implementation finds all (pure) SPEs, as opposed to just one.

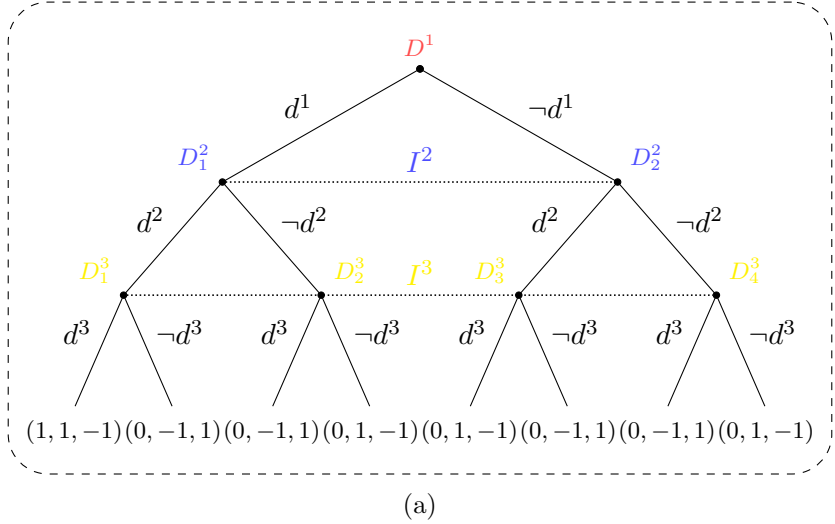
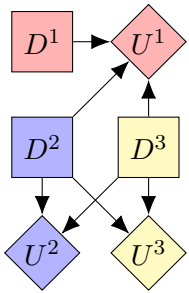
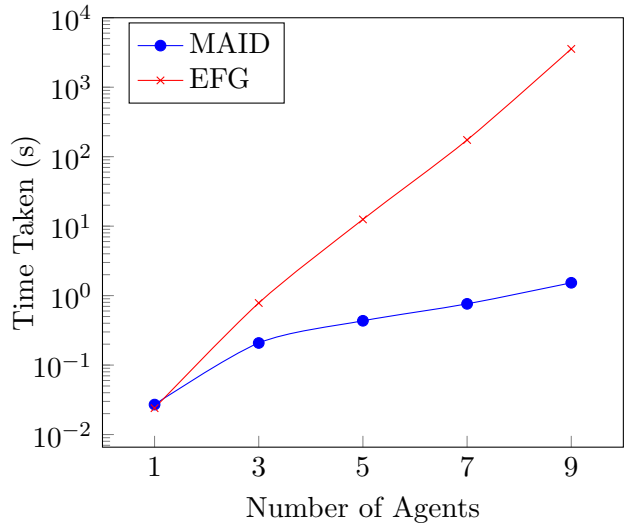
\begin{figure}[h]
    \centering
    \begin{subfigure}[b]{0.9\linewidth}
        \centering
        \begin{istgame}
            \xtdistance{15mm}{52mm}
            \istroot(0)<90, red!70>{$D^1$}
            \istb{d^1}[al]
            \istb{\neg d^1}[ar]

            \endist
            \istroot(0)
            \endist
            \xtdistance{15mm}{26mm}
            \istroot(1)(0-1)<135, blue!70>{$D^2_1$}
            \istb{d^2}[al]
            \istb{\neg d^2}[ar] 
            \endist
            \istroot(2)(0-2)<45, blue!70>{$D^2_2$}
            \istb{d^2}[al]
            \istb{\neg d^2}[ar] 
            \endist
            \xtdistance{15mm}{13mm}
            \istroot(3)(1-1)<135, yellow!100>{$D^3_1$}
            \istb{d^3}[al]{(1,1,-1)}
            \istb{\neg d^3}[ar]{(0,-1,1)} 
            \endist
            \istroot(4)(1-2)<45, yellow!100>{$D^3_2$}
            \istb{d^3}[al]{(0,-1,1)}
            \istb{\neg d^3}[ar]{(0,1,-1)} 
            \endist
            \istroot(5)(2-1)<135,yellow!100>{$D^3_3$}
            \istb{d^3}[al]{(0,1,-1)}
            \istb{\neg d^3}[ar]{(0,-1,1)} 
            \endist
            \istroot(6)(2-2)<45,yellow!100>{$D^3_4$}
            \istb{d^3}[al]{(0,-1,1)}
            \istb{\neg d^3}[ar]{(0,1,-1)} 
            \endist
            \xtInfoset(1)(2){\textcolor{blue!70}{$I^2$}}[above]
            \xtInfoset(3)(6){\textcolor{yellow!100}{$I^3$}}[above]
            
            \xtSubgameBox(0){(0)(1)(3-1)(3-2)(4-1)(4-2)(2)(5-1)(5-2)(6-1)(6-2)}[black,inner sep = 24pt, xshift=0pt, yshift=0pt]
        \end{istgame}
        \caption{}
        \label{fig:NEcompute:EFG}
    \end{subfigure}

    \vspace{0.2cm}

    \begin{subfigure}[b]{0.3\linewidth}
        \centering
        \begin{influence-diagram}
            \node (D2) [decision, player2] {$D^2$};
            \node (D3) [decision, right = of D2, player3] {$D^3$};
            \node (D1) [decision, above = of D2, player1] {$D^1$};
            \node (U1) [utility, above = of D3, player1] {$U^1$};
            \node (U2) [utility, below = of D2, player2] {$U^2$};
            \node (U3) [utility, below = of D3, player3] {$U^3$};
            \edge {D1} {U1};
            \edge {D2} {U1,U2,U3};
            \edge {D3} {U1,U2,U3};
        \end{influence-diagram}
        \vspace{1.1cm}
        \caption{}
        \label{fig:NEcompute:MAID}
    \end{subfigure}
    \begin{subfigure}[b]{0.6\linewidth}
        \begin{tikzpicture}
            \begin{axis}[
                ymode=log,
                ylabel style={at={(axis description cs:0.08,.5)},anchor=south},
                xlabel=Number of Agents,
                ylabel=Time Taken (s),
                xmin=0, xmax=10,
                ymin=0, ymax=10000,
                ytick pos=left,
                xtick pos=bottom,
                xtick={1,3,5,7,9,11},
                xticklabels={1,3,5,7,9,11},
                legend style={at={(0.2,0.8)},anchor=south,legend cell align=left}
                ]
            \addplot[smooth,mark=*,blue] plot coordinates {
                (1,0.027)
                (3,0.208)
                (5,0.434)
                (7,0.761)
                (9,1.53)
            };
            \addlegendentry{MAID}
            
            \addplot[smooth,color=red,mark=x]
                plot coordinates {
                    (1,0.024)
                    (3,0.783)
                    (5,12.5)
                    (7,174)
                    (9,3561)
                };
            \addlegendentry{EFG}
            \end{axis}
        \end{tikzpicture}
        \caption{}
        \label{fig:NEcompute:graph}
    \end{subfigure}

    \caption{{The (a) MAID and (b) EFG for the 3-agent version of the matching-pennies-like game described above. (d) A plot of the time taken to find an NE in the MAID and EFG representations of this game for varying numbers of agents.}}
    \label{fig:NEcompute}
\end{figure}

{Finally, we demonstrate the computational usefulness of subgames (and SPEs) in MAIDs. Consider a class of games with agents simultaneously choosing whether to place a coin heads or tails face up. The first agent gets utility 1 if and only if all agents in the game choose heads, otherwise their utility equals zero. All other agents are added to the game in pairs, and receive utility according a standard game of matching pennies with their partner: 1 (or -1) utility for matching (heads or tails), and -1 (or 1) for mismatching, respectively. Matching pennies has no pure NEs, only a mixed NE in which both agents randomise equally between choosing heads or tails. A MAID for the 3-agent variant of this game is shown in Figure \ref{fig:NEcompute:MAID}.}

{By the above construction, in any game belonging to this class with more than one agent, there must be no pure NE, as one agent in each pair would always have an incentive to deviate. Moreover, the EFG representation of the game (shown for the 3-agent case in Figure \ref{fig:NEcompute:EFG}) has no proper subgames.\footnote{{Technically there are six possible corresponding EFGs because of the six permutations of $D^1$, $D^2$, and $D^3$, but none of them have any proper subgames.}} Since the complexity of finding even an approximate mixed NE is PPAD-complete \cite{daskalakis2009complexity}, finding an NE in this game is hard for any EFG solver, and soon becomes intractable as the number of agent pairs playing matching pennies increases.}

{In the MAID representation, however, every matching-pennies-playing pair of agents is identified as a proper $s$-subgame. This means that the unique behavioural SPE (agents randomising equally between choosing heads or tails) can be found in these subgames (with fewer agents) before returning to the full game where, finally, the best response of agent one will be to always play heads. Figure \ref{fig:NEcompute:graph} compares the time taken to find an NE in the MAID and EFG representations of games in this game class with between one and nine agents; for the nine-agent game, the difference is three orders of magnitude.}\footnote{{These calculations were performed using PyCID \cite{pycid} and Gambit \cite{mckelvey2006gambit} on an NVIDIA Tesla K80 GPU and the time taken is the mean of seven runs.}} 

{There are games in which neither a MAID nor a corresponding EFG has any proper subgames. In these games, the time taken to compute an NE will be comparable. Nevertheless, many games will have more subgames in the MAID than the EFG (as explained in Section \ref{sec:equivalences}) and every subgame in an EFG is guaranteed to be a subgame in the corresponding MAID (as proven in Corollary \ref{prop:SPE}). Thus, subgames may often allow for the more efficient computation of equilibria in MAIDs (and hence causal games) compared to EFGs.}

\end{document}